\documentclass{article}

\usepackage{arxiv}

\usepackage[utf8]{inputenc} 
\usepackage[T1]{fontenc}    
\usepackage{hyperref}       
\usepackage{url}            
\usepackage{booktabs}       
\usepackage{amsfonts}       
\usepackage{nicefrac}       
\usepackage{microtype}      
\usepackage{xcolor}         

\usepackage{amsmath}
\usepackage{amssymb}
\usepackage{mathtools}
\usepackage{amsthm}
\usepackage{pifont}
\usepackage{multirow}
\usepackage{enumitem}
\usepackage{dashrule}
\usepackage{mathdots}
\usepackage{array}
\newcommand{\cmark}{\ding{51}} 
\newcommand{\xmark}{\ding{55}} 
\usepackage[numbers,sort&compress]{natbib}
\usepackage{wrapfig}

\usepackage{amsmath}\allowdisplaybreaks
\usepackage{amsfonts,bm}
\usepackage{amssymb}
\usepackage{algorithm}
\usepackage{algorithmic}

\def\eqref#1{equation~(\ref{#1})}

\def\1{\bf{1}}

\def\vv{{\bf{v}}}

\def\vx{{\bf{x}}}
\def\vy{{\bf{y}}}
\def\vz{{\bf{z}}}

\usepackage{amsthm}
\theoremstyle{plain}

\makeatletter
\def\Ddots{\mathinner{\mkern1mu\raise\p@
\vbox{\kern7\p@\hbox{.}}\mkern2mu
\raise4\p@\hbox{.}\mkern2mu\raise7\p@\hbox{.}\mkern1mu}}
\makeatother

\makeatletter
\newcommand*{\rom}[1]{\expandafter\@slowromancap\romannumeral #1@}
\makeatother

\theoremstyle{plain}
\newtheorem{theorem}{Theorem}[section]

\newtheorem{lemma}[theorem]{Lemma}

\theoremstyle{definition}
\newtheorem{definition}[theorem]{Definition}
\newtheorem{assumption}[theorem]{Assumption}
\theoremstyle{remark}
\newtheorem{remark}[theorem]{Remark}

\title{Achieving Better Local Regret Bound for\\ Online Non-Convex Bilevel Optimization}

\author{
 Tingkai Jia \\
  East China Normal University\\
  Shanghai China \\
  \texttt{51275902086@stu.ecnu.edu.cn} \\
   \And
 Haiguang Wang \\
  East China Normal University\\
  Shanghai China \\
  \texttt{10235101405@stu.ecnu.edu.cn} \\
  \And
 Cheng Chen \\
  East China Normal University\\
  Shanghai China \\
  \texttt{chchen@sei.ecnu.edu.cn} \\
}

\begin{document}

\maketitle

\begin{abstract}
Online bilevel optimization (OBO) has emerged as a powerful framework for many machine learning problems.
Prior works have developed several algorithms that minimize the standard bilevel local regret or the window-averaged bilevel local regret of the OBO problem, but the optimality of existing regret bounds remains unclear. 
In this work, we establish optimal regret bounds for both settings. 
For standard bilevel local regret, we propose an algorithm with adaptive iteration strategy that achieves the optimal regret $\Omega(1+V_T)$ with at most $O(T\log T)$ total inner-level gradient evaluations.
We further develop a fully single-loop algorithm whose regret bound includes an additional gradient-variation terms.
For the window-averaged bilevel local regret, we design an algorithm that captures linear environmental variation through a novel window-based analysis and achieves the optimal regret $\Omega(T/W^2)$. 
The algorithm also supports an efficient single-loop structure, achieving an $O(T/W)$ regret bound with $O(WT)$ total gradient evaluations.
Experiments validate our theoretical findings and demonstrate the practical effectiveness of the proposed methods.
\end{abstract}
\section{Introduction}

Online bilevel optimization (OBO) has become a powerful framework for a wide range of machine learning problems.
In this work, we study non-convex-strongly-convex online bilevel optimization (OBO), defined as follows: 
\begin{align}\label{eq:obo}
\begin{split}    
    \min_{\mathbf{x} \in \mathcal{X}}  F_t(\mathbf{x}) := f_t \left(\mathbf{x}, \mathbf{y}_t^*(\mathbf{x})\right) \quad \text{s.t.} \quad \mathbf{y}_t^*(\mathbf{x}) = \mathop{\rm argmin}_{\mathbf{y} \in \mathbb{R}^{d_2}} g_t (\mathbf{x}, \mathbf{y}),
\end{split}
\end{align}
where the upper-level function $f_t:\mathcal{X} \times \mathbb{R}^{d_2} \to \mathbb{R}$ is smooth but possibly non-convex, the inner-level function $g_t:\mathcal{X} \times \mathbb{R}^{d_2} \to \mathbb{R}$ is smooth and $\mu_g$-strongly convex which ensures the existence of a unique minimizer $\mathbf{y}_t^*(\mathbf{x})$ for all given $\mathbf{x}\in\mathcal{X}\subset\mathbb{R}^{d_1}$. 

\begin{table*}[t]
\caption{Comparison of algorithms on standard bilevel local regret defined in (\ref{eq:reg}) including SOBOW \cite{lin2023non}, OBBO \cite{bohne2024online} and SOGD \cite{nazari2025stochastic} in deterministic OBO (\ref{eq:obo}) on gradient queries, required  hessian-vector product (HVP) term numbers, fully single-loop structure, local regret bounds and corresponding theorems. 
$V_T$, $H_{2,T}$ are defined in (\ref{pathV}), $E_{2,T}$ and $P_T$ are in Theorem~\ref{thm:CTHO_upper} and (\ref{eq:P_T}), respectively. 
We prove the standard bilevel local regret upper bounds of OBBO and SOBOW with window size $w=1$ in Appendices~\ref{sec:proof_OBBO} and~\ref{sec:proof_SOBOW}, respectively. 
Deterministic SOGD is proven in Appendix~\ref{sec:proof_SOGD}.
}
\label{tab:1}
\vskip -0.2in
\begin{center}
\begin{small}
\setlength{\tabcolsep}{5.5pt} 
\renewcommand{\arraystretch}{1.3} 

\begin{tabular}{|c|c|c|c|c|c|}
\hline
\begin{tabular}{c}
Algorithm
\end{tabular} &
\begin{tabular}{c}
Grad. \\
Query
\end{tabular} & 
\begin{tabular}{c}
HVP \\
Num.
\end{tabular} &
\begin{tabular}{c}
Fully \\
S-Loop
\end{tabular} &
\begin{tabular}{c}
Local Regret\\
Bound
\end{tabular} & 
\begin{tabular}{c}
Theorem
\end{tabular} \\
\hline
\noalign{\vskip 2pt}
\hline
SOBOW ($w{=}1$) & $O(T)$ & $O(T^2)$ & \xmark & $O(1 + V_T + H_{2,T})$ & Theorem~\ref{thm:CIGO_upper} \\ 
\hline
OBBO ($w{=}1$) & $O(T\log T)$ & $O(T\log T)$ & \xmark & $O(1 + V_T + H_{2,T})$ & Theorem~\ref{thm:improved_OBBO} \\
\hline
SOGD & \multirow{2}{*}{$O(T)$} & \multirow{2}{*}{$O(T)$} & \multirow{2}{*}{\cmark} & $O(1 + V_T + H_{2,T} $ & \multirow{2}{*}{Theorem~\ref{thm:SOGD}} \\ 
(deterministic) & & & & $\,\,\, + E_{2,T} + P_T)$ & \\
\hline
Algorithm~\ref{alg:AOBO} & $O(T\log T)$ & $O(T\log T)$ & \xmark & $O\left(1 + V_T\right)$ & Theorem~\ref{thm:SOGO_upper} \\
\hline
Algorithm~\ref{alg:FSOBO} & $O(T)$ & $O(T)$ & \cmark & $O(1 {+} V_T {+} H_{2,T} {+} E_{2,T})$ & Theorem~\ref{thm:CTHO_upper} \\
\hline
\multicolumn{4}{|c|}{Lower Bound} & $\Omega(1 + V_T)$ & Theorem~\ref{thm:SOGO_lower} \\
\hline
\end{tabular}
\end{small}
\end{center}
\vskip -0.2in
\end{table*}

Existing literatures  mainly consider two kinds of bilevel local regrets. \citet{tarzanagh2024online, lin2023non} and \citet{bohne2024online} study window-averaged local regret for the OBO problem and achieve sublinear regret bound when the environmental variation regularities are sublinear~\cite{chiang2012online,besbes2015non,jadbabaie2015online,xu2024online}.  
\citet{nazari2025stochastic} show that smoothing may misrepresent regret, and thus study the standard bilevel local regret without window averaging. However, their bounds rely on more sublinear environmental variation regularities and do not appear to be optimal.

\begin{table*}
\caption{Comparison of algorithms on different forms of window-averaged bilevel local regret on gradient queries, required HVP term numbers, regret bounds and corresponding theorems. $W = \sum_{i=0}^{w-1}\eta^i$ for some window size $w$ and parameter $\eta\in(0,1]$. 
S-Loop means Algorithm~\ref{alg:WOBO} under single-loop structure.
}
\label{tab:2}
\begin{center}
\begin{small}
\setlength{\tabcolsep}{6pt} 
\renewcommand{\arraystretch}{1.3} 

\begin{tabular}{|c|c|c|c|c|}
\hline
\begin{tabular}{c}
Algorithm
\end{tabular} &
\begin{tabular}{c}
Grad. \\
Query
\end{tabular} &
\begin{tabular}{c}
HVP \\
Num.
\end{tabular} &
\begin{tabular}{c}
WinAvg. Local\\
Regret Bound
\end{tabular} & 
\begin{tabular}{c}
Theorem
\end{tabular} \\
\hline
\noalign{\vskip 2pt}
\hline
OAGD~\cite{tarzanagh2024online} & $O(wT)$ & Exact & $O\left(T / W + H_{1,T} + H_{2,T}\right)$ & Theorem 9 \\ 
\hline
SOBOW~\cite{lin2023non} & $O(T)$ & $O(T^2)$ & $O(T / W + V_T + H_{2,T})$ & Theorem 5.7 \\
\hline
OBBO~\cite{bohne2024online} & $O(T\log T)$ & $O(T\log T)$ & $O(T / W + V_T + H_{2,T})$ & Theorem 5.2 \\
\hline
Algorithm~\ref{alg:WOBO} (S-Loop) & $O(wT)$ & $O(wT)$ & $O\left(T / W\right)$ & Theorem~\ref{thm:unres_win_sl} \\
\hline
Algorithm~\ref{alg:WOBO} & $O(wWT)$ & $O(wWT)$ & $O\left(T / W^2\right)$ & Theorem~\ref{thm:unres_win_dl} \\
\hline
\multicolumn{3}{|c|}{Lower Bound} & $\Omega\left(T / W^2\right)$ & Theorem~\ref{thm:unres_win_lower} \\
\hline 
\end{tabular}

\end{small}
\end{center}
\vskip -0.2in
\end{table*}


In this work, we study optimal algorithms for both standard and window-averaged
bilevel local regret. 
For the standard bilevel local regret, we propose a novel adaptive inner-loop online bilevel optimizer (AOBO). 
It tracks the drift of the subproblem optimum $\mathbf{y}_t^*(\mathbf{x})$ by dynamically adjusting the number of inner iterations. 
Our analysis on the standard local regret shows that AOBO achieves the optimal $O(1+V_T)$ bound under mild  assumptions on environmental variation. We further propose a fully single-loop variant of AOBO, which we call the Fully Single-Loop Online Bilevel Optimizer (FSOBO), requiring only one subproblem iteration per round. Additionally, we analyze the standard bilevel local regret bounds of several existing methods, including SOBOW~\cite{lin2023non} and OBBO~\cite{bohne2024online}. A comparison of the regret bounds is presented in Table~\ref{tab:1}. 

For the window-averaged bilevel local regret, we introduce a novel notion of window-averaged bilevel local regret to evaluate algorithmic performance under linearly changing environments.
Building on this regret, we propose an window-averaged online bilevel optimizer (WOBO). 
Given a time window of length $w$, our WOBO algorithm attains an upper bound of $O(T/W^2)$, where $W = \sum_{i=0}^{w-1}\eta^i$ for some $\eta\in(0,1]$. 
We further establish a matching lower bound of $\Omega(T/W^2)$, which matches the proposed upper bound. 
Moreover, WOBO supports an efficient single-loop structure, achieving sublinear regret with lower computational cost, as shown in Table~\ref{tab:2}.

\subsection{Related Work}

\textbf{Bilevel Optimization.} The bilevel optimization (BO) problem was originally formulated by \citet{bracken1973mathematical}. Existing gradient-based algorithms can be broadly categorized into two classes: (1) the AID based methods \cite{pedregosa2016hyperparameter,gould2016differentiating,ghadimi2018approximation,grazzi2020iteration,ji2021bilevel,pedregosa2016hyperparameter}; (2) the ITD based approach \cite{maclaurin2015gradient,franceschi2017forward,shaban2019truncated,mackay2019self,ji2021bilevel}. Several works reformulate the bilevel problem into a constrained single-level optimization by replacing the inner solution with its optimality conditions \cite{liu2022bome, kwon2023fully, lu2024first, lu2025first, sow2022convergence}. In parallel, Hessian- or Jacobian-free methods circumvent the explicit computation of second-order derivatives by using zeroth-order approximations-such as finite differences-to estimate the hypergradient \cite{vuorio2019multimodal, ji2022will, yang2023achieving}.

\textbf{Online Bilevel Optimization.} Compared to BO, OBO remains relatively underexplored. \citet{tarzanagh2024online} and \citet{lin2023non} were the first to investigate online bilevel optimization (OBO) and introduced the window-averaged hypergradient framework. \citet{bohne2024online} adopted the iterative-tangent-differentiation (ITD) approach for hypergradient computation and proposed the OBBO algorithm, along with its stochastic variant. Subsequently, \citet{nazari2025stochastic} introduced a momentum-based search direction and proposed SOGD in stochastic environments. Recently, \citet{bohne2026non} studied non-stationary functional bilevel optimization, which can also be analyzed within the existing OBO framework. \citet{jia2026fully} studied fully first-order OBO algorithms based on Lagrangian reformulation functions. Although these methods avoid Hessian-vector product (HVP) computations, they incur weaker regret bounds.

\section{Preliminaries}

\subsection{Notations and Assumptions}

Let $\|\cdot\|$ denote the $\ell_2$-norm of the vector. 
For the inner-level function in problem (\ref{eq:obo}), we denote the partial derivatives of $g_t(\mathbf{x},\mathbf{y})$ with respect to $\mathbf{x}$ and $\mathbf{y}$ by $\nabla_\mathbf{x} g_t(\mathbf{x}, \mathbf{y})$ and $\nabla_\mathbf{y} g_t(\mathbf{x}, \mathbf{y})$, respectively. 
We also use $\nabla_{\mathbf{x}\mathbf{y}}^2 g_t(\mathbf{x}, \mathbf{y})$ and $\nabla_{\mathbf{y}\mathbf{y}}^2 g_t(\mathbf{x}, \mathbf{y})$ to denote the Jacobian of $\nabla_{\mathbf{y}} g_t(\mathbf{x}, \mathbf{y})$ with respect to $\mathbf{x}$ and the Hessian of $g_t(\mathbf{x},\mathbf{y})$ with respect to $\mathbf{y}$. 
For the upper-level function in problem (\ref{eq:obo}), we denote the total derivative of $f_t(\mathbf{x}, \mathbf{y}_t^*(\mathbf{x}))$ by $\nabla f_t(\mathbf{x}, \mathbf{y}_t^*(\mathbf{x}))$.

We impose the following Lipschitz continuous and convex assumption for problem (\ref{eq:obo}) as follows.
\begin{assumption}\label{asm1}
For all $t\in[T]$ and given $\mathbf{x}\in\mathcal{X}$, we suppose $g_t(\mathbf{x}, \mathbf{y})$ is $\mu_g$-strongly convex in $\mathbf{y}$.
\end{assumption}
\begin{assumption}\label{asm2}
For all $t\in[T]$, we suppose: (i) $f_t(\mathbf{x},\mathbf{y})$ is $L_{f,0}$-Lipschitz continuous and $\nabla f_t(\mathbf{x}, \mathbf{y})$ is $L_{f,1}$-Lipschitz continuous; (ii) $\nabla_\mathbf{y} g_t(\mathbf{x}, \mathbf{y})$ is $L_{g,1}$-Lipschitz continuous; (iii) $\nabla_{\mathbf{x}\mathbf{y}}^2 g_t(\mathbf{x},\mathbf{y})$ and $\nabla_{\mathbf{y}\mathbf{y}}^2 g_t(\mathbf{x},\mathbf{y})$ are $L_{g,2}$-Lipschitz continuous.
\end{assumption}
$\kappa_g := L_{g,1}/\mu_g$ is used to denote the condition number of $g_t(\mathbf{x}, \mathbf{y})$ with respect to $\mathbf{y}$. We also impose the standard upper bound assumption on the outer objective function value, as commonly used in online non-convex optimization.
\begin{assumption}\label{asm3}
    For all $t\in[T]$, there exists a positive constant $Q$ such that $\left| f_t(\mathbf{x}, \mathbf{y}_t^*(\mathbf{x})) \right| \leq Q$ for any $\mathbf{x}\in\mathcal{X}$.
\end{assumption}
The following assumption is used in our analysis in Section~\ref{sec:4} to restrict arbitrary superlinear drift of the inner-level optimal solutions.
\begin{assumption}\label{asm4}
    For all $t\in[T]$ and given $\mathbf{x}\in\mathcal{X}$, there exists a finite constraint domain $\mathcal{Y}$ that satisfies $\mathbf{y}_t^*(\mathbf{x}) \in \mathcal{Y}$. And exists $D>0$ if $\mathbf{y}$, $\mathbf{y}'\in\mathcal{Y}$, then $\|\mathbf{y} - \mathbf{y}'\| \leq D$.
\end{assumption}

\subsection{Online Bilevel Optimization}

We consider the same \textit{bilevel local regret} as in \citet{nazari2025stochastic}, which is defined as follows:
\begin{align}\label{eq:reg}
    \!\!\!\mathrm{Reg} (T) = \sum_{t=1}^T \left\| \mathcal{G}_\mathcal{X}\!\left( \mathbf{x}_t, \nabla f_t(\mathbf{x}_t, \mathbf{y}_t^*(\mathbf{x}_t)), \gamma \right) \right\|^2\!.\!
\end{align}
Here the notation $\mathcal{G}_\mathcal{X}$ presents the gradient mapping with respect to $\mathcal{X}$, that is
\begin{align}
    \mathcal{G}_\mathcal{X}\left( \mathbf{x}, \mathbf{g}, \gamma \right) := \frac{1}{\gamma}(\mathbf{x} - \mathbf{x}^+), \quad \text{where} \quad \mathbf{x}^+ = \mathop{\rm argmin}_{\mathbf{u}\in\mathcal{X}}\left\{\left\langle \mathbf{g}, \mathbf{u} \right\rangle + \frac{1}{2\gamma}\|\mathbf{u} - \mathbf{x}\|^2\right\} \label{eq:x_update}
\end{align}  
for given $\mathbf{g} \in \mathbb{R}^{d_1}$ and $\gamma>0$. The analysis of this regret relies primarily on the following two variation regularities for describing function changes, which are also used in \citet{lin2023non,bohne2024online} and \citet{nazari2025stochastic}.
\begin{align}\label{pathV}
&V_T := \sum_{t=2}^{T} \sup_{\mathbf{x}\in\mathcal{X}}\!\left| F_{t-1}(\mathbf{x}) - F_t(\mathbf{x}) \right| \quad \text{and} \quad H_{p,T} := \sum_{t=2}^T \sup_{\mathbf{x}\in\mathcal{X}} \left\|\mathbf{y}_{t-1}^*(\mathbf{x}) - \mathbf{y}_t^*(\mathbf{x})\right\|^p.
\end{align}






\subsection{Hypergradient-Based Algorithm}

This work primarily adopts approximate implicit differentiation (AID) to estimate the hypergradient, which is more efficient than iterative differentiation (ITD)~\cite{ji2021bilevel}. By properties of implicit functions~\cite{ghadimi2018approximation}, the hypergradient admits the following closed-form expression:
\begin{align}
	\nabla f_t(\mathbf{x}, \mathbf{y}_t^*(\mathbf{x})) 
	= \nabla_\mathbf{x} f_t(\mathbf{x}, \mathbf{y}_t^*(\mathbf{x})) - \nabla_{\mathbf{x}\mathbf{y}}^2 g_t(\mathbf{x}, \mathbf{y}_t^*(\mathbf{x})) \mathbf{v}_t^*(\mathbf{x}), 
\end{align}
where $\mathbf{v}_t^*(\mathbf{x}) := (\nabla_{\mathbf{y}\mathbf{y}}^2 g_t(\mathbf{x}, \mathbf{y}_t^*(\mathbf{x})))^{-1} \nabla_\mathbf{y} f_t(\mathbf{x}, \mathbf{y}_t^*(\mathbf{x}))$ 
can be viewed as the solution to the quadratic function $\Phi_t$:
\begin{align}
    &\mathbf{v}_t^*(\mathbf{x}) = \mathop{\rm argmin}_{\mathbf{v}\in\mathcal{V} }\Phi_t(\mathbf{x}, \mathbf{y}_t^*({\mathbf{x}}), \mathbf{v}), \quad \Phi_t(\mathbf{x}, \mathbf{y}, \mathbf{v}) := \frac{1}{2}\mathbf{v}^{\top} \nabla_{\mathbf{y}\mathbf{y}}^2 g_t(\mathbf{x}, \mathbf{y}) \mathbf{v} - \mathbf{v}^{\top}\nabla_\mathbf{y} f_t(\mathbf{x}, \mathbf{y}), \label{eq:Phi_t}
\end{align}    
where $\mathcal{V} = \{\mathbf{v} \in \mathbb{R}^{d_2} \,|\, \| \mathbf{v} \| \leq L_{f,0}/\mu_g \}$. 
By solving the subproblem in Eq.(\ref{eq:obo}) and Eq.(\ref{eq:Phi_t}) to obtain the  $\widetilde{\mathbf{y}}_t^*$ and $\widetilde{\mathbf{v}}_t^*$, we can use
\begin{align}
    \widetilde{\nabla} f_t(\mathbf{x}_t, \widetilde{\mathbf{y}}_t^*, \widetilde{\mathbf{v}}_t^*) =& \nabla_\mathbf{x}f_t(\mathbf{x}_t, \widetilde{\mathbf{y}}_t^*) - \nabla_{\mathbf{x}\mathbf{y}}^2 g_t(\mathbf{x}_t, \widetilde{\mathbf{y}}_t^*)\widetilde{\mathbf{v}}_t^* \label{eq:hypergrad}
\end{align}
to approximate the exact hypergradient $\nabla f_t(\mathbf{x}_t, \mathbf{y}_t^*(\mathbf{x}_t))$. Below we present a general definition of online hypergradient-based algorithm class. 
\begin{definition}[Online Hypergradient-Based Algorithm Class] \label{dfn:alg}
Suppose there are totally $T$ time steps, the iterates $\{(\mathbf{x}_t, \mathbf{y}_t)\}_{t=1}^T$ are generated according to $(\mathbf{x}_t, \mathbf{y}_t)\in \mathcal{H}_\mathbf{x}^t, \mathcal{H}_\mathbf{y}^t$ with $\mathcal{H}_\mathbf{x}^1 = \mathcal{H}_\mathbf{y}^1 = \{0\}$. 

For all $t$, the linear subspace $\mathcal{H}_\mathbf{y}^t$ is allowed to expand $K$ times with $\mathcal{H}_\mathbf{y}^{t,1} = \mathcal{H}_\mathbf{y}^t$ for $k=1,\dots,K$, shown as:
\begin{align*}
    \mathcal{H}_\mathbf{y}^{t,k+1} \leftarrow \mathrm{Span}\big\{&\mathbf{y}_i, \nabla_\mathbf{y} g_i(\widetilde{\mathbf{x}}_i, \widetilde{\mathbf{y}}_i)\big\},
\end{align*}
where $\widetilde{\mathbf{x}}_i \in \mathcal{H}_\mathbf{x}^i$, $\mathbf{y}_i, \widetilde{\mathbf{y}}_i\in\mathcal{H}_\mathbf{y}^{t,j}$ and $i\in[1,t]$, $j\in[1,k]$.
Then with $\mathcal{H}_\mathbf{y}^{t+1} = \mathcal{H}_\mathbf{y}^{t,K}$, $S > 0$, $\mathcal{H}_\mathbf{x}^t$ expand as follows:
\begin{align}
    \mathcal{H}_\mathbf{x}^{t+1} \leftarrow \mathrm{Span}\Big\{ \mathbf{x}_i, \nabla_\mathbf{x} f_i(\widetilde{\mathbf{x}}_i, \widetilde{\mathbf{y}}_j), \nabla_{\mathbf{x}\mathbf{y}}^2g_i(\overline{\mathbf{x}}_i, \overline{\mathbf{y}}_j)\prod_{j=1}^s\left( \mathbf{I}_{d_2} {-} \alpha \nabla_{\mathbf{y}\mathbf{y}}^2 g_i(\mathbf{x}_i^s, \mathbf{y}_j^s)\right) \nabla_\mathbf{y} f_i(\widehat{\mathbf{x}}_i, \widehat{\mathbf{y}}_i)\Big\}, \label{dfn:paper_span}
\end{align}
where $\mathbf{x}_i, \widetilde{\mathbf{x}}_i, \overline{\mathbf{x}}_i, \mathbf{x}_i^s, \widehat{\mathbf{x}}_i \in\mathcal{H}_\mathbf{x}^i$, $\widetilde{\mathbf{y}}_j, \overline{\mathbf{y}}_j, \mathbf{y}_j^s, \widehat{\mathbf{y}}_j \in\mathcal{H}_\mathbf{y}^j$ with any $\alpha\in \mathbb{R}$ for $i\in[1,t]$, $j\in[1,t+1]$ and $s\in[1,S]$.
\end{definition}
Our definition can be regard as an online extension of Definition 2 in \citet{liang2023lower}, they show that the form of Eq.(\ref{dfn:paper_span}) can simultaneously summarize the hypergradient computation of both AID and ITD methods.

\section{Online Bilevel Optimization with Standard Local Regret}\label{sec:3}

In this section, we study the online algorithms under standard bilevel local regret.

\subsection{Adaptive inner-loop Online Bilevel Optimizer}\label{sec:3.1}

\begin{algorithm}[tb] 
    \caption{Online Adaptive inner-loop Bilevel Optimizer (AOBO)}
    \label{alg:AOBO}
    \textbf{Input}:  $\mathbf{x}_1$ , $\mathbf{y}_1$, $\mathbf{v}_1$,  $\alpha$, $\beta$ $\gamma$,  $\delta$,  $M$.\\
    \textbf{Output}: Decision sequences: $\{\mathbf{x}_t\}_{t=1}^T$, $\{\mathbf{y}_t\}_{t=1}^T$. 
    
    \begin{algorithmic}[1] 
        \FOR{\(t = 1\) to \(T\)}
        \STATE Output $\mathbf{x}_t$ and $\mathbf{y}_t$ and receive feedback $f_t$ and $g_t$
        \STATE Set $\mathbf{y}_{t+1}\leftarrow\mathbf{y}_t$, $\mathbf{v}_t^1 \leftarrow \mathbf{v}_t$
        \WHILE{$\left\|\nabla_\mathbf{y} g_t(\mathbf{x}_t, \mathbf{y}_{t+1})\right\| > \delta$}
        \STATE 
        \(
        \mathbf{y}_{t+1} \leftarrow \mathbf{y}_{t+1} - \alpha\nabla g_t(\mathbf{x}_t, \mathbf{y}_{t+1})
        \)
        \ENDWHILE
        \FOR{\(m = 1\) to \(M\)}
        \STATE 
        \(
        \mathbf{v}_t^{m+1} \leftarrow \mathcal{P}_\mathcal{V}\left(\mathbf{v}_t^m - \beta\nabla_\mathbf{v}\Phi_t(\mathbf{x}_t, \mathbf{y}_{t+1}, \mathbf{v}_t^m)\right)
        \)
        \ENDFOR
        \STATE Calculate $\widetilde{\nabla}f_t(\mathbf{x}_t, \mathbf{y}_{t+1}, \mathbf{v}_{t+1})$ by (\ref{eq:hypergrad}) with $\mathbf{v}_{t+1} = \mathbf{v}_t^{M+1}$ and use it to update $\mathbf{x}_{t+1}$ by (\ref{eq:x_update})
        \ENDFOR
    \end{algorithmic}
\end{algorithm}

Note that existing methods such as OBBO \cite{bohne2024online} and SOBOW \cite{lin2023non} require a fixed number of inner-level gradient queries in each round to control the approximation error. However, this static inner-iteration strategy struggles to adapt to adversarial changes in the inner function. When the difference between $g_t$ and $g_{t-1}$ is large, the approximation error can be substantial. 


Inspired by the follow-the-leader iteration form used in \citet{hazan2017efficient,huang2023online}, we propose Adaptive inner-loop Online Bilevel Optimizer (AOBO), presented in Algorithm~\ref{alg:AOBO}. AOBO uses a dynamic inner iteration strategy to track the drift of $\mathbf{y}_t^*(\mathbf{x})$ in each round. In time step $t$, we use simple gradient descent method to solve the subproblem as follows
\begin{align}\label{eq:y_update}
    \mathbf{y}_t^{k+1} \leftarrow \mathbf{y}_t^k - \alpha\nabla_\mathbf{y}g_t(\mathbf{x}_t, \mathbf{y}_t^k).
\end{align}
Due to the $\mu_g$-strongly convexity of $g_t(\mathbf{x},\cdot)$, $\|\mathbf{y}_t^{k+1} - \mathbf{y}_t^*(\mathbf{x}_t)\|$ can converge with a constant ratio $\rho\in(0,1)$ for a fixed step size $\alpha$. 
The error tolerance parameter $\delta$ is chosen to determine when a good approximate solution $\mathbf{y}_{t+1}$ has been obtained, and the approximation error $\|\mathbf{y}_{t+1} - \mathbf{y}_t^*(\mathbf{x}_t)\|$ can be controlled by
\begin{align}
    \mu_g\left\|\mathbf{y}_{t+1} - \mathbf{y}_t^*(\mathbf{x}_t)\right\| \leq \left\|\nabla_\mathbf{y}g_t(\mathbf{x}_t, \mathbf{y}_{t+1})\right\| \leq \delta. \label{eq:AOBO_condition}
\end{align}

Under Assumptions~\ref{asm1} and~\ref{asm2}, 
we can verify the solution $\mathbf{v}_t^*(\mathbf{x})$ to the linear system (\ref{eq:Phi_t}) always lies in the Euclidean ball $\mathcal{V} = \{\mathbf{v} \in \mathbb{R}^{d_2} \mid \| \mathbf{v} \|^2 \leq L_{f,0}/\mu_g \}$ for any $\mathbf{x}\in\mathcal{X}$. Thus we adopt projected gradient descent to optimize $\Phi_t(\mathbf{x}_t, \mathbf{y}_{t+1}, \cdot)$, 
\begin{align}
    \mathbf{v}_t^{m+1} \leftarrow \mathcal{P}_\mathcal{V}\left(\mathbf{v}_t^m - \beta\nabla_\mathbf{v}\Phi_t(\mathbf{x}_t, \mathbf{y}_{t+1}, \mathbf{v}_t^m)\right). \label{eq:v_update}
\end{align}
Finally, by setting $\widetilde{\mathbf{y}}_t^*=\mathbf{y}_{t+1}$ and $\widetilde{\mathbf{v}}_t^* = \mathbf{v}_{t+1}$, we can use Eq.(\ref{eq:hypergrad}) to calculate $\widetilde{\nabla}f_t(\mathbf{x}_t, \mathbf{y}_{t+1}, \mathbf{v}_{t+1})$ and update $\mathbf{x}_t$ by Eq.(\ref{eq:x_update}).


We provide a regret upper bound and finite iteration guarantee for AOBO in the following theorem.
\begin{theorem}\label{thm:SOGO_upper}
With Assumptions~\ref{asm1}-\ref{asm3}, if we choose $\delta > 0$, $\alpha \leq \frac{1}{L_{g,1}}$, $M = \lceil-\frac{\ln T}{\ln\rho}\rceil$ and $\gamma \leq \frac{1}{2L_F}$ for some constant $L_F$ (defined in Lemma~\ref{lem:L_F}), Algorithm~\ref{alg:AOBO} can guarantee $\mathrm{Reg}(T) \leq O(1 + \delta^2T + V_T)$, and the total number of inner iterations satisfies $\mathcal{I}_T \leq O(T\log\delta^{-1} + H_{2,T})$.
\end{theorem}
Note that by choosing $\delta=1/\sqrt{T}$, we can guarantee that Algorithm~\ref{alg:AOBO} satisfies $\mathrm{Reg}(T) \leq O(1+V_T)$. The following lower bound shows that the reget upper bound of AOBO is optimal: 
\begin{theorem}\label{thm:SOGO_lower}
Consider any algorithm $\mathcal{A}$ that satisfies Definition~\ref{dfn:alg} with any $K, S>0$, under Assumptions~\ref{asm1}-\ref{asm3} with $d_1=d_2 \geq \Omega(1+V_T)$, there exist $\{f_t, g_t\}_{t=1}^T$ that satisfy $V_T = o(T)$ for which $\mathrm{Reg}(T) \geq \Omega (1+V_T)$.
\end{theorem}

\begin{remark}
We also analyze the standard bilevel local regret of SOBOW and OBBO, summarized in Table~\ref{tab:1} with complete proofs in Appendices~\ref{sec:proof_SOBOW} and~\ref{sec:proof_OBBO}, respectively. Our analysis shows that their regret upper bounds remain dependent on $H_{2,T}$, which is consistent with the window-averaged regret reported in their original work. In contrast, our AOBO algorithm eliminates the dependence on $H_{2,T}$, albeit at the expense of additional $H_{2,T}$ total gradient queries. Given that $H_{2,T}$ typically satisfies $H_{2,T} \leq O(T\log T)$, we can still guarantee the total computational cost of $\mathcal{I}_T \leq O(T\log T)$.
\end{remark}


\subsection{Fully Single-loop Online Bilevel Optimize}\label{sec:3.2}

Although AOBO achieves the optimal regret bound compared with SOBOW and OBBO, it requires to perform multiple iterations in each round. Thus we further propose a single-loop variant of AOBO, called Fully Single-loop Online Bilevel Optimizer (FSOBO), where each subproblem is solved with only one iteration. The algorithm is presented in Algorithm~\ref{alg:FSOBO}.
The following theorem provides a regret upper bound of FSOBO.

\begin{theorem}\label{thm:CTHO_upper}
Under Assumptions~\ref{asm1}-\ref{asm3}, let $\beta \leq \frac{1}{L_{g,1}}$, $\alpha \leq \frac{2\kappa_F\mu_g^2}{(\mu_g + L_{g,1})(\kappa_F\mu_g^2 + C_\beta)}$, $\gamma < \min \big\{\frac{1}{2L_F}, \sqrt{\frac{(1-\rho)C_1}{12C_\mathbf{x}}} \big\}$ for some constants $C_1$, $C_\beta$, $C_\mathbf{x}$ and $\kappa_F$ (defined in (\ref{eq:kappa_F})), Algorithm~\ref{alg:FSOBO} can guarantee $\mathrm{Reg}(T) \leq O(1 + V_T + H_{2,T} + E_{2,T})$. Here $E_{2,T} = E_{\mathbf{y}\mathbf{y},T}^g + E_{\mathbf{y},T}^f$, for $\vz = (\vx, \vy)$,
\begin{align*}
    &E_{\mathbf{y}\mathbf{y},T}^g =  \sum_{t=2}^T \sup_\mathbf{z}\left\|\nabla_{\mathbf{y}\mathbf{y}}^2 g_{t-1}(\mathbf{z}) - \nabla_{\mathbf{y}\mathbf{y}}^2 g_t(\mathbf{z})\right\|^2, \quad E_{\mathbf{y},T}^f = \sum_{t=2}^T \sup_\mathbf{z}\left\|\nabla_\mathbf{y} f_{t-1}(\mathbf{z}) - \nabla_\mathbf{y} f_t(\mathbf{z})\right\|^2.
\end{align*}

\end{theorem}

\begin{algorithm}[tb]
    \caption{Fully Single-loop Online Bilevel Optimizer (FSOBO)}
    \label{alg:FSOBO}
    \textbf{Input}: $\mathbf{x}_1$, $\mathbf{y}_1$, $\mathbf{v}_1$, $\alpha$, $\beta$, $\gamma$. \\
    \textbf{Output}: Decision sequences: $\{\mathbf{x}_t\}_{t=1}^T$, $\{\mathbf{y}_t\}_{t=1}^T$. 
    
    \begin{algorithmic}[1] 
        \FOR{\(t = 1\) to \(T\)}
        \STATE Output $\mathbf{x}_t$ and $\mathbf{y}_t$ and receive feedback $f_t$ and $g_t$
        \STATE
        Set $ \mathbf{y}_{t+1} \leftarrow \mathbf{y}_t - \alpha  \nabla_{\mathbf{y}} g_t(\mathbf{x}_t,\mathbf{y}_t) $, $\mathbf{v}_{t+1} \leftarrow \mathcal{P}_{\mathcal{V}}( \mathbf{v}_t - \beta \nabla\Phi_t(\mathbf{x}_t, \mathbf{y}_{t+1}, \mathbf{v}_t) )$
        \STATE Calculate $\widetilde{\nabla}f_t(\mathbf{x}_t, \mathbf{y}_{t+1}, \mathbf{v}_{t+1})$ by (\ref{eq:hypergrad}) and use it to update $\mathbf{x}_{t+1}$ by (\ref{eq:x_update})
        \ENDFOR
    \end{algorithmic}
\end{algorithm}

Theorem~\ref{thm:CTHO_upper} shows that by introducing additional environmental variation $E_{2,T}$, we only requires one iteration of subproblem solver in each round and thus reduce the computational cost. 
Intuitively, $E_{2,T}=o(T)$ guarantees the stability of the gradient descent update for the linear system (\ref{eq:Phi_t}), thereby justifying the single-iteration scheme. 
Note that SOGD~\cite{nazari2025stochastic} also provides single-loop guarantees for stochastic OBO. However, its bound relies on stronger sublinear variation regularities (see Table~\ref{tab:1}). 

\section{Online Bilevel Optimization with Window-Averaged  Local Regret}\label{sec:4}

Existing regret analyses rely on specific regularity measures of environmental variation, typically assuming that these variations grow sublinearly. However, such assumptions are inherently unobservable a priori. In this section, we investigate the OBO problem with window-averaged local regret. We propose a novel algorithm whose regret bound is independent of environmental variations, providing a more robust performance guarantee.

\subsection{Definition of Window-Averaged Bilevel Local Regret}\label{sec:4.1}

We first introduce a novel \textit{window-averaged bilevel local regret} in the following definition.
\begin{definition}[Window-Averaged Bilevel Local Regret]\label{dfn:window_regret}
Given a window size $w \leq T$ and weight parameter $\eta \in (0,1)$, consider the averaged objective functions:
\begin{align}\label{eq:avgfunc}
    \widehat{f}_{t,w}(\mathbf{x}, \mathbf{y}) := \frac{1}{W}\sum_{i=0}^{w-1}\eta^if_{t-i}(\mathbf{x}, \mathbf{y}) \quad \text{and} \quad \widehat{g}_{t,w}(\mathbf{x}, \mathbf{y}) := \frac{1}{W}\sum_{i=0}^{w-1}\eta^ig_{t-i}(\mathbf{x}, \mathbf{y}), 
\end{align}
where $W := \sum_{i=0}^{w-1}\eta^i$ and $f_t \equiv 0$ for $t<0$. Suppose the inner-level optimum on $\widehat{g}_{t,w}$ is  $\mathbf{y}_{t,w}^*(\mathbf{x}) := \arg\min_\mathbf{y}\widehat{g}_{t,w}(\mathbf{x}, \mathbf{y})$. Our window-averaged bilevel local regret is defined as
\begin{align}
    \text{Reg}_{w}(T) := \sum_{t=1}^T\left\|\mathcal{G}_\mathcal{X}\left(\mathbf{x}_t, \nabla\widehat{f}_{t,w}(\mathbf{x}_t, \mathbf{y}_{t,w}^*(\mathbf{x}_t)), \gamma\right)\right\|^2. \label{eq:win_reg}
\end{align}
\end{definition}

Our definition is similar to the definition in \citet{tarzanagh2024online}, however, they do not take into account the averaging of the inner functions, which partly explains the additional dependency of their results on $H_{1,T}$ and $H_{2,T}$, which describe the drift of $\mathbf{y}_t^*(\mathbf{x})$ with time step $t$.


\subsection{Window-Averaged Online Bilevel Optimizer}\label{sec:4.2}

\begin{algorithm}[tb]
    \caption{Window-averaged Online Bilevel Optimizer (WOBO)}
    \label{alg:WOBO}
    \textbf{Input}: $\mathbf{x}_1$, $\mathbf{y}_1$, $\mathbf{v}_1$, $\alpha$, $\beta$, $\gamma$, $w$, $\eta$, $\delta$. \\
    \textbf{Output}: Decision sequences: $\{\mathbf{x}_t\}_{t=1}^T$, $\{\mathbf{y}_t\}_{t=1}^T$. 
    
    \begin{algorithmic}[1] 
        \FOR{\(t = 1\) to \(T\)}
        \STATE Output $\mathbf{x}_t$ and $\mathbf{y}_t$ and receive feedback $f_t$ and $g_t$
        \STATE Set $(\mathbf{x}_{t+1}, \mathbf{y}_{t+1}, \mathbf{v}_{t+1}) \leftarrow (\mathbf{x}_t, \mathbf{y}_t, \mathbf{v}_t)$
        \REPEAT
        \STATE \(
        \mathbf{y}_{t+1} \leftarrow \mathbf{y}_{t+1} - \alpha\nabla_\mathbf{y}\widehat{g}_{t,w}(\mathbf{x}_{t+1}, \mathbf{y}_{t+1})
        \)
        \STATE  \(
        \mathbf{v}_{t+1} \leftarrow \mathbf{v}_{t+1} - \frac{\beta}{W}\sum_{i=0}^{w-1}\eta^i\Phi_{t-i}(\vx_{t+1}, \vy_{t+1}, \vv_{t+1})
        \)
        \STATE Calculate $\widetilde{\nabla}\widehat{f}_{t,w}(\mathbf{x}_{t+1}, \mathbf{y}_{t+1}, \mathbf{v}_{t+1})$ by (\ref{eq:win_hypergrad}) and use it to update $\mathbf{x}_{t+1}$ by (\ref{eq:x_update})
        \UNTIL condition~(\ref{eq:condition}) is satisfied (for WOBO-SL, the loop terminates after one iteration)
        \ENDFOR
    \end{algorithmic}
\end{algorithm}

Similar to (\ref{eq:y_update}) and (\ref{eq:v_update}), we apply gradient descent to solve $\min_\vy\widehat{g}_{t,w}(\vx_t, \vy)$ and the averaged linear system, respectively.
Note that the constraint set $\mathcal{Y}$ in Assumption~\ref{asm4} is prior information, so the algorithm is not allowed to directly apply projected gradient methods over it.
The window-averaged hypergradient can be calculated by the following  formula:
\begin{align}
    \widetilde{\nabla}\widehat{f}_{t,w}(\mathbf{x}, \mathbf{y}, \mathbf{v}) = \nabla_\mathbf{x}\widehat{f}_{t,w}(\mathbf{x}, \mathbf{y}) - \nabla_{\mathbf{x}\mathbf{y}}^2\widehat{g}_{t,w}(\mathbf{x}, \mathbf{y})\mathbf{v}. \label{eq:win_hypergrad}
\end{align}
We extend the adaptive inner-loop iteration strategy in Algorithm~\ref{alg:AOBO} and aim to capture the sublinear environmental
variation within a given time window. We propose Window-averaged Online Bilevel Optimizer (WOBO) in Algorithm~\ref{alg:WOBO} which adopts following error tolerance condition:
{\small
\begin{align}
    \left\| \widetilde{\nabla} \widehat{f}_{t,w}(\mathbf{x}_{t+1}, \mathbf{y}_{t+1}, \mathbf{v}_{t+1}) \right\|^2 + \kappa_F \left\| \nabla_\mathbf{y} \widehat{g}_{t,w}(\mathbf{x}_{t+1}, \mathbf{y}_{t+1}) \right\|^2 
    + 8\kappa_g^2 \left\| \nabla_\mathbf{v} \widehat{\Phi}_{t,w}(\mathbf{x}_{t+1}, \mathbf{y}_{t+1}, \mathbf{v}_{t+1}) \right\|^2 \leq \frac{\delta^2}{W^2}\label{eq:condition}
\end{align}}
The essence of Eq.(\ref{eq:condition}) is to use a computable upper bound of $\nabla\widehat{f}_{t,w}(\mathbf{x}, \mathbf{y}_{t,w}^*(\mathbf{x}))$ at each round to capture a sequence of $(\mathbf{x}, \mathbf{y}, \mathbf{v})$ that makes the hypergradient sufficiently small.
The following theorem establishes an upper bound on the proposed regret in~\eqref{eq:win_reg} for WOBO.
\begin{theorem}\label{thm:unres_win_dl}
Under Assumptions~\ref{asm1}-\ref{asm3} and~\ref{asm4}, let $\delta>0$, $\beta = \frac{1}{L_{g,1}}$, $\alpha = \frac{1}{L_{g,1}}$ and $\gamma \leq \min \{\frac{1}{4L_F}, \frac{1}{456\kappa_g^4L_{g,1}L_\mathbf{v}\sqrt{\kappa_F}} \}$ with some constant $L_\vv$ (defined in Lemma~\ref{lem:L_v}), WOBO can obtain $\mathrm{Reg}_w(T) \leq O(T/W^2)$, and the total number of iterations satisfies $\mathcal{I}_T \leq O(WT)$.
\end{theorem}
Theorem~\ref{thm:unres_win_dl} shows that, without imposing any sublinear assumption on environmental variation, WOBO achieves sublinear regret when $W=o(T)$.
In the following theorem, we establish a lower bound on our window-averaged regret that matches our upper bound.
\begin{theorem}\label{thm:unres_win_lower}
Consider any online hypergradient-based algorithm, under Assumptions~\ref{asm1}-\ref{asm3} and~\ref{asm4} with $d_1=d_2 \geq \Omega(T)$, given window size $w > 0$ and $\eta \in(0,1)$, there exist $\{f_t, g_t\}_{t=1}^T$ for which $\mathrm{Reg}_w(T) \geq \Omega(T/W^2)$.
\end{theorem}


Note that although WOBO achieves optimal regret, it requires multiple iterations in each round. To address this, we consider a single-loop variant of WOBO, where the updates in lines 4–8 of Algorithm~\ref{alg:WOBO} are performed only once per round. We refer to this variant as WOBO-SL. The following theorem shows that WOBO-SL still preserves a sublinear regret when $W=o(T)$.
\begin{theorem}\label{thm:unres_win_sl}
    Under Assumptions~\ref{asm1}-\ref{asm4}, let $\beta\leq \frac{1}{L_{g,1}}$, $\alpha \leq \frac{4\kappa_F\mu_gL_{g,1}}{\kappa_g\mu_g(2\kappa_F\mu_g(\mu_g + L_{g,1}) + 8c_\beta L_{g,1}^3)}$ and $\gamma \leq \min\{\frac{1}{2L_F}, \sqrt{\frac{1-\rho}{48c_\vx}}\}$ with some constants $c_\beta$ and $c_\vx$ (defined in (\ref{eq:c_beta}) and (\ref{eq:c_x}), respectively), WOBO-SL can obtain $\mathrm{Reg}_w(T) \leq O(T/W)$.
\end{theorem}

\section{Experiments}\label{sec:6}

In this section, we conduct one simulation experiment and two real-data experiments to validate our theory and evaluate the effectiveness of our algorithms.

\subsection{Synthetic Experiment}\label{sec:toy_experiment}

In Section 3, our theoretical analysis relies on certain measures of environmental variation, such as $V_T$ and $H_{2,T}$, which are intractable to compute on real-world datasets. Therefore, we validate our theoretical results on a toy example, where the environmental variation can be computed exactly and controlled via specific parameters.
We devide $f_t$ into $N$ blocks. For $k\in[N]$ and $B=\lceil T/N \rceil$, we set $f_{(k-1)B+1}(\mathbf{x},\mathbf{y}) = f_{(k-1)B+2}(\mathbf{x},\mathbf{y}) = \cdots = f_{kB}(\mathbf{x},\mathbf{y}) = \tilde{f}_k(\mathbf{x},\mathbf{y})$ and set $g_t$ analogously. Here we define
\begin{align*}
    \tilde{f}_k(\mathbf{x}, \mathbf{y}) = \left(1 - e^{-[\mathbf{y}]_1^{10}}\right)\tanh(10[\mathbf{x}]_k) \quad \text{and} \quad \tilde{g}_k(\mathbf{x}, \mathbf{y}) = 0.5\left\|\mathbf{y} - a_k\mathbf{e}_1\right\|^2, 
\end{align*}
where $a_k = 1.5(-1)^{\left\lceil \frac{k}{n} \right\rceil}$ for some $n\in \mathbb{N}^+$. We run AOBO, SOBOW, and OBBO for $T=1000$, with the inner learning rate $\alpha=0.1$ and the outer learning rate $\gamma=0.05$. $\delta=0.5$ for AOBO and the number of inner iterations for OBBO is $K=10$. The variations $V_T$ and $H_{2,T}$ can be controlled by adjusting $B$ and $n$. Our experimental results are shown in the following Table~\ref{tab:variation_comparison}.
\begin{table}[t]
\caption{Comparison of algorithm regrets under environments with different variation patterns.}
\label{tab:variation_comparison}
\centering
\small
\setlength{\tabcolsep}{5pt}
\begin{tabular}{lccccc}
\toprule
\textbf{Pattern (a)} & $(B{=}500, n{=}1)$ & $(B{=}250, n{=}2)$ & $(B{=}100, n{=}5)$ & $(B{=}50, n{=}10)$ & $(B{=}10, n{=}50)$ \\
\midrule
$V_T$     & 2.0  & 4.0  & 10.0  & 20.0  & 100.0 \\
$H_{2,T}$ & 18.0 & 18.0 & 18.0  & 18.0  & 18.0 \\
\midrule
AOBO  & 300.011  & 500.018  & 1100.030  & 2100.015  & 10100.005 \\
SOBOW & 954.342  & 1154.342 & 1754.336  & 2754.323  & 11064.305 \\
OBBO  & 400.173  & 600.170  & 1200.165  & 2200.156  & 10200.104 \\
\midrule
\textbf{Pattern (b)} & $n{=}50$ & $n{=}25$ & $n{=}10$ & $n{=}5$ & $n{=}1$ \\
\midrule
$V_T$     & 100.0 & 100.0 & 100.0 & 100.0 & 100.0 \\
$H_{2,T}$ & 18.0  & 36.0  & 90.0  & 180.0 & 900.0 \\
\midrule
AOBO  & 10100.005 & 10100.005 & 10100.005 & 10100.005 & 10100.005 \\
SOBOW & 11064.305 & 11688.409 & 13560.512 & 16591.754 & 99717.328 \\
OBBO  & 10200.104 & 10400.103 & 11000.098 & 12000.090 & 20000.039 \\
\bottomrule
\end{tabular}
\vspace{-0.1in}
\end{table}
In Pattern(a), we fix $H_{2,T}$ and compare the regret under different $V_T$.
From the table, we can observe that the regrets of AOBO are approximately satisfies $\mathrm{Reg}(T) \approx 100(1+V_T)$, which validates our theoretical regret bound in Theorem~\ref{thm:SOGO_upper}. In Pattern(b), we fix $V_T$ and set $B=10$, vary $n$ to adjust $H_{2,T}$ for comparison.
It shows that the regret of AOBO does not rely on $H_{2,T}$, whereas the regret of SOBOW and OBBO increases as $H_{2,T}$ grows.

\subsection{Online Data Hyper-cleaning}\label{sec:hyper_cleaning}

\begin{figure*}[t]
\centering
\begin{tabular}{ccc}
\includegraphics[width=0.3\linewidth]{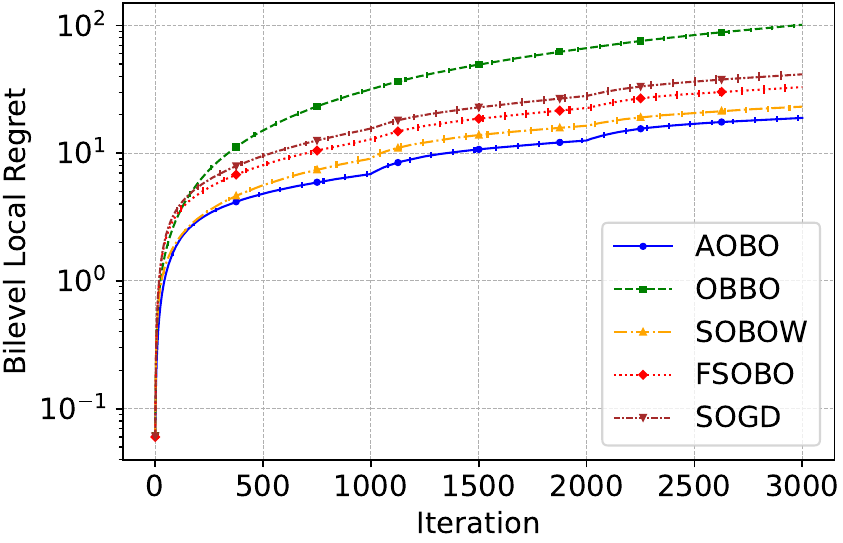} &
\includegraphics[width=0.3\linewidth]{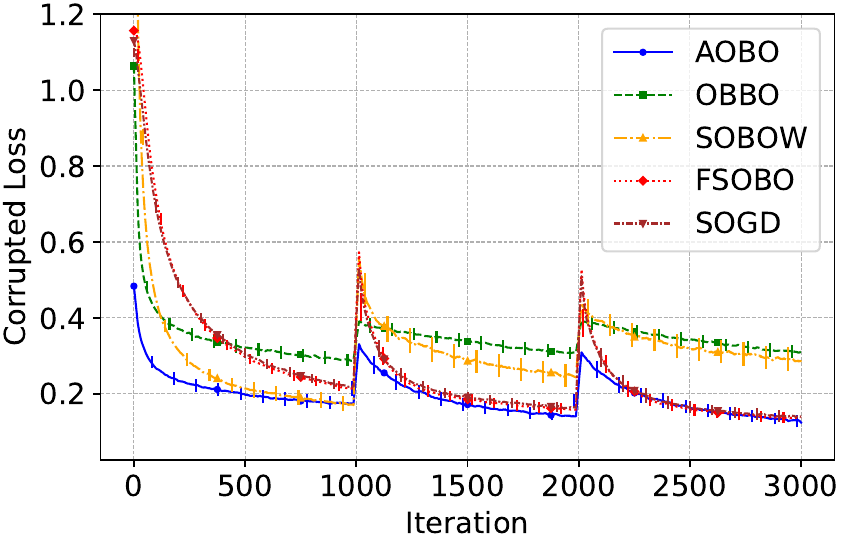} &
\includegraphics[width=0.3\linewidth]{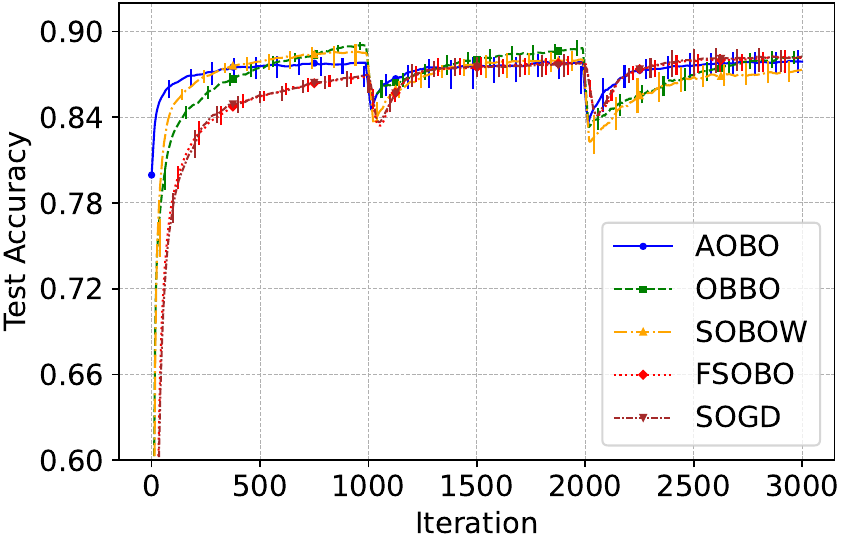} \\
(a) regret & 
(b) training loss & 
(c) test accuracy \\
\end{tabular}
\caption{Performance of different algorithms on the online hyper-cleaning.}
\label{fig:hyperclean_standard}
\end{figure*}

Hyper-cleaning formulates label-noise learning as a bilevel optimization problem, where training samples are assigned adaptive weights to down-weight noisy data and up-weight clean data. In the online setting, data arrive sequentially, requiring the optimizer to update efficiently and robustly under changing data distributions. At each step $t$, the optimizer receives a noisy training set $\mathcal{D}_t^{\mathrm{tr}}$ and a clean validation set $\mathcal{D}_t^{\mathrm{val}}$. The model parameters $\theta_t$ are updated by the inner problem, while the sample weights $w_t$ are optimized through the outer objective, forming an online bilevel framework.
\begin{align}
    \min_{w\in\mathbb{R}^{d_1}} F_t(w) =& \frac{1}{\left| \mathcal{D}_t^{\mathrm{val}} \right|}\sum_{(x_i,y_i)\in \mathcal{D}_t^{\mathrm{val}} } \mathcal{L}(f(\theta_t^*(w);x_i),y_i) \quad \mathrm{s.t.} \quad \theta_t^*(w) = \mathop{\rm argmin}_{\theta\in\mathbb{R}^{d_2}} g_t(w,\theta), \nonumber\\
    &\text{where} \quad 
    g_t(w,\theta) = \frac{1}{\left|\mathcal{D}_t^{\mathrm{tr}}\right|}\sum_{(x_i,y_i)\in \mathcal{D}_t^{\mathrm{tr}}}\mathcal{L}\left( \sigma(w)f(\theta;x_i),y_i \right). \label{eq:hyper_cleaning}
\end{align}
Here, $\mathcal{L}$ denotes the cross-entropy loss and $\sigma(\cdot)$ is the sigmoid function. We evaluate this task on MNIST~\cite{726791} using a linear classifier.
We conduct five repeated runs with a data stream starting from a training set with 10\% label noise. Corrupted labels are consistently replaced with a fixed class label (set to 1), causing a stronger misleading effect on the optimizer. Each method runs for 3000 steps, divided into three phases: steps 0--999 with 10\% noisy data, steps 1,000--1,999 with 20\%, and steps 2,000--2,999 with 30\%. 
In Figure~\ref{fig:hyperclean_standard}, we compare the performance of the algorithms listed in Table~\ref{tab:1}, where AOBO achieves the lowest cumulative regret. In Figure~\ref{fig:hyperclean_w_eta_compare}, we compare the effects of different window parameters on WOBO. The result shows that larger $w$ and $\eta$ can reduce the regret, which confirms the conclusion of Theorem~\ref{thm:unres_win_dl}. 

\subsection{Parametric Loss Tuning for Imbalanced Data}

\begin{figure*}[t]
\centering
\begin{tabular}{ccc}
\includegraphics[width=0.3\linewidth]{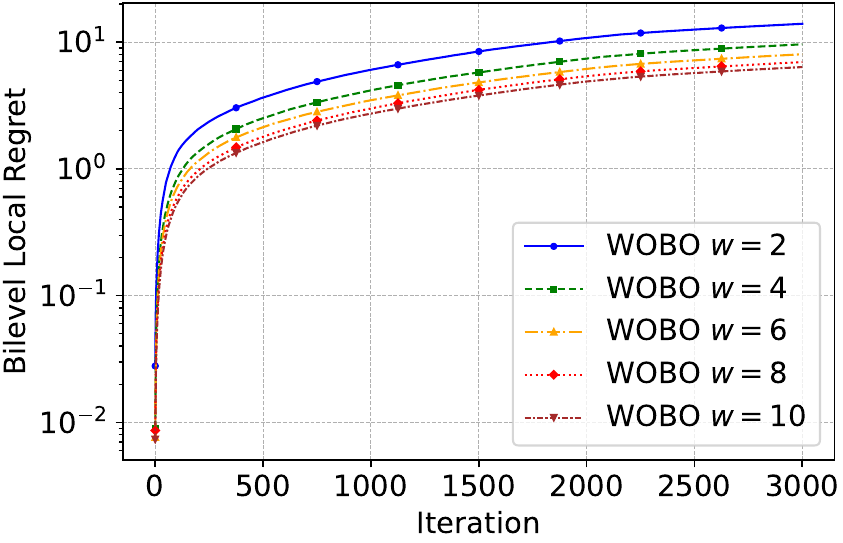} &
\includegraphics[width=0.3\linewidth]{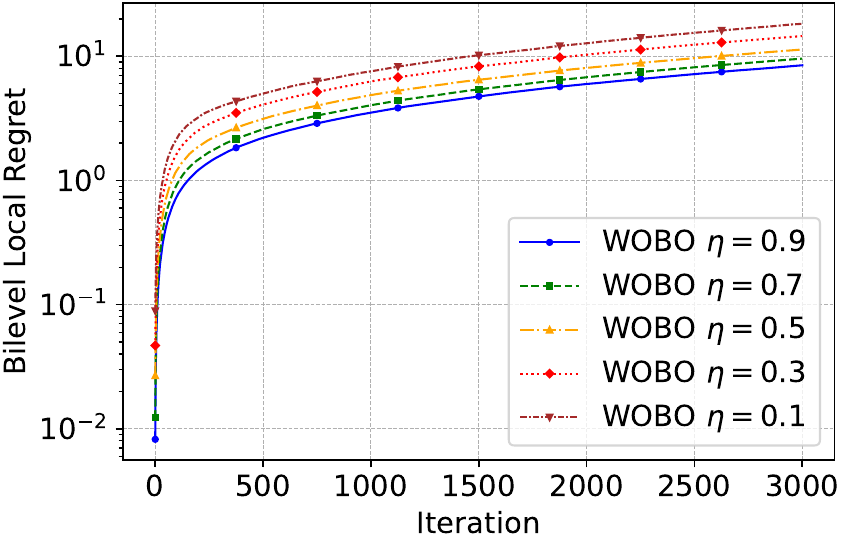} &
\includegraphics[width=0.3\linewidth]{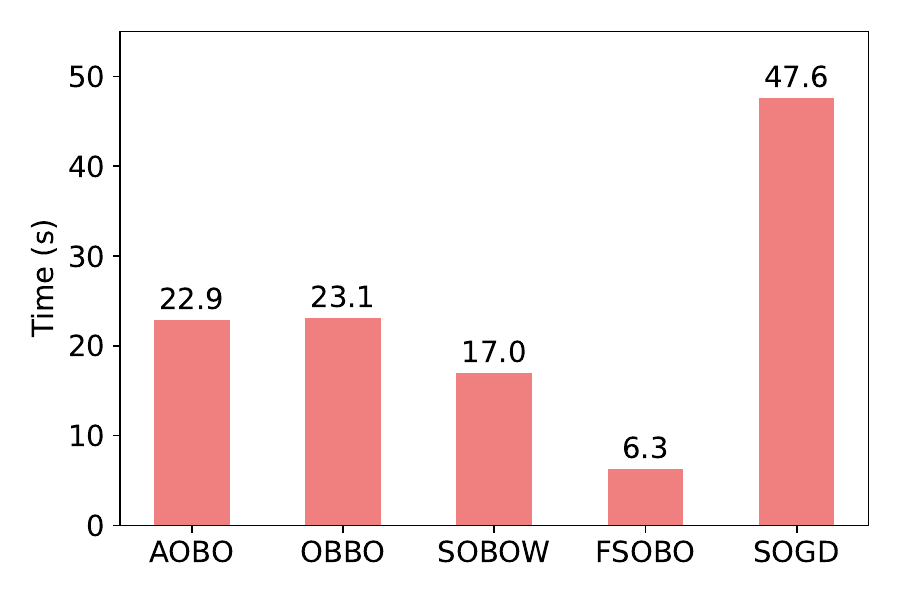} \\
(a) regret under different $w$ & 
(b) regret under different $\eta$ & 
(c) runing time 
\end{tabular}
\caption{Performance of WOBO under different window size $w$ with fixed $\eta=0.9$, and different weight $\eta$ with fixed $w=5$, in subfigures (a) and (b), respectively. The running time of different algorithms in Figure~\ref{fig:hyperclean_standard} is shown in subfigure (c).}
\label{fig:hyperclean_w_eta_compare}
\end{figure*}

\begin{figure*}[t]
\begin{tabular}{ccc}
\includegraphics[width=0.3\linewidth]{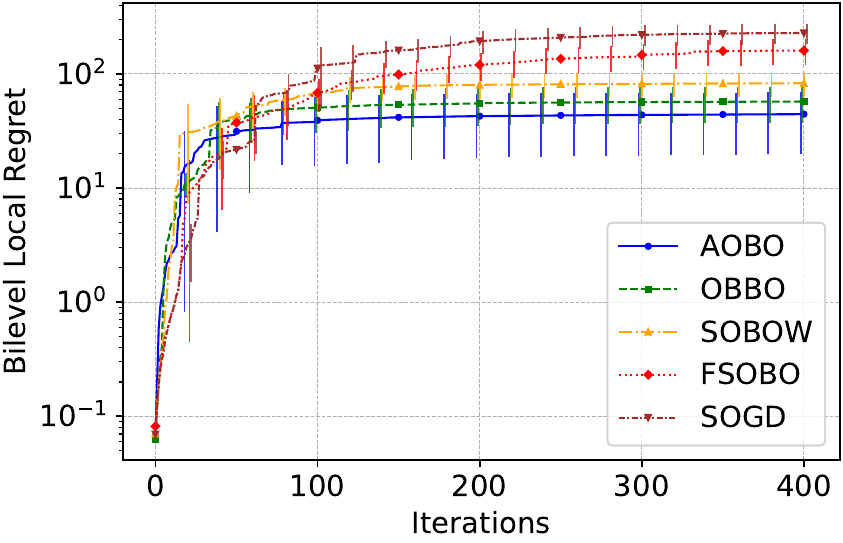} &
\includegraphics[width=0.3\linewidth]{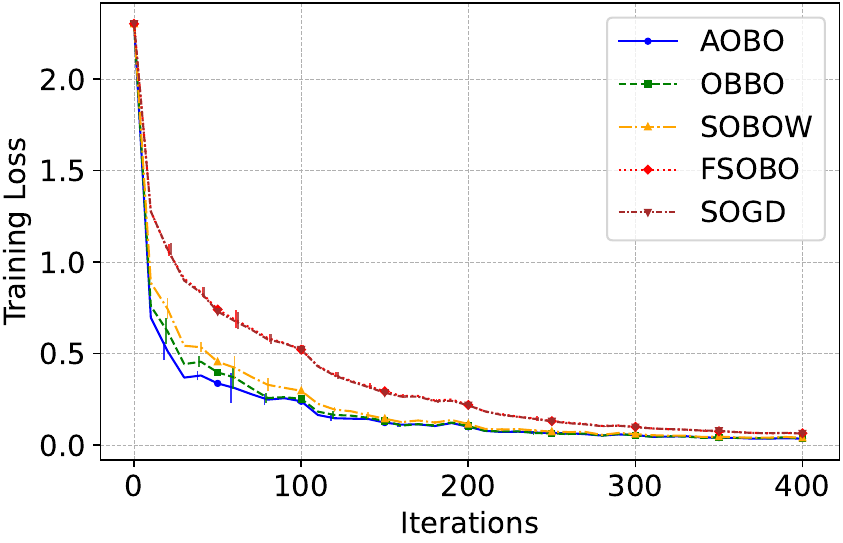} &
\includegraphics[width=0.3\linewidth]{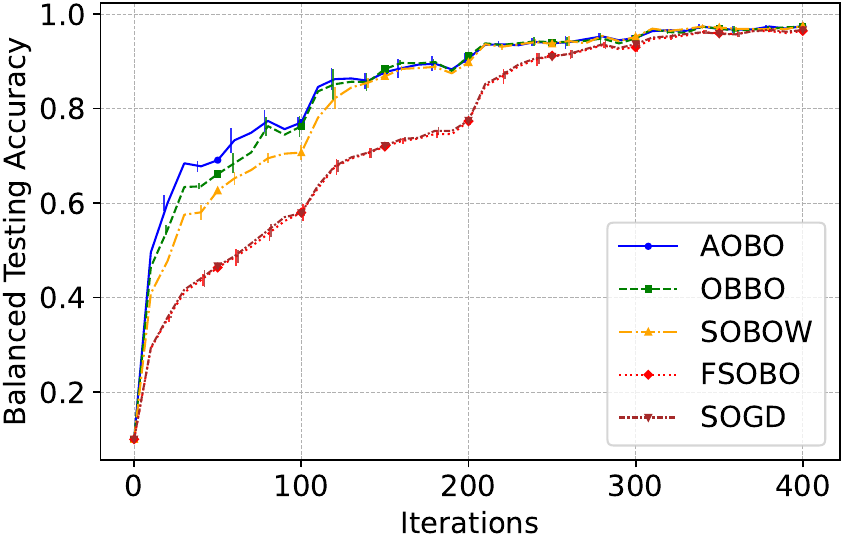} \\
(a) regret & 
(b) training loss & 
(c) test accuracy 
\end{tabular}
\caption{Performance of different algorithms on the parametric loss tuning for imbalanced data.}
\label{fig:losstuning_standard}
\end{figure*}

Tuning for imbalanced data addresses the bias induced by skewed class distributions, where models often favor majority classes over minority classes \cite{tarzanagh2024online,nazari2025stochastic}. We use a parametric training loss to improve class balance and reduce overfitting. The model is trained on $\mathcal{D}_t^{\mathrm{tr}}$, while the logit-adjustment parameters $\lambda = (\gamma_j, \Delta_j)_{j=1}^J$ are treated as outer variables and optimized on $\mathcal{D}_t^{\mathrm{val}}$. Similar to (\ref{eq:hyper_cleaning}), the problem is formulated as
\begin{align}
    \min_{\lambda\in\mathbb{R}^{d_1}} F_t(\lambda) &= \frac{1}{\left| \mathcal{D}_t^{\mathrm{val}} \right|}\sum_{(x_i,y_i)\in \mathcal{D}_t^{\mathrm{val}} } u_{y_i}\mathcal{L}(f(\theta_t^*(\lambda);x_i), y_i), \quad \text{s.t.} \quad \theta_t^*(\lambda) = \mathop{\rm argmin}_{\theta\in\mathbb{R}^{d_2}} g_t(\lambda,\theta) \nonumber\\
    &\text{where} \quad g_t(\lambda,\theta) = -\frac{1}{\left|\mathcal{D}_t^{\mathrm{tr}}\right|}\sum_{(x_i,y_i)\in \mathcal{D}_t^{\mathrm{tr}}}\log\frac{\exp(\gamma_{y_i}[f(\theta;x_i)]_{y_i} + \Delta_{y_i})}{\sum_{j=1}^J \exp(\gamma_j[f(\theta;x_i)]_j + \Delta_j)}. \nonumber
\end{align}
Here, $u_j$ is the reciprocal class proportion for class $j$~\cite{DBLP:journals/corr/abs-2201-01212}. We use a lightweight CNN on MNIST with four convolutional blocks, followed by a flattening layer and a final linear layer producing $10$ logits.
We evaluated performance over $400$ timesteps in four $100$-timestep phases, transitioning from $0.6^i$ to $0.9^i$ distribution for each class $i = 0, 1, \dots , 9$. Our experimental results are shown in Figure~\ref{fig:losstuning_standard}. AOBO achieves the smallest cumulative local regret and the best balanced accuracy on the test set. Moreover, FSOBO attains a smaller regret than SOGD, which corroborates our theoretical results.

\section{Conclusion and Discussion}
This work studies the non-convex–strongly-convex OBO problem under the standard local regret and window-averaged local regret. 
We first propose a novel algorithm with adaptive inner-loop iterations and obtain an optimal bound under the standard bilevel local regret. We also provide a tighter upper bound for a fully single-loop algorithm. Then we introduce a novel window-averaged bilevel local regret to evaluate performance in linearly varying environments, derive a lower bound for this regret, and show that our proposed algorithm is optimal. In addition, our algorithm supports an efficient single-loop structure.
In future work, we will extend our methods to OBO problems with stochastic noise. In addition, studying OBO problems with nonconvex inner-level objectives is another interesting and important direction, which could substantially broaden the applicability of OBO algorithms.

\bibliographystyle{plainnat}
\bibliography{references}


\clearpage
\appendix

\section{Experimental Details and Additional Results}

In this section, we provide additional experiments to further validate our theoretical results, together with detailed experimental settings for Section~\ref{sec:6}. 
Appendices~\ref{sec:detail_hyper_cleaning} and~\ref{sec:detail_loss_tuning} provide the detailed settings for real-data experiments in Section~\ref{sec:6}, including the parameter choices for all algorithms. 
In Appendix~\ref{sec:winavg_compare}, we further report the comparison results of the window-averaged algorithms in Table~\ref{tab:2} on the data hyper-cleaning experiment. 
Finally, in Appendix~\ref{sec:cifar10}, we provide supplementary results on the CIFAR10 dataset for the loss tuning experiment, validating the feasibility of our algorithm on high-dimensional tasks.

\subsection{Details on Online Data Hyper-cleaning}\label{sec:detail_hyper_cleaning}

Hyper-cleaning uses a bilevel optimization framework to assign weights to training samples with label noise, down-weighting noisy data and up-weighting clean data to improve model performance. In the online hyper-cleaning setting, where data arrives continuously, the optimization algorithm must be efficient, adaptive, and robust to changing data distributions, ensuring effective noise suppression and reliable real-time learning.

In this experiment, we utilize the MNIST dataset to simulate an online data hyper-cleaning task with time-varying label noise. The dataset contains 60,000 training images, which are split into a corrupted training set of 45,000 samples and a clean validation set of 15,000 samples. The model employed is a linear classifier with parameters $\theta \in \mathbb{R}^{d \times C}$, where $d=28 \times 28 = 784$ and $C=10$.

We introduce systematic label noise to the training stream by flipping the labels of a portion of samples to class `1`. The noise rate $\rho$ serves as the environmental variable and changes over time: $\rho=0.1$ for $t \in [1, 1000]$, $\rho=0.2$ for $t \in [1001, 2000]$, and $\rho=0.3$ for $t \in [2001, 3000]$. The online process runs for $T=3000$ steps with a batch size of 1000.

The inner objective is to train the classifier on the potentially noisy training data using a weighted cross-entropy loss, where the sample weights are learned to down-weight identified noisy samples. The inner function is defined as:
\begin{align*}
    g_t(\mathbf{w}, \theta) = \frac{1}{|\mathcal{D}_t^{\mathrm{tr}}|} \sum_{(x_i, y_i) \in \mathcal{D}_t^{\mathrm{tr}}} \ell_{\text{ce}}(\sigma(w_i)f(\theta; x_i), y_i),
\end{align*}
where $w_i$ is the learnable weight parameter associated with the training sample index $i$, and $\sigma(\cdot)$ is the sigmoid function.
The outer objective is to minimize the standard cross-entropy loss on the clean validation set with respect to the weight vector $\mathbf{w}$:
\begin{align*}
    F_t(\mathbf{w}) = \frac{1}{|\mathcal{D}_t^{\mathrm{val}}|} \sum_{(x_j, y_j) \in \mathcal{D}_t^{\mathrm{val}}} \ell_{\text{ce}}(f(\theta^*(\mathbf{w}); x_j), y_j).
\end{align*}

Regarding hyperparameters, we set the inner learning rate $\alpha=0.05$, the linear system learning rate $\beta=0.01$ and the outer learning rate $\gamma=0.05$. For AOBO, the inner-loop error tolerance is $\delta=0.1$. For OBBO, we fix the number of inner optimization steps at $K=5$. For SOGD, we set the momentum coefficient to $0.9$. The linear system iteration number are set to $5$ for all methods. 

\subsection{Details on Parametric Loss Tuning for Imbalanced Data}\label{sec:detail_loss_tuning}

Tuning for imbalanced data focuses on addressing the challenges posed by skewed class distributions in machine learning tasks. In such cases, models, especially deep neural networks with strong fitting capabilities, tend to be biased toward the majority class, leading to poor performance on the minority class.

In this experiment, we apply OBO to the parametric loss tuning task on imbalanced data streams. We use the MNIST dataset and simulate a time-varying class imbalance environment. The data stream consists of 400 steps with a batch size of 64, divided into 4 phases of 100 steps each. In each phase, the class distribution of the training data follows a geometric progression $n_c \propto \mu^c$ for class $c \in \{0, \dots, 9\}$, where the imbalance factor $\mu$ changes sequentially in $\{0.6, 0.7, 0.8, 0.9\}$. The validation set remains balanced throughout the process.

The model is a Convolutional Neural Network (CNN) with four convolutional blocks ($3\times3$ Conv--ReLU--$2\times2$ MaxPool--BatchNorm), followed by a linear classifier. The feature dimension before the linear layer is 64. The inner objective is to train the classifier by minimizing the logit-adjusted cross-entropy loss on the imbalanced training data, which is defined as:
\begin{align*}
    g_t(\mathbf{y}, \theta) = \frac{1}{|\mathcal{D}_t^{\mathrm{tr}}|} \sum_{(x_i, y_i) \in \mathcal{D}_t^{\mathrm{tr}}} \ell_{\text{adj}}(f(\theta; x_i), y_i; \mathbf{d}, \mathbf{l}),
\end{align*}
where $f(\theta; x)$ denotes the logits output by the model, and $\ell_{\text{adj}}$ is the adjusted loss with learnable parameters $\mathbf{y} = (\mathbf{d}, \mathbf{l})$:
\begin{align*}
    \ell_{\text{adj}}(\mathbf{z}, y; \mathbf{d}, \mathbf{l}) = -\log \frac{\exp(\sigma(d_y) z_y + l_y)}{\sum_{j=1}^{C} \exp(\sigma(d_j) z_j + l_j)}.
\end{align*}
Here, $\mathbf{d}$ controls the temperature-like scaling and $\mathbf{l}$ acts as a bias correction. The outer objective $F_t(\mathbf{y})$ is to minimize the standard cross-entropy loss on the balanced validation set.
\begin{align*}
    F_t(\mathbf{y}) = \frac{1}{|\mathcal{D}_t^{\mathrm{val}}|} \sum_{(x_i, y_i) \in \mathcal{D}_t^{\mathrm{val}}} \left(-\left[f(\theta^*(\mathbf{y});x_i)\right]_{y_i} + \log\sum_{j=1}^{C}e^{\left[f(\theta^*(\mathbf{y});x_i)\right]_{j}}\right).
\end{align*}

For all online bilevel optimizers, we set the inner learning rate $\alpha=0.1$ and the outer learning rate $\gamma=0.001$. We provide specific hyperparameter settings for the baselines as follows: for OBBO, we set the number of inner optimization steps and Neumann approximation terms to $K=5$; for SOBOW, we solve the linear system using Conjugate Gradient (CG) with $M = 5$ iterations; for AOBO, the inner loop adopts an adaptive stopping criterion with error threshold $\delta_t = 0.1$, and the linear system is solved via 5 CG iterations; for FSOBO and SOGD, we update the auxiliary variable for the linear system solution with a learning rate of $\beta=0.001$.

\subsection{Comparison of Window-Averaged Algorithms on Online Data Hyper-Cleaning}\label{sec:winavg_compare}

\begin{figure*}   
\begin{tabular}{ccc}
\includegraphics[width=0.3\linewidth]{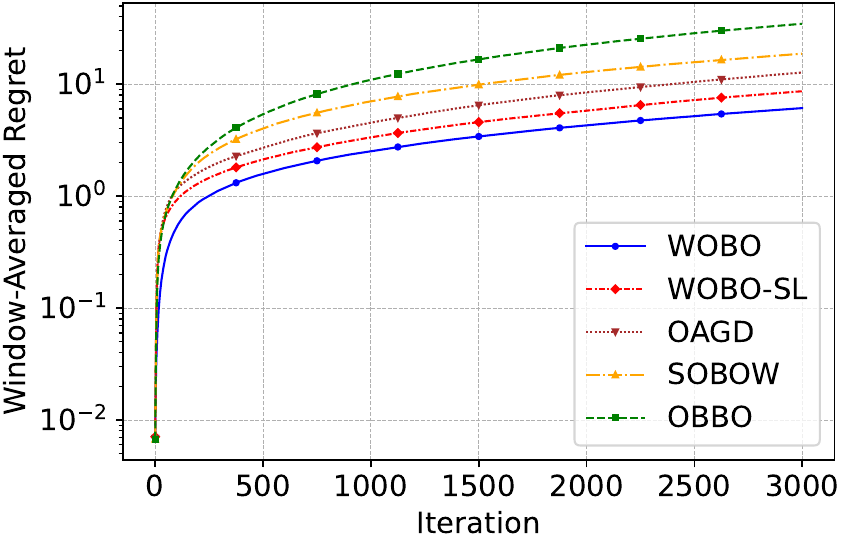} &
\includegraphics[width=0.3\linewidth]{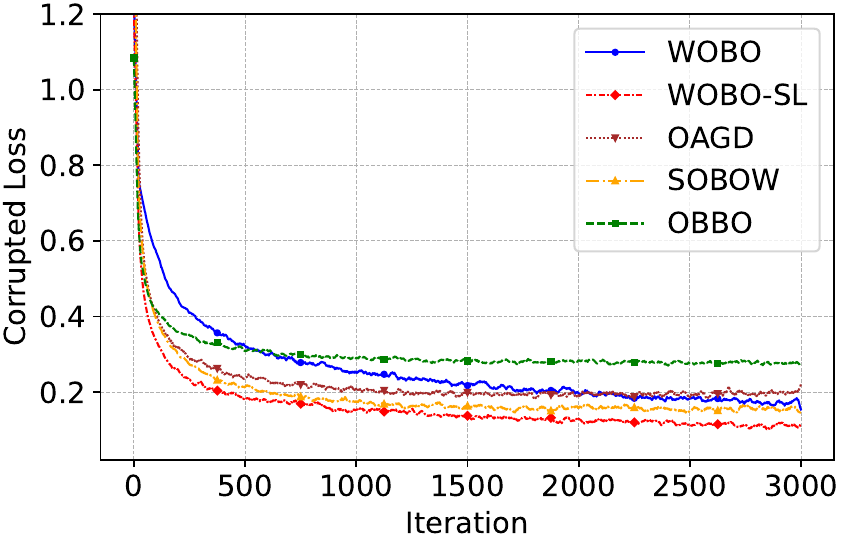} &
\includegraphics[width=0.3\linewidth]{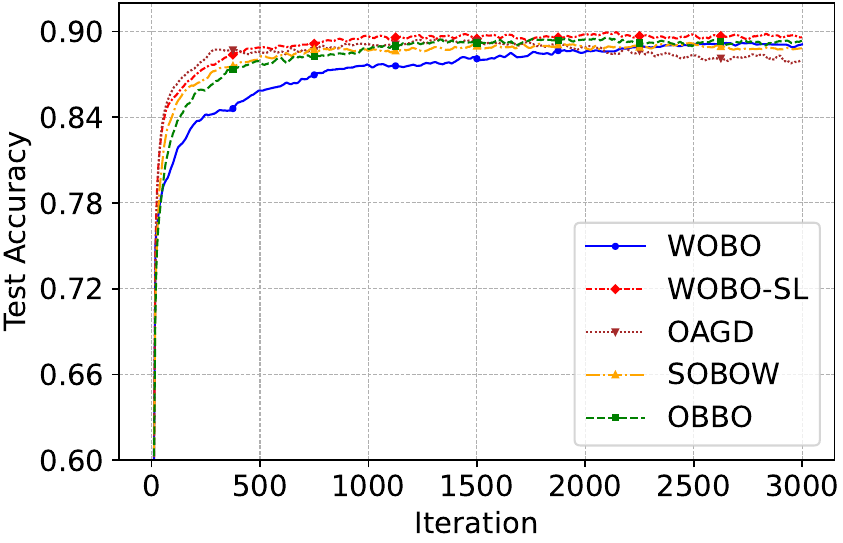} \\
(a) regret & 
(b) training loss & 
(c) test accuracy 
\end{tabular}
\caption{Performance of different window-averaegd algorithms on the online hyper-cleaning with $w=10$, $\eta=0.9$.}
\label{fig:hyperclean_winavg_compare}
\end{figure*}

In this section, we compare different window-averaged algorithms on the online data hyper-cleaning experiment to validate the theoretical results summarized in Table~\ref{tab:2}. Unlike the setting in Section~\ref{sec:hyper_cleaning}, we consider a nearly linearly varying environment, where the proportion of corrupted data increases every 20 steps, from an initial 10\% to 30\%. The results are shown in Figure~\ref{fig:hyperclean_winavg_compare}.

Here we set the inner learning rate to $\alpha=0.05$, the linear-system learning rate to $\beta=0.01$, and the outer learning rate to $\gamma=0.05$. For WOBO, the inner-loop error tolerance is set to $\delta=0.1$. For OBBO, the number of inner optimization steps is fixed as $K=5$. The number of linear-system iterations is set to $5$ for all methods. As shown in the figure, WOBO and WOBO-SL achieve the lowest and second-lowest cumulative window-averaged regret, respectively, while maintaining accuracy comparable to the other algorithms. In addition, WOBO-SL attains the lowest training loss, which we attribute to its model-parameter updates over the windowed data stream.

\subsection{Additional CIFAR10 Results for Parametric Loss Tuning under Imbalanced Data}\label{sec:cifar10}

\begin{figure*}

\begin{center}
\begin{tabular}{cc}
\includegraphics[width=0.45\linewidth]{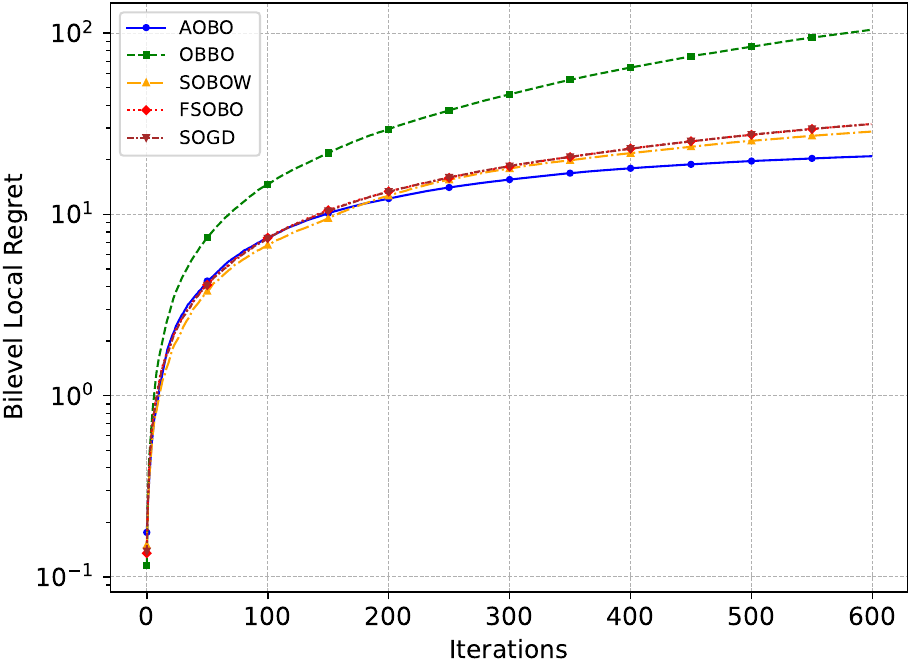} &
\includegraphics[width=0.45\linewidth]{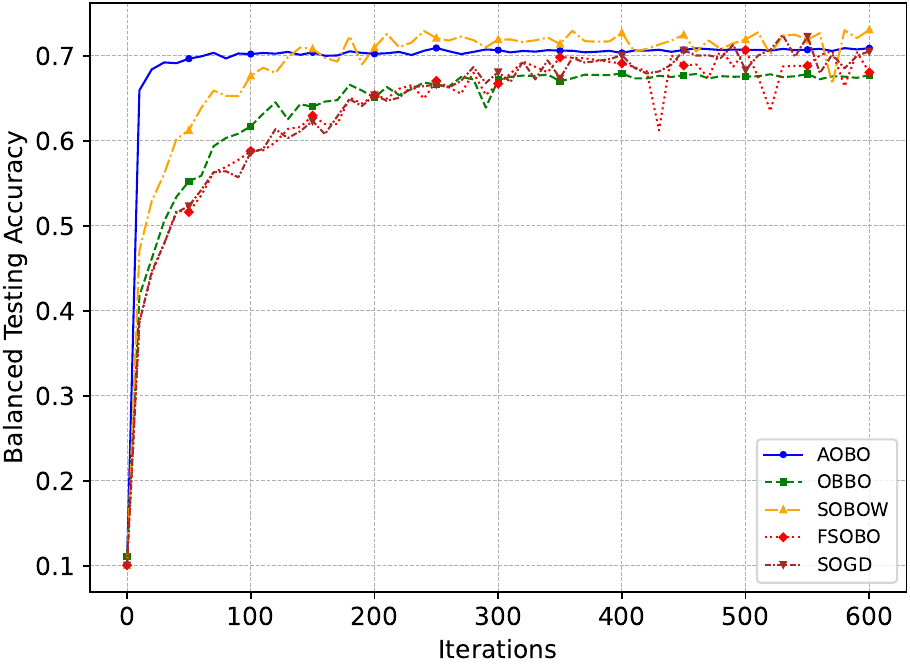} \\
(a) regret & 
(b) test accuracy 
\end{tabular}
\end{center}

\begin{center}
\begin{tabular}{cc}
\includegraphics[width=0.45\linewidth]{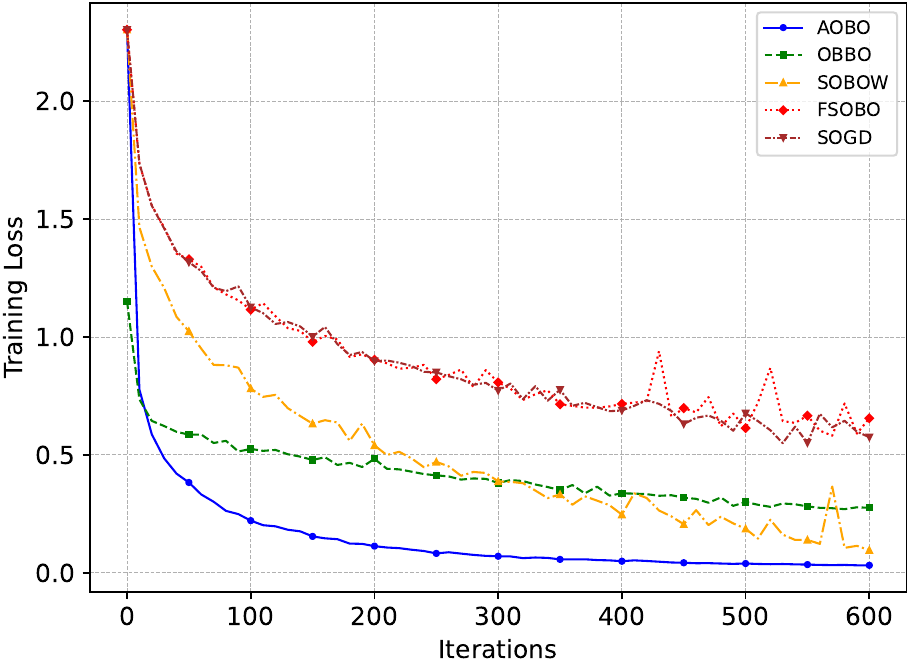} &
\includegraphics[width=0.45\linewidth]{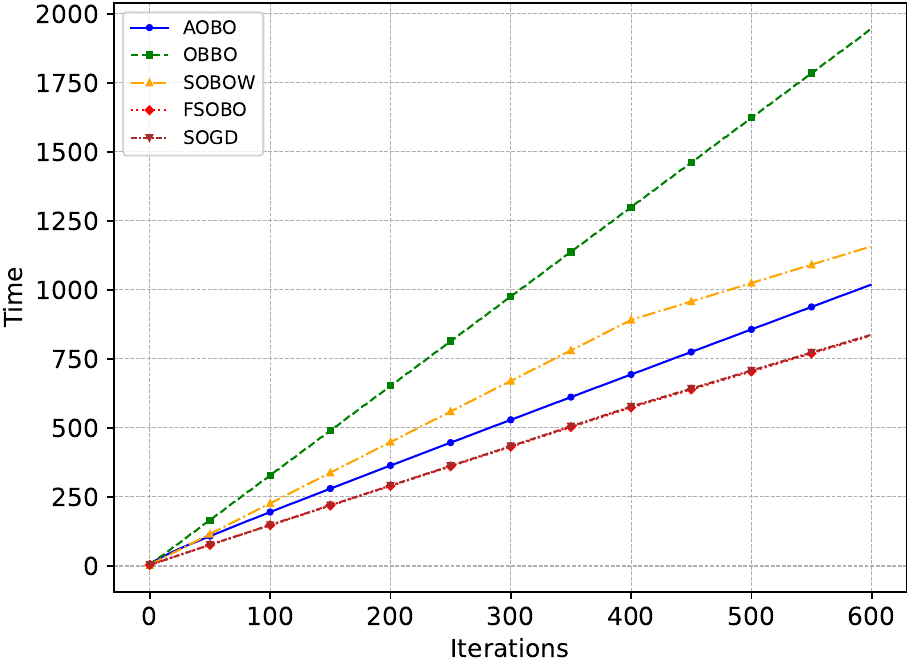} \\
(c) training loss & 
(d) runing time
\end{tabular}
\end{center}

\caption{Performance of different algorithms on the parametric loss tuning for imbalanced data under CIFAR10 dataset.}
\label{fig:losstuning_standard_cifar10}
\end{figure*}

To further evaluate the effectiveness of our algorithms on higher-dimensional datasets, we supplement the loss tuning experiment for imbalanced data with CIFAR-10 results, as shown in Figure~\ref{fig:losstuning_standard_cifar10}. We run $600$ steps, where each step uses data from a partitioned subset with batch size $3000$.

For all online bilevel optimizers, we set the inner learning rate to $\alpha=0.5$ and the outer learning rate to $\gamma=0.005$. The baseline-specific hyperparameters are set as follows. For OBBO, the number of inner optimization steps and Neumann approximation terms is set to $K=10$. For SOBOW, the linear system is solved by Conjugate Gradient (CG) with $M=15$ iterations. For AOBO, the inner loop uses an adaptive stopping criterion with error threshold $\delta_t=0.1$, and the linear system is solved using $5$ CG iterations. For FSOBO and SOGD, the auxiliary variable for solving the linear system is updated with learning rate $\beta=0.005$.
The results show that AOBO achieves strong performance in test accuracy, cumulative regret, and training loss, while requiring less running time than SOBOW and OBBO. This further validates the effectiveness of our algorithm.

\section{Proofs of Section~\ref{sec:3} under Standard Bilevel Local Regret}\label{sec:proof_A}

In this section, we provide the proofs of the regret upper bounds for AOBO and FSOBO, described in Algorithms~\ref{alg:AOBO} and~\ref{alg:FSOBO}, respectively, as well as the lower bound for the standard bilevel local regret. These proofs do not rely on any window-averaging strategy, but require certain sublinear variation assumptions.

In Section~\ref{sec:proof_A_lem}, we present several preliminary lemmas used throughout this section. 
Section~\ref{sec:proof_SOGO} proves the regret upper bound for AOBO, while Section~\ref{sec:proof_SOGO_lower} establishes the corresponding lower bound. 
Section~\ref{sec:proof_CTHO} provides the proof for OFSBO. 
In Sections~\ref{sec:proof_SOBOW} and~\ref{sec:proof_OBBO}, we further derive the upper bounds for SOBOW~\cite{lin2023non} and OBBO~\cite{bohne2024online} under the standard bilevel local regret. 
Finally, Section~\ref{sec:proof_SOGD} shows that SOGD~\cite{nazari2025stochastic} may achieve an improved upper bound under deterministic OBO.

\subsection{Preliminary Lemmas}\label{sec:proof_A_lem}

Under Assumptions~\ref{asm1} and~\ref{asm2}, the inner function $g_t(\mathbf{x}, \mathbf{y})$ is $L_{g,1}$-smooth and $\mu_g$-strongly convex with respect to $\mathbf{y}\in\mathbb{R}^{d_2}$, we present the additional properties used below.

For all $t\in[T]$, any $\mathbf{x}\in\mathcal{X}$, $\mathbf{y}_1$, $\mathbf{y}_2\in\mathbb{R}_{d_2}$, it holds that 
\begin{align}
    & \left\|\nabla_\mathbf{y}g_t(\mathbf{x}, \mathbf{y}_1) - \nabla_\mathbf{y}g_t(\mathbf{x}, \mathbf{y}_2)\right\| \leq L_{g,1}\left\|\mathbf{y}_1 - \mathbf{y}_2\right\| \label{eq:sc_1} \\
    & \left\|\nabla_\mathbf{y}g_t(\mathbf{x}, \mathbf{y}_1) - \nabla_\mathbf{y}g_t(\mathbf{x}, \mathbf{y}_2)\right\| \geq \mu_g\left\|\mathbf{y}_1 - \mathbf{y}_2\right\| \label{eq:sc_2},
\end{align}
also we have
\begin{align}
    \left\langle \nabla_\mathbf{y} g_t(\mathbf{x}, \mathbf{y}_1) - \nabla_\mathbf{y} g_t(\mathbf{x}, \mathbf{y}_2), \mathbf{y}_1 - \mathbf{y}_2 \right\rangle \geq& \frac{\mu_gL_{g,1}}{\mu_g + L_{g,1}}\left\|\mathbf{y}_1 - \mathbf{y}_2\right\|^2 \nonumber\\
    & + \frac{1}{\mu_g + L_{g,1}}\left\|\nabla_\mathbf{y} g_t(\mathbf{x}, \mathbf{y}_1) - \nabla_\mathbf{y}g_t(\mathbf{x}, \mathbf{y}_2)\right\|^2. \label{eq:sc_3}
\end{align}

\begin{lemma}[Lemma 16 in \cite{tarzanagh2024online}]\label{lem:L_F}
Under Assumptions~\ref{asm1},~\ref{asm2}, for all $t\in[T]$, any $\mathbf{x}_1$, $\mathbf{x}_2\in\mathcal{X}$, we have
\begin{align*}
    \|\nabla f_t(\mathbf{x}_1, \mathbf{y}_t^*(\mathbf{x}_1)) - \nabla f_t(\mathbf{x}_2, \mathbf{y}_t^*(\mathbf{x}_2))\| \leq L_F \|\mathbf{x}_1 - \mathbf{x}_2\|,
\end{align*}
where 
\begin{align*}
    L_F := L_{f,1} + \frac{L_{f,0}L_{g,2}}{\mu_g} +\frac{2L_{g,1}L_{f,1}}{\mu_g} + \frac{L_{g,1}^2L_{f,1}}{\mu_g^2} + \frac{2L_{f,0}L_{g,1}L_{g,2}}{\mu_g^2} + \frac{L_{f,0}L_{g,1}^2L_{g,2}}{\mu_g^3}.
\end{align*}
\end{lemma}

\begin{lemma}[Lemma 16 in \cite{tarzanagh2024online}]\label{lem:L_y}
Under Assumptions~\ref{asm1},~\ref{asm2}, for all $t\in[T]$, any $\mathbf{x}_1$, $\mathbf{x}_2\in\mathcal{X}$, we have
\begin{align*}
    \left\|\mathbf{y}_t^*(\mathbf{x}_1) - \mathbf{y}_t^*(\mathbf{x}_2)\right\| \leq \frac{L_{g,1}}{\mu_g} \|\mathbf{x}_1 - \mathbf{x}_2\|.
\end{align*}
\end{lemma}

\begin{lemma}\label{lem:v_cons}
Under Assumptions~\ref{asm1},~\ref{asm2}, for all $t\in[T]$ and $\mathbf{x}\in\mathcal{X}$, $\mathbf{y}\in\mathbb{R}^{d_2}$, we have
\begin{align*}
    \|\mathbf{v}_t^*(\mathbf{x}, \mathbf{y})\| \leq \frac{L_{f,0}}{\mu_g}, &\quad \text{where } \mathbf{v}_t^*(\mathbf{x}, \mathbf{y}) := \mathop{\rm argmin}_{\mathbf{v}\in\mathbb{R}^{d_2}} \Phi_t(\mathbf{x}, \mathbf{y}, \mathbf{v}).
\end{align*}
\end{lemma}
\begin{proof}
notice that
\begin{align*}
    \nabla_\mathbf{v} \Phi_t(\mathbf{x}, \mathbf{y}, \mathbf{v}_t^*(\mathbf{x}, \mathbf{y})
    ) = \nabla_{\mathbf{y}\mathbf{y}}^2 g_t(\mathbf{x}, \mathbf{y}) \mathbf{v}_t^*(\mathbf{x}, \mathbf{y}) - \nabla_\mathbf{y} f_t(\mathbf{x}, \mathbf{y}) = 0,
\end{align*}
thus we can see
\begin{align*}
    \mathbf{v}_t^*(\mathbf{x}, \mathbf{y}) =& \left[ \nabla_{\mathbf{y}\mathbf{y}}^2 g_t(\mathbf{x}, \mathbf{y}) \right]^{-1} \nabla_\mathbf{y} f_t(\mathbf{x}, \mathbf{y}) \\
    \|\mathbf{v}_t^*(\mathbf{x}, \mathbf{y})\| =& \left\| \left[ \nabla_{\mathbf{y}\mathbf{y}}^2 g_t(\mathbf{x}, \mathbf{y}) \right]^{-1} \nabla_\mathbf{y} f_t(\mathbf{x}, \mathbf{y}) \right\| \\
    \leq& \left\| \nabla_{\mathbf{y}\mathbf{y}}^2 g_t(\mathbf{x}, \mathbf{y}) \right\|^{-1} \left\|\nabla_\mathbf{y} f_t(\mathbf{x}, \mathbf{y})\right\| \leq \frac{L_{f,0}}{\mu_g} ,
\end{align*}
where the last inequality comes from Assumption~\ref{asm2}.
\end{proof}

The following lemma provides a fundamental bound for solving linear system $\Phi_t(\mathbf{x}, \mathbf{y}, \cdot)$ with respect to $\mathbf{v}$.

\begin{lemma}\label{lem:v_bound_1}
Under Assumptions~\ref{asm1},~\ref{asm2}, consider the inner iteration process of $\Phi_t(\mathbf{x}_t, \mathbf{y}_{t+1}, \cdot)$ in Algorithms~\ref{alg:AOBO} and~\ref{alg:FSOBO} for all $t\in[T]$, shown as
\begin{align*}
    \mathbf{v}_t^{m+1} \leftarrow \mathcal{P}_\mathcal{V} \left(\mathbf{v}_t^m - \beta\nabla_\mathbf{v}\Phi_t(\mathbf{x}_t, \mathbf{y}_{t+1}, \mathbf{v}_t^m)\right).
\end{align*}
If set $\beta \leq \frac{1}{L_{g,1}}$ and the above iteration performs $M_t$ times for time step $t$, we can obtain
\begin{align*}
    \left\|\mathbf{v}_{t+1} - \mathbf{v}_t^*(\mathbf{x}_t, \mathbf{y}_t^*(\mathbf{x}_t))\right\|^2 \leq \frac{32L_{f,0}^2}{\mu_g^2}(1-\beta\mu_g)^{M_t} + \left(\frac{4L_{f,1}^2}{\mu_g^2} + \frac{4L_{f,0}^2L_{g,2}^2}{\mu_g^4}\right)\left\|\mathbf{y}_{t+1} - \mathbf{y}_t^*(\mathbf{x}_t)\right\|^2.
\end{align*}
\end{lemma}
\begin{proof}
We can split the bound between $\mathbf{v}_{t+1}$ and $\mathbf{v}_t^*(\mathbf{x}_t, \mathbf{y}_t^*(\mathbf{x}_t))$ into two parts - the bound between $\mathbf{v}_{t+1}$ and $\mathbf{v}_t^*(\mathbf{x}_t, \mathbf{y}_{t+1})$, and the bound between $\mathbf{v}_t^*(\mathbf{x}_t, \mathbf{y}_{t+1})$ and $\mathbf{v}_t^*(\mathbf{x}_t, \mathbf{y}_t^*(\mathbf{x}_t))$. For simplicity, we define $\mathbf{w}_t^* := \mathbf{v}_t^*(\mathbf{x}_t, \mathbf{y}_{t+1})$ and $\mathbf{v}_t^* := \mathbf{v}_t^*(\mathbf{x}_t, \mathbf{y}_t^*(\mathbf{x}_t))$. It begins with
\begin{align}
    \|\mathbf{v}_{t+1} - \mathbf{v}_t^*\|^2 \leq 2\|\mathbf{v}_{t+1} - \mathbf{w}_t^*\|^2 + 2\|\mathbf{w}_t^* - \mathbf{v}_t^*\|^2. \label{A.6_1}
\end{align}
Let $\mathbf{w}_t$ represents
\begin{align*}
    \mathbf{w}_t^{m+1} = \mathbf{v}_t^m - \beta\widetilde{\nabla}\Phi_t(\mathbf{x}_t, \mathbf{y}_{t+1}, \mathbf{v}_t^m), \quad
    \mathbf{v}_t^{m+1} =\mathcal{P}_{\mathcal{V}}\left(\mathbf{w}_t^{m+1}\right).
\end{align*}
We begin with
\begin{align}
    \left\|\mathbf{v}_t^{m+1} - \mathbf{w}_t^*\right\|^2 \leq& \left\|\mathbf{w}_t^{m+1} - \mathbf{w}_t^*\right\|^2 \nonumber\\
    =& \left\|\mathbf{w}_t^{m+1} - \mathbf{v}_t^m + \mathbf{v}_t^m - \mathbf{w}_t^*\right\|^2 \nonumber\\
    =& \left\|\mathbf{w}_t^{m+1} - \mathbf{v}_t^m\right\|^2 + \left\|\mathbf{v}_t^m - \mathbf{w}_t^*\right\|^2 + 2\left\langle \mathbf{w}_t^{m+1} - \mathbf{v}_t^m, \mathbf{v}_t^m - \mathbf{w}_t^* \right\rangle \nonumber\\
    =& \beta^2\left\| \nabla \Phi_t(\mathbf{x}_t, \mathbf{y}_{t+1}, \mathbf{v}_t^m) \right\|^2 + \left\|\mathbf{v}_t^m {-} \mathbf{w}_t^*\right\|^2 - 2\beta\left\langle \nabla \Phi_t(\mathbf{x}_t, \mathbf{y}_{t+1}, \mathbf{v}_t^m), \mathbf{v}_t^m {-} \mathbf{w}_t^* \right\rangle. \label{A.6_2}
\end{align}
Now because of the $\mu_g$-strongly convexity and $L_{g,1}$-smoothness of $\Phi_t$, we have
\begin{align}
    \left\langle \nabla \Phi_t(\mathbf{x}_t, \mathbf{y}_{t+1}, \mathbf{v}_t^m), \mathbf{v}_t^m - \mathbf{w}_t^* \right\rangle \geq& \frac{1}{L_{g,1}} \left\| \nabla \Phi_t(\mathbf{x}_t, \mathbf{y}_{t+1}, \mathbf{v}_t^m) \right\|^2 \label{A.6_3} \\
    \left\langle \nabla \Phi_t(\mathbf{x}_t, \mathbf{y}_{t+1}, \mathbf{v}_t^m), \mathbf{v}_t^m - \mathbf{w}_t^* \right\rangle \geq& \mu_g\left\|\mathbf{v}_t^m - \mathbf{w}_t^*\right\|^2 \label{A.6_4}
\end{align}
and substitude (\ref{A.6_3}), (\ref{A.6_4}) into (\ref{A.6_2}) to have
\begin{align}
    &\left\|\mathbf{v}_t^{m+1} - \mathbf{w}_t^*\right\|^2 \nonumber\\
    \leq& \beta^2\left\| \nabla \Phi_t(\mathbf{x}_t, \mathbf{y}_{t+1}, \mathbf{v}_t^m) \right\|^2 + \left\|\mathbf{v}_t^m - \mathbf{w}_t^*\right\|^2 - \beta^2L_{g,1}\frac{1}{L_{g,1}}\left\|\nabla \Phi_t(\mathbf{x}_t, \mathbf{y}_{t+1}, \mathbf{v}_t^m)\right\|^2 \nonumber\\
    & - \left(2\beta - \beta^2L_{g,1}\right)\mu_g\left\|\mathbf{v}_t^m - \mathbf{w}_t^*\right\|^2 \nonumber\\
    \leq& \left(1-\beta\mu_g\right)\left\|\mathbf{v}_t^m - \mathbf{w}_t^*\right\|^2, 
\end{align}
and thus we have
\begin{align}
    \left\|\mathbf{v}_{t+1} - \mathbf{w}_t^*\right\|^2 \leq& (1-\beta\mu_g)^{M_t}\left\|\mathbf{v}_t - \mathbf{w}_t^*\right\|^2 \nonumber\\
    =& (1-\beta\mu_g)^{M_t}\left\|\mathbf{v}_t^m - \mathbf{w}_{t-1}^* + \mathbf{w}_{t-1}^* - \mathbf{w}_t^*\right\|^2 \nonumber\\
    \leq& 2(1-\beta\mu_g)^{M_t}\left\|\mathbf{v}_t^m - \mathbf{w}_{t-1}^*\right\|^2 + 2(1-\beta\mu_g)^{M_t}\left\|\mathbf{w}_{t-1}^* - \mathbf{w}_t^*\right\|^2. \label{A.6_5}
\end{align}
Now remind of Lemma~\ref{lem:v_cons}, we actually have
\begin{align*}
    \left\|\mathbf{v}_t^m - \mathbf{w}_{t-1}^*\right\|^2 \leq 2\left\|\mathbf{v}_t^m\right\|^2 + 2\left\|\mathbf{w}_{t-1}^*\right\|^2 \leq \frac{4L_{f,0}^2}{\mu_g^2}, \\
    \left\|\mathbf{w}_{t-1}^* - \mathbf{w}_t^*\right\|^2\leq 2\left\|\mathbf{w}_{t-1}^*\right\|^2 + 2\left\|\mathbf{w}_{t}^*\right\|^2 \leq \frac{4L_{f,0}^2}{\mu_g^2},
\end{align*}
and substitude into (\ref{A.6_5}), we obtain
\begin{align}
    \left\|\mathbf{v}_{t+1} - \mathbf{w}_t^*\right\|^2 \leq& \frac{16L_{f,0}^2}{\mu_g^2}(1-\beta\mu_g)^{M_t}. \label{A.6_6}
\end{align}

Now we consider about the bound of $\left\|\mathbf{w}_t^* - \mathbf{v}_t^*\right\|$, which begins with
\begin{align}
    &\|\mathbf{w}_t^* - \mathbf{v}_t^*\|^2 \nonumber \\
    =& \left\| \left( \nabla_{\mathbf{y}\mathbf{y}}^2 g_t(\mathbf{x}_t, \mathbf{y}_{t+1}) \right)^{-1} \nabla_\mathbf{y} f_t(\mathbf{x}_t, \mathbf{y}_{t+1}) - \left( \nabla_{\mathbf{y}\mathbf{y}}^2 g_t(\mathbf{x}_t, \mathbf{y}_t^*(\mathbf{x}_t)) \right)^{-1} \nabla_\mathbf{y} f_t(\mathbf{x}_t, \mathbf{y}_t^*(\mathbf{x}_t)) \right\|^2 \nonumber \\
    \leq& 2\left\|\left( \nabla_{\mathbf{y}\mathbf{y}}^2 g_t(\mathbf{x}_t, \mathbf{y}_{t+1}) \right)^{-1} \nabla_\mathbf{y} f_t(\mathbf{x}_t, \mathbf{y}_{t+1}) - \left( \nabla_{\mathbf{y}\mathbf{y}}^2 g_t(\mathbf{x}_t, \mathbf{y}_{t+1}) \right)^{-1} \nabla_\mathbf{y} f_t(\mathbf{x}_t, \mathbf{y}_t^*(\mathbf{x}_t)) \right\|^2 \nonumber \\
    & + 2\left\|\left( \nabla_{\mathbf{y}\mathbf{y}}^2 g_t(\mathbf{x}_t, \mathbf{y}_{t+1}) \right)^{-1} \nabla_\mathbf{y} f_t(\mathbf{x}_t, \mathbf{y}_t^*(\mathbf{x}_t)) - \left( \nabla_{\mathbf{y}\mathbf{y}}^2 g_t(\mathbf{x}_t, \mathbf{y}_t^*(\mathbf{x}_t)) \right)^{-1} \nabla_\mathbf{y} f_t(\mathbf{x}_t, \mathbf{y}_t^*(\mathbf{x}_t)) \right\|^2 \nonumber \\
    \leq& \frac{2L_{f,1}^2}{\mu_g^2} \|\mathbf{y}_{t+1} - \mathbf{y}_t^*(\mathbf{x}_t)\|^2 + 2L_{f,0}^2\left\|\left( \nabla_{\mathbf{y}\mathbf{y}}^2 g_t(\mathbf{x}_t, \mathbf{y}_{t+1}) \right)^{-1}  - \left( \nabla_{\mathbf{y}\mathbf{y}}^2 g_t(\mathbf{x}_t, \mathbf{y}_t^*(\mathbf{x}_t)) \right)^{-1} \right\|^2 \label{A.6_7} ,
\end{align}
where the last inequality comes from the smoothness of $f_t$ and $\nabla f_t$ being Lipschitz continuous, and for the second term, we can obtain
\begin{align}
    &\left\|\left( \nabla_{\mathbf{y}\mathbf{y}}^2 g_t(\mathbf{x}_t, \mathbf{y}_{t+1}) \right)^{-1} - \left( \nabla_{\mathbf{y}\mathbf{y}}^2 g_t(\mathbf{x}_t, \mathbf{y}_t^*(\mathbf{x}_t)) \right)^{-1} \right\|^2 \nonumber \\
    =& \left\| \left( \nabla_{\mathbf{y}\mathbf{y}}^2 g_t(\mathbf{x}_t, \mathbf{y}_{t+1}) \right)^{-1} \left( \nabla_{\mathbf{y}\mathbf{y}}^2 g_t(\mathbf{x}_t, \mathbf{y}_t^*(\mathbf{x}_t)) - \nabla_{\mathbf{y}\mathbf{y}}^2 g_t(\mathbf{x}_t, \mathbf{y}_{t+1}) \right) \left( \nabla_{\mathbf{y}\mathbf{y}}^2 g_t(\mathbf{x}_t, \mathbf{y}_t^*(\mathbf{x}_t)) \right)^{-1}\right\|^2 \nonumber \\
    =& \left\| \left( \nabla_{\mathbf{y}\mathbf{y}}^2 g_t(\mathbf{x}_t, \mathbf{y}_{t+1}) \right)^{-1} \right\|^2 \left\| \nabla_{\mathbf{y}\mathbf{y}}^2 g_t(\mathbf{x}_t, \mathbf{y}_t^*(\mathbf{x}_t)) {-} \nabla_{\mathbf{y}\mathbf{y}}^2 g_t(\mathbf{x}_t, \mathbf{y}_{t+1}) \right\|^2 \left\| \left( \nabla_{\mathbf{y}\mathbf{y}}^2 g_t(\mathbf{x}_t, \mathbf{y}_t^*(\mathbf{x}_t)) \right)^{-1} \right\|^2 \nonumber \\
    \leq& \frac{L_{g,2}^2}{\mu_g^4}\|\mathbf{y}_{t+1} - \mathbf{y}_t^*\|^2 \label{A.6_8} ,
\end{align}
where the last inequality comes from Assumption~\ref{asm2}, $\nabla_{\mathbf{y}\mathbf{y}} g_t$ is $L_{g,2}$-Lipschitz continuous, thus after subtituting (\ref{A.6_8}) into (\ref{A.6_7}), we obtain
\begin{align}
    \left\|\mathbf{w}_t^* - \mathbf{v}_t^*\right\|^2 \leq \left( \frac{2L_{f,1}^2}{\mu_g^2} + \frac{2L_{f,0}^2L_{g,2}^2}{\mu_g^4} \right) \left\|\mathbf{y}_{t+1} - \mathbf{y}_t^*\right\|^2. \label{A.6_9}
\end{align}

Finally, we substitude (\ref{A.6_6}) and (\ref{A.6_9}) into (\ref{A.6_1}) to obtain
\begin{align*}
    \left\|\mathbf{v}_{t+1} - \mathbf{v}_t^*\right\|^2 \leq \frac{32L_{f,0}^2}{\mu_g^2}(1-\beta\mu_g)^{M_t} + \left(\frac{4L_{f,1}^2}{\mu_g^2} + \frac{4L_{f,0}^2L_{g,2}^2}{\mu_g^4}\right)\left\|\mathbf{y}_{t+1} - \mathbf{y}_t^*(\mathbf{x}_t)\right\|^2,
\end{align*}
which finish the proof.
\end{proof}

Lemmas~\ref{lem:y_bound_2} and~\ref{lem:v_bound_2} are used to prove the upper bound of Algorithm~\ref{alg:FSOBO}. In particular, Lemma~\ref{lem:y_bound_2}, via the result in~\eqref{eq:sc_3}, further extracts a negative accumulation term involving the inner-level gradients. This term is then used in Lemma~\ref{lem:v_bound_2} to partially offset the additional error caused by solving the linear system $\Phi_t$ with only one iteration.

\begin{lemma}\label{lem:y_bound_2}
Under Assumptions~\ref{asm1},~\ref{asm2}, consider the inner iteration process of $g_t(\mathbf{x}_t, \cdot)$ in Algorithm~\ref{alg:FSOBO} with $\forall\alpha >0$ and $N=1$, we can obtain
\begin{align*}
    &\left\|\mathbf{y}_{t+1} - \mathbf{y}_t^*(\mathbf{x}_t)\right\|^2 \\
    \leq& \left(1 - \frac{\alpha\mu_gL_{g,1}}{\mu_g + L_{g,1}}\right)\left\|\mathbf{y}_t - \mathbf{y}_{t-1}^*(\mathbf{x}_{t-1})\right\|^2 - \left(\frac{2\alpha}{\mu_g + L_{g,1}} - \alpha^2\right)\left\|\nabla_\mathbf{y} g_t(\mathbf{x}_t, \mathbf{y}_t)\right\|^2 \\
    & + 2\left(1 - \frac{2\alpha\mu_gL_{g,1}}{\mu_g + L_{g,1}}\right)\left(1+\frac{\mu_g + L_{g,1}}{\alpha\mu_gL_{g,1}}\right)\left\|\mathbf{y}_{t-1}^*(\mathbf{x}_{t-1}) - \mathbf{y}_t^*(\mathbf{x}_{t-1})\right\|^2 \\
    & + 2\kappa_g^2\left(1 - \frac{2\alpha\mu_gL_{g,1}}{\mu_g + L_{g,1}}\right)\left(1+\frac{\mu_g + L_{g,1}}{\alpha\mu_gL_{g,1}}\right)\left\|\mathbf{x}_t - \mathbf{x}_{t-1}\right\|^2.
\end{align*}
\end{lemma}
\begin{proof}
It begins with
\begin{align}
    &\left\|\mathbf{y}_{t+1} - \mathbf{y}_t^*(\mathbf{x}_t)\right\|^2 \nonumber\\
    =& \left\|\mathbf{y}_t - \alpha\nabla_\mathbf{y} g_t(\mathbf{x}_t, \mathbf{y}_t) - \mathbf{y}_t^*(\mathbf{x}_t)\right\|^2 \nonumber\\
    =& \left\|\mathbf{y}_t - \mathbf{y}_t^*(\mathbf{x}_t)\right\|^2 + \alpha^2\left\|\nabla_\mathbf{y} g_t(\mathbf{x}_t, \mathbf{y}_t)\right\|^2 - 2\alpha\left\langle \nabla_\mathbf{y} g_t(\mathbf{x}_t, \mathbf{y}_t), \mathbf{y}_t - \mathbf{y}_t^*(\mathbf{x}_t) \right\rangle, \label{A.7_1}
\end{align}
due to (\ref{eq:sc_3}), let $\mathbf{x} = \mathbf{x}_t$ $\mathbf{y}_1 = \mathbf{y}_t$ and $\mathbf{y}_2 = \mathbf{y}_t^*(\mathbf{x}_t)$, we have
\begin{align}
    \left\langle \nabla_\mathbf{y} g_t(\mathbf{x}_t, \mathbf{y}_t), \mathbf{y}_t - \mathbf{y}_t^*(\mathbf{x}_t) \right\rangle \geq \frac{\mu_gL_{g,1}}{\mu_g + L_{g,1}}\left\|\mathbf{y}_t - \mathbf{y}_t^*(\mathbf{x}_t)\right\|^2 + \frac{1}{\mu_g + L_{g,1}}\left\|\nabla_\mathbf{y} g_t(\mathbf{x}_t, \mathbf{y}_t)\right\|^2, \label{A.7_2}
\end{align}
and substitute (\ref{A.7_2}) into (\ref{A.7_1}), we obtain
\begin{align*}
    &\left\|\mathbf{y}_{t+1} - \mathbf{y}_t^*(\mathbf{x}_t)\right\|^2 \\
    \leq& \left(1 - \frac{2\alpha\mu_gL_{g,1}}{\mu_g + L_{g,1}}\right)\left\|\mathbf{y}_t - \mathbf{y}_t^*(\mathbf{x}_t)\right\|^2 - \left(\frac{2\alpha}{\mu_g + L_{g,1}} - \alpha^2\right)\left\|\nabla_\mathbf{y} g_t(\mathbf{x}_t, \mathbf{y}_t)\right\|^2 \\
    \leq& \left(1 - \frac{2\alpha\mu_gL_{g,1}}{\mu_g + L_{g,1}}\right)(1+\lambda)\left\|\mathbf{y}_t - \mathbf{y}_{t-1}^*(\mathbf{x}_{t-1})\right\|^2 - \left(\frac{2\alpha}{\mu_g + L_{g,1}} - \alpha^2\right)\left\|\nabla_\mathbf{y} g_t(\mathbf{x}_t, \mathbf{y}_t)\right\|^2 \\
    & + 2\left(1 - \frac{2\alpha\mu_gL_{g,1}}{\mu_g + L_{g,1}}\right)\left(1+\frac{1}{\lambda}\right)\left\|\mathbf{y}_{t-1}^*(\mathbf{x}_{t-1}) - \mathbf{y}_t^*(\mathbf{x}_{t-1})\right\|^2 \\
    & + 2\kappa_g^2\left(1 - \frac{2\alpha\mu_gL_{g,1}}{\mu_g + L_{g,1}}\right)\left(1+\frac{1}{\lambda}\right)\left\|\mathbf{x}_t - \mathbf{x}_{t-1}\right\|^2,
\end{align*}
let $\lambda = \frac{\alpha\mu_gL_{g,1}}{\mu_g + L_{g,1}}$, it follows that
\begin{align*}
    &\left\|\mathbf{y}_{t+1} - \mathbf{y}_t^*(\mathbf{x}_t)\right\|^2 \\
    \leq& \left(1 - \frac{\alpha\mu_gL_{g,1}}{\mu_g + L_{g,1}}\right)\left\|\mathbf{y}_t - \mathbf{y}_{t-1}^*(\mathbf{x}_{t-1})\right\|^2 - \left(\frac{2\alpha}{\mu_g + L_{g,1}} - \alpha^2\right)\left\|\nabla_\mathbf{y} g_t(\mathbf{x}_t, \mathbf{y}_t)\right\|^2 \\
    & + 2\left(1 - \frac{2\alpha\mu_gL_{g,1}}{\mu_g + L_{g,1}}\right)\left(1+\frac{\mu_g + L_{g,1}}{\alpha\mu_gL_{g,1}}\right)\left\|\mathbf{y}_{t-1}^*(\mathbf{x}_{t-1}) - \mathbf{y}_t^*(\mathbf{x}_{t-1})\right\|^2 \\
    & + 2\kappa_g^2\left(1 - \frac{2\alpha\mu_gL_{g,1}}{\mu_g + L_{g,1}}\right)\left(1+\frac{\mu_g + L_{g,1}}{\alpha\mu_gL_{g,1}}\right)\left\|\mathbf{x}_t - \mathbf{x}_{t-1}\right\|^2.
\end{align*}
\end{proof}

\begin{lemma}\label{lem:v_bound_2}
Under Assumptions~\ref{asm1},~\ref{asm2}, consider the inner iteration process of $\Phi_t(\mathbf{x}_t, \mathbf{y}_{t+1}, \cdot)$ in Algorithm~\ref{alg:FSOBO} with $M_t \equiv M = 1$, if set $\beta\leq\frac{1}{L_{g,1}}$ we can obtain
\begin{align*}
    &\left\|\mathbf{v}_{t+1} - \mathbf{w}_t^*\right\|^2 \\
    \leq& \left(1-\frac{\beta\mu_g}{2}\right)\left\|\mathbf{v}_t - \mathbf{w}_{t-1}^*\right\|^2 + \frac{4(1-\beta\mu_g)\left(1+\frac{2}{\beta\mu_g}\right)}{\mu_g^2}\left\|\nabla_\mathbf{y} f_{t-1}(\mathbf{x}_t, \mathbf{y}_{t+1}) - \nabla_\mathbf{y} f_t(\mathbf{x}_t, \mathbf{y}_{t+1})\right\|^2 \\
    & + \frac{4L_{f,0}^2(1-\beta\mu_g)\left(1+\frac{2}{\beta\mu_g}\right)}{\mu_g^4}\left\|\nabla_{\mathbf{y}\mathbf{y}}^2 g_{t-1}(\mathbf{x}_t, \mathbf{y}_{t+1}) - \nabla_{\mathbf{y}\mathbf{y}}^2 g_t(\mathbf{x}_t, \mathbf{y}_{t+1})\right\|^2 \\
    & +  \left(\frac{4L_{f,1}^2}{\mu_g^2} + \frac{4L_{f,0}^2L_{g,2}^2}{\mu_g^4}\right)(1-\beta\mu_g)\left(1+\frac{2}{\beta\mu_g}\right)\left(\|\mathbf{x}_{t-1} - \mathbf{x}_t\|^2 + \|\mathbf{y}_t - \mathbf{y}_{t+1}\|^2\right), 
\end{align*}
here for simplity, we define
\begin{align*}
    \mathbf{w}_t^* {:=} \mathbf{v}_t^*(\mathbf{x}_t, \mathbf{y}_{t+1}) {=}  \mathop{\rm argmin}_{\mathbf{v}\in\mathcal{V}}\Phi_t(\mathbf{x}_t, \mathbf{y}_{t+1}, \mathbf{v}), \quad \mathbf{v}_t^* {:=} \mathbf{v}_t^*(\mathbf{x}_t, \mathbf{y}_t^*(\mathbf{x}_t)) {=}  \mathop{\rm argmin}_{\mathbf{v}\in\mathcal{V}}\Phi_t(\mathbf{x}_t, \mathbf{y}_{t+1}, \mathbf{v}).
\end{align*}
\end{lemma}
\begin{proof}
Similar to the proof in Lemma~\ref{lem:v_bound_1}, we divide $\left\|\mathbf{v}_{t+1} - \mathbf{v}_t^*\right\|$ into $\left\|\mathbf{v}_{t+1} - \mathbf{w}_t^*\right\|$ and $\left\|\mathbf{w}_t^* - \mathbf{v}_t^*\right\|^2$ and handle them separately.

It begins with (\ref{A.6_5}) in Lemma~\ref{lem:v_bound_1} but with $M_t \equiv 1$ in single-loop situation, which shown as:
\begin{align}
    \left\|\mathbf{v}_{t+1} - \mathbf{w}_t^*\right\| \leq& \left(1-\beta\mu_g\right)\left\|\mathbf{v}_t - \mathbf{w}_t^*\right\|^2 \nonumber\\
    \leq& (1-\beta\mu_g)(1+\lambda)\left\|\mathbf{v}_t - \mathbf{w}_{t-1}^*\right\|^2 + (1-\beta\mu_g)\left(1+\frac{1}{\lambda}\right)\left\|\mathbf{w}_{t-1}^* - \mathbf{w}_t^*\right\|^2 \nonumber\\
    \leq& \left(1-\frac{\beta\mu_g}{2}\right)\left\|\mathbf{v}_t - \mathbf{w}_{t-1}^*\right\|^2 + (1-\beta\mu_g)\left(1+\frac{2}{\beta\mu_g}\right)\left\|\mathbf{w}_{t-1}^* - \mathbf{w}_t^*\right\|^2, \label{A.8_1}
\end{align}
where $\lambda = \frac{\beta\mu_g}{2}$. Now we process $\left\|\mathbf{w}_{t-1}^* - \mathbf{w}_t^*\right\|^2$ in two parts separately, we have
\begin{align}
    \left\|\mathbf{w}_{t-1}^* - \mathbf{w}_t^*\right\|^2 =& \left\|\mathbf{v}_{t-1}^*(\mathbf{x}_{t-1}, \mathbf{y}_t) - \mathbf{v}_t^*(\mathbf{x}_t, \mathbf{y}_{t+1})\right\|^2 \nonumber\\
    \leq& 2\underbrace{\left\|\mathbf{v}_{t-1}^*(\mathbf{x}_{t-1}, \mathbf{y}_t) - \mathbf{v}_{t-1}^*(\mathbf{x}_{t}, \mathbf{y}_{t+1})\right\|^2}_{(h.1)} + 2\underbrace{\left\|\mathbf{v}_{t-1}^*(\mathbf{x}_t, \mathbf{y}_{t+1}) - \mathbf{v}_t^*(\mathbf{x}_t, \mathbf{y}_{t+1})\right\|^2}_{(h.2)}. \label{A.8_2}
\end{align}
For $(h.1)$, we have
\begin{align}
    &(h.1) \nonumber\\
    =& \left\|\left[\nabla_{\mathbf{y}\mathbf{y}}^2 g_{t-1}(\mathbf{x}_{t-1}, \mathbf{y}_t)\right]^{-1}\nabla_\mathbf{y} f_{t-1}(\mathbf{x}_{t-1}, \mathbf{y}_t) {-} \left[\nabla_{\mathbf{y}\mathbf{y}}^2 g_{t-1}(\mathbf{x}_{t}, \mathbf{y}_{t+1})\right]^{-1}\nabla_\mathbf{y} f_{t-1}(\mathbf{x}_{t}, \mathbf{y}_{t+1})\right\|^2 \nonumber\\
    \leq& 2\left\|\left[\nabla_{\mathbf{y}\mathbf{y}}^2 g_{t-1}(\mathbf{x}_{t-1}, \mathbf{y}_t)\right]^{-1}\nabla_\mathbf{y} f_{t-1}(\mathbf{x}_{t-1}, \mathbf{y}_t) {-} \left[\nabla_{\mathbf{y}\mathbf{y}}^2 g_{t-1}(\mathbf{x}_{t-1}, \mathbf{y}_t)\right]^{-1}\nabla_\mathbf{y} f_{t-1}(\mathbf{x}_{t}, \mathbf{y}_{t+1})\right\|^2 \nonumber\\
    & + 2\left\|\left[\nabla_{\mathbf{y}\mathbf{y}}^2 g_{t-1}(\mathbf{x}_{t-1}, \mathbf{y}_t)\right]^{-1}\nabla_\mathbf{y} f_{t-1}(\mathbf{x}_{t}, \mathbf{y}_{t+1}) {-} \left[\nabla_{\mathbf{y}\mathbf{y}}^2 g_{t-1}(\mathbf{x}_{t}, \mathbf{y}_{t+1})\right]^{-1}\nabla_\mathbf{y} f_{t-1}(\mathbf{x}_{t}, \mathbf{y}_{t+1})\right\|^2 \nonumber\\
    \leq& \left(\frac{2L_{f,1}^2}{\mu_g^2} + \frac{2L_{f,0}^2L_{g,2}^2}{\mu_g^4}\right)\left(\left\|\mathbf{x}_{t-1} - \mathbf{x}_t\right\|^2 + \left\|\mathbf{y}_t - \mathbf{y}_{t+1}\right\|^2\right), \label{A.8_3}
\end{align}
where the last inequality comes from the Assumption~\ref{asm2}. For $(h.2)$, we have
\begin{align}
    &(h.2) \nonumber\\
    =& \left\|\left[\nabla_{\mathbf{y}\mathbf{y}}^2 g_{t-1}(\mathbf{x}_{t}, \mathbf{y}_{t+1})\right]^{-1}\nabla_\mathbf{y} f_{t-1}(\mathbf{x}_{t}, \mathbf{y}_{t+1}) {-} \left[\nabla_{\mathbf{y}\mathbf{y}}^2 g_t(\mathbf{x}_{t}, \mathbf{y}_{t+1})\right]^{-1}\nabla_\mathbf{y} f_t(\mathbf{x}_{t}, \mathbf{y}_{t+1})\right\|^2 \nonumber\\
    \leq& 2\left\|\left[\nabla_{\mathbf{y}\mathbf{y}}^2 g_{t-1}(\mathbf{x}_{t}, \mathbf{y}_{t+1})\right]^{-1}\nabla_\mathbf{y} f_{t-1}(\mathbf{x}_{t}, \mathbf{y}_{t+1}) {-} \left[\nabla_{\mathbf{y}\mathbf{y}}^2 g_t(\mathbf{x}_{t}, \mathbf{y}_{t+1})\right]^{-1}\nabla_\mathbf{y} f_{t-1}(\mathbf{x}_{t}, \mathbf{y}_{t+1})\right\|^2 \nonumber\\
    & + 2\left\|\left[\nabla_{\mathbf{y}\mathbf{y}}^2 g_t(\mathbf{x}_{t}, \mathbf{y}_{t+1})\right]^{-1}\nabla_\mathbf{y} f_{t-1}(\mathbf{x}_{t}, \mathbf{y}_{t+1}) {-} \left[\nabla_{\mathbf{y}\mathbf{y}}^2 g_t(\mathbf{x}_{t}, \mathbf{y}_{t+1})\right]^{-1}\nabla_\mathbf{y} f_t(\mathbf{x}_{t}, \mathbf{y}_{t+1})\right\|^2 \nonumber\\
    \leq& \frac{2L_{f,0}^2}{\mu_g^4}\left\|\nabla_{\mathbf{y}\mathbf{y}}^2 g_{t-1}(\mathbf{x}_t, \mathbf{y}_{t+1}) - \nabla_{\mathbf{y}\mathbf{y}}^2 g_t(\mathbf{x}_t, \mathbf{y}_{t+1})\right\|^2 \nonumber\\
    & + \frac{2}{\mu_g^2}\left\|\nabla_\mathbf{y} f_{t-1}(\mathbf{x}_t, \mathbf{y}_{t+1}) - \nabla_\mathbf{y} f_t(\mathbf{x}_t, \mathbf{y}_{t+1})\right\|^2. \label{A.8_4}
\end{align}
Substitute (\ref{A.8_3}) and (\ref{A.8_4}) into (\ref{A.8_2}), we finally have
\begin{align}
    \left\|\mathbf{w}_{t-1}^* - \mathbf{w}_t^*\right\|^2 
    \leq& \left(\frac{4L_{f,1}^2}{\mu_g^2} + \frac{4L_{f,0}^2L_{g,2}^2}{\mu_g^4}\right)\left(\|\mathbf{x}_{t-1} - \mathbf{x}_t\|^2 + \|\mathbf{y}_t - \mathbf{y}_{t+1}\|^2\right) \\
    & + \frac{4L_{f,0}^2}{\mu_g^4}\left\|\nabla_{\mathbf{y}\mathbf{y}}^2 g_{t-1}(\mathbf{x}_t, \mathbf{y}_{t+1}) - \nabla_{\mathbf{y}\mathbf{y}}^2 g_t(\mathbf{x}_t, \mathbf{y}_{t+1})\right\|^2 \nonumber\\
    & + \frac{4}{\mu_g^2}\left\|\nabla_\mathbf{y} f_{t-1}(\mathbf{x}_t, \mathbf{y}_{t+1}) - \nabla_\mathbf{y} f_t(\mathbf{x}_t, \mathbf{y}_{t+1})\right\|^2, \label{A.8_5}
\end{align}
and substitute (\ref{A.8_5}) into (\ref{A.8_1}) to obtain
\begin{align}
    &\left\|\mathbf{v}_{t+1} - \mathbf{w}_t^*\right\|^2 \nonumber\\
    \leq& \left(1-\frac{\beta\mu_g}{2}\right)\left\|\mathbf{v}_t - \mathbf{w}_{t-1}^*\right\|^2 + \frac{4(1-\beta\mu_g)\left(1+\frac{2}{\beta\mu_g}\right)}{\mu_g^2}\left\|\nabla_\mathbf{y} f_{t-1}(\mathbf{x}_t, \mathbf{y}_{t+1}) - \nabla_\mathbf{y} f_t(\mathbf{x}_t, \mathbf{y}_{t+1})\right\|^2 \nonumber\\
    & + \frac{4L_{f,0}^2(1-\beta\mu_g)\left(1+\frac{2}{\beta\mu_g}\right)}{\mu_g^4}\left\|\nabla_{\mathbf{y}\mathbf{y}}^2 g_{t-1}(\mathbf{x}_t, \mathbf{y}_{t+1}) - \nabla_{\mathbf{y}\mathbf{y}}^2 g_t(\mathbf{x}_t, \mathbf{y}_{t+1})\right\|^2 \nonumber\\
    & +  \left(\frac{4L_{f,1}^2}{\mu_g^2} + \frac{4L_{f,0}^2L_{g,2}^2}{\mu_g^4}\right)(1-\beta\mu_g)\left(1+\frac{2}{\beta\mu_g}\right)\left(\|\mathbf{x}_{t-1} - \mathbf{x}_t\|^2 + \|\mathbf{y}_t - \mathbf{y}_{t+1}\|^2\right).
\end{align}
Here we finish the proof.
\end{proof}

\begin{lemma}\label{lem:2M+V_T}
Under Assumption~\ref{asm3}, we have
\begin{align*}
    \sum_{t=1}^{T} \left( f_t\left(\mathbf{x}_t, \mathbf{y}^*_t(\mathbf{x}_t) \right) - f_t\left(\mathbf{x}_{t+1}, \mathbf{y}^*_t(\mathbf{x}_{t+1}) \right) \right) \leq 2Q + V_T
\end{align*}
where $V_T$ is defined in (\ref{pathV}).
\end{lemma}
\begin{proof}
it naturally holds that
\begin{align*}
    &\sum_{t=1}^{T} \left( f_t\left(\mathbf{x}_t, \mathbf{y}^*_t(\mathbf{x}_t) \right) - f_t\left(\mathbf{x}_{t+1}, \mathbf{y}^*_t(\mathbf{x}_{t+1}) \right) \right) \\
    =& \sum_{t=1}^T\left(f_t(\mathbf{x}_t, \mathbf{y}_t^*(\mathbf{x}_t)) - f_{t+1}(\mathbf{x}_{t+1}, \mathbf{y}_{t+1}^*(\mathbf{x}_{t+1}))\right) \\
    & + \sum_{t=1}^T\left(f_{t+1}(\mathbf{x}_{t+1}, \mathbf{y}_{t+1}^*(\mathbf{x}_{t+1})) - f_t\left(\mathbf{x}_{t+1}, \mathbf{y}^*_t(\mathbf{x}_{t+1}) \right)\right) \\
    \leq& f_1(\mathbf{x}_1, \mathbf{y}_1^*(\mathbf{x}_1)) - f_{T+1}(\mathbf{x}_{T+1}, \mathbf{y}_{T+1}^*(\mathbf{x}_{T+1})) + V_T \leq 2Q+V_T,
\end{align*}
where the last inequality comes from Assumption~\ref{asm3}.
\end{proof}

The following lemma is fundamental for the proofs of algorithms~\ref{alg:AOBO},~\ref{alg:FSOBO} and~\ref{alg:CIGO}. It reveals that bilevel local regret defined in (\ref{eq:reg}) can be upper bounded to the accumulation of function value differences on the sequence $\{\mathbf{x}\}_{t=1}^{T+1}$ and the accumulation of hypergradient approximation error. The former has been dealt with by lemma~\ref{lem:2M+V_T} above, while the key to the different algorithms lies in the handling of hypergradient error.

\begin{lemma}\label{lem:begin}
Under Assumptions~\ref{asm1},~\ref{asm2}, for all $t\in[T]$, we can obtain
\begin{align*}
    &\sum_{t=1}^{T} \left\| \mathcal{G}_\mathcal{X}\left( \mathbf{x}_t, \nabla f_t(\mathbf{x}_t, \mathbf{y}_t^*(\mathbf{x}_t)), \gamma \right) \right\|^2 \\
    \leq& \frac{2}{\gamma\theta} \sum_{t=1}^{T} \left( f_t(\mathbf{x}_t, \mathbf{y}_t^*(\mathbf{x}_t)) - f_t(\mathbf{x}_{t+1}, \mathbf{y}_t^*(\mathbf{x}_{t+1})) \right) \\
    & + \left(2+\frac{1}{\lambda\theta}\right) \sum_{t=1}^{T} \left\| \nabla f_t(\mathbf{x}_t, \mathbf{y}_t^*(\mathbf{x}_t)) - \widetilde{\nabla} f_t(\mathbf{x}_t, \mathbf{y}_{t+1}, \mathbf{v}_{t+1}) \right\|^2,
\end{align*}
where \( \theta = 1 - \frac{\lambda}{2} - \frac{\gamma L_F}{2} \) for some constant $\lambda > 0$.
\end{lemma}
\begin{proof}
We begin with the smoothness of \( f_t(\mathbf{x}, \mathbf{y}_t^*(\mathbf{x})) \), given by Lemma~\ref{lem:L_F}
\begin{align*}
    &f_t\left(\mathbf{x}_{t+1}, \mathbf{y}_t^*(\mathbf{x}_{t+1})\right) - f_t\left(\mathbf{x}_t, \mathbf{y}_t^*(\mathbf{x}_t)\right) \\
    \leq& \left\langle \nabla f_t(\mathbf{x}_t, \mathbf{y}_t^*(\mathbf{x}_t)), \mathbf{x}_{t+1} - \mathbf{x}_t \right\rangle + \frac{L_F}{2}\left\|\mathbf{x}_{t+1} - \mathbf{x}_t\right\|^2 \\
    =& - \gamma\left\langle \nabla f_t(\mathbf{x}_t, \mathbf{y}_t^*(\mathbf{x}_t)), \mathcal{G}_\mathcal{X}\left(\mathbf{x}_t, \widetilde{\nabla}f_t(\mathbf{x}_t, \mathbf{y}_{t+1}, \mathbf{v}_{t+1}), \gamma\right) \right\rangle \\
    & + \frac{\gamma^2L_F}{2}\left\|\mathcal{G}_\mathcal{X}\left(\mathbf{x}_t, \widetilde{\nabla}f_t(\mathbf{x}_t, \mathbf{y}_{t+1}, \mathbf{v}_{t+1}), \gamma\right)\right\|^2 \\
    =& - \gamma\left\langle \widetilde{\nabla}f_t(\mathbf{x}_t, \mathbf{y}_{t+1}, \mathbf{v}_{t+1}), \mathcal{G}_\mathcal{X}\left(\mathbf{x}_t, \widetilde{\nabla}f_t(\mathbf{x}_t, \mathbf{y}_{t+1}, \mathbf{v}_{t+1}), \gamma\right) \right\rangle \\
    & + \gamma \left\langle \widetilde{\nabla} f_t (\mathbf{x}_t, \mathbf{y}_{t+1}, \mathbf{v}_{t+1}) - \nabla f_t (\mathbf{x}_t, \mathbf{y}_t^*(\mathbf{x}_t)), \mathcal{G}_\mathcal{X}\left(\mathbf{x}_t, \widetilde{\nabla}f_t(\mathbf{x}_t, \mathbf{y}_{t+1}, \mathbf{v}_{t+1}), \gamma\right) \right\rangle \\
    & + \frac{\gamma^2L_F}{2}\left\|\mathcal{G}_\mathcal{X}\left(\mathbf{x}_t, \widetilde{\nabla}f_t(\mathbf{x}_t, \mathbf{y}_{t+1}, \mathbf{v}_{t+1}), \gamma\right)\right\|^2,
\end{align*}
then with Young's inequality,
\begin{align}
    &f_t(\mathbf{x}_{t+1}, \mathbf{y}_t^*(\mathbf{x}_{t+1})) - f_t(\mathbf{x}_t, \mathbf{y}_t^*(\mathbf{x}_t)) \nonumber\\
    \leq& -\gamma\left\|\mathcal{G}_\mathcal{X}\left(\mathbf{x}_t, \widetilde{\nabla} f_t(\mathbf{x}_t, \mathbf{y}_{t+1}, \mathbf{v}_{t+1}), \gamma\right)\right\|^2 {+} \frac{\gamma\lambda}{2}\left\|\mathcal{G}_\mathcal{X}\left(\mathbf{x}_t, \widetilde{\nabla}f_t(\mathbf{x}_t, \mathbf{y}_{t+1}, \mathbf{v}_{t+1}), \gamma\right)\right\|^2 \nonumber\\
    & + \frac{\gamma}{2\lambda}\left\| \widetilde{\nabla} f_t(\mathbf{x}_t, \mathbf{y}_{t+1}, \mathbf{v}_{t+1}) {-} \nabla f_t(\mathbf{x}_t, \mathbf{y}_t^*(\mathbf{x}_t)) \right\|^2 {+} \frac{\gamma^2L_F}{2}\left\|\mathcal{G}_\mathcal{X}\left(\mathbf{x}_t, \widetilde{\nabla}f_t(\mathbf{x}_t, \mathbf{y}_{t+1}, \mathbf{v}_{t+1}), \gamma\right)\right\|^2 \nonumber\\
    \leq& -\gamma\left( 1 - \frac{\lambda}{2} - \frac{\gamma L_F}{2} \right)\left\|\mathcal{G}_\mathcal{X}\left(\mathbf{x}_t, \widetilde{\nabla}f_t(\mathbf{x}_t, \mathbf{y}_{t+1}, \mathbf{v}_{t+1}), \gamma\right)\right\|^2 \nonumber\\
    & + \frac{\gamma}{2\lambda}\left\| \widetilde{\nabla} f_t(\mathbf{x}_t, \mathbf{y}_{t+1}, \mathbf{v}_{t+1}) - \nabla f_t(\mathbf{x}_t, \mathbf{y}_t^*(\mathbf{x}_t)) \right\|^2, \label{A.10_1}
\end{align}
for $\forall \lambda>0$. We also have
\begin{align}
    &-\left\|\mathcal{G}_\mathcal{X}\left(\mathbf{x}_t, \widetilde{\nabla}f_t(\mathbf{x}_t, \mathbf{y}_{t+1}, \mathbf{v}_{t+1}), \gamma\right)\right\|^2 \nonumber\\
    \leq& \left\|\mathcal{G}_\mathcal{X}(\mathbf{x}_t, \nabla f_t(\mathbf{x}_t, \mathbf{y}_t^*(\mathbf{x}_t)), \gamma) - \mathcal{G}_\mathcal{X}\left(\mathbf{x}_t, \widetilde{\nabla}f_t(\mathbf{x}_t, \mathbf{y}_{t+1}, \mathbf{v}_{t+1}), \gamma\right)\right\|^2 \nonumber\\
    & - \frac{1}{2} \left\|\mathcal{G}_\mathcal{X}(\mathbf{x}_t, \nabla f_t(\mathbf{x}_t, \mathbf{y}_t^*(\mathbf{x}_t)), \gamma)\right\|^2, \label{A.10_2}
\end{align}
and substituting (\ref{A.10_2}) into (\ref{A.10_1}), it follows that
\begin{align*}
    &f_t(\mathbf{x}_{t+1}, \mathbf{y}_t^*(\mathbf{x}_{t+1})) - f_t(\mathbf{x}_t, \mathbf{y}_t^*(\mathbf{x}_t)) \\
    \leq& -\frac{\gamma}{2}\left( 1 - \frac{\lambda}{2} - \frac{\gamma L_F}{2} \right) \left\|\mathcal{G}_\mathcal{X}(\mathbf{x}_t, \nabla f_t(\mathbf{x}_t, \mathbf{y}_t^*(\mathbf{x}_t)), \gamma)\right\|^2 \\
    & + \gamma\left( 1 - \frac{\lambda}{2} - \frac{\gamma L_F}{2} \right)\left\|\mathcal{G}_\mathcal{X}(\mathbf{x}_t, \nabla f_t(\mathbf{x}_t, \mathbf{y}_t^*(\mathbf{x}_t)), \gamma) - \mathcal{G}_\mathcal{X}\left(\mathbf{x}_t, \widetilde{\nabla}f_t(\mathbf{x}_t, \mathbf{y}_{t+1}, \mathbf{v}_{t+1}), \gamma\right)\right\|^2 \\
    & + \frac{\gamma}{2\lambda}\left\| \widetilde{\nabla} f_t(\mathbf{x}_t, \mathbf{y}_{t+1}, \mathbf{v}_{t+1}) - \nabla f_t(\mathbf{x}_t, \mathbf{y}_t^*(\mathbf{x}_t)) \right\|^2 \\
    \leq& -\frac{\gamma}{2}\left( 1 - \frac{\lambda}{2} - \frac{\gamma L_F}{2} \right) \|\mathcal{G}_\mathcal{X}(\mathbf{x}_t, \nabla f_t(\mathbf{x}_t, \mathbf{y}_t^*(\mathbf{x}_t)), \gamma)\|^2 \\
    & + \gamma\left( 1 - \frac{\lambda}{2} + \frac{1}{2\lambda} - \frac{\gamma L_F}{2} \right) \left\| \widetilde{\nabla} f_t(\mathbf{x}_t, \mathbf{y}_{t+1}, \mathbf{v}_{t+1}) - \nabla f_t(\mathbf{x}_t, \mathbf{y}_t^*(\mathbf{x}_t)) \right\|^2.
\end{align*}
Let \( \theta = 1 - \frac{\lambda}{2} - \frac{\gamma L_F}{2} \) and summing from \( 1, \dots, K \), we obtain
\begin{align*}
    &\sum_{t=1}^{T}\left( f_t(\mathbf{x}_{t+1}, \mathbf{y}_t^*(\mathbf{x}_{t+1})) - f_t(\mathbf{x}_t, \mathbf{y}_t^*(\mathbf{x}_t)) \right) \\
    \leq& -\frac{\gamma\theta}{2} \sum_{t=1}^T \left\|\mathcal{G}_\mathcal{X}(\mathbf{x}_t, \nabla f_t(\mathbf{x}_t, \mathbf{y}_t^*(\mathbf{x}_t)), \gamma)\right\|^2 \\
    & + \gamma\left(\theta + \frac{1}{2\lambda}\right)\left\| \widetilde{\nabla} f_t(\mathbf{x}_t, \mathbf{y}_{t+1}, \mathbf{v}_{t+1}) - \nabla f_t(\mathbf{x}_t, \mathbf{y}_t^*(\mathbf{x}_t)) \right\|^2,
\end{align*}
and after rearranging the terms, we finally have
\begin{align*}
    &\sum_{t=1}^{T} \left\| \mathcal{G}_\mathcal{X}\left( \mathbf{x}_t, \nabla f_t(\mathbf{x}_t, \mathbf{y}_t^*(\mathbf{x}_t)), \gamma \right) \right\|^2 \\
    \leq& \frac{2}{\gamma\theta} \sum_{t=1}^{T} \left( f_t(\mathbf{x}_t, \mathbf{y}_t^*(\mathbf{x}_t)) - f_t(\mathbf{x}_{t+1}, \mathbf{y}_t^*(\mathbf{x}_{t+1})) \right) \\
    & + \left(2+\frac{1}{\lambda\theta}\right) \sum_{t=1}^{T} \left\| \nabla f_t(\mathbf{x}_t, \mathbf{y}_t^*(\mathbf{x}_t)) - \widetilde{\nabla} f_t(\mathbf{x}_t, \mathbf{y}_{t+1}, \mathbf{v}_{t+1}) \right\|^2
\end{align*}
to finish the proof.
\end{proof}

\subsection{Proof of Theorem~\ref{thm:SOGO_upper}}\label{sec:proof_SOGO}

\begin{theorem}[Restatement of Theorem~\ref{thm:SOGO_upper}]
Under Assumptions~\ref{asm1}-\ref{asm3}, let $\alpha \leq \frac{1}{L_{g,1}}$, $M = \left\lceil-\frac{\ln T}{\ln\rho}\right\rceil$ and $\gamma \leq \frac{1}{2L_F}$, Algorithm~\ref{alg:AOBO} can obtain
\begin{align*}
    \mathrm{Reg}(T) \leq \frac{16Q}{\gamma} + 6\kappa_F\delta^2T + 768L_{f,0}^2\kappa_g^2 + \frac{8}{\gamma}V_T = O\left(\delta^2T + V_T\right),
\end{align*}
with total number of inner iterations $\mathcal{I}_T$ satisfies
\begin{align*}
    \mathcal{I}_T \leq& \left(1 {+} \frac{2\ln\delta}{\ln(1-\alpha\mu_g))}\right)T {+} \frac{3\delta^2\kappa_g^2}{\ln(1{-}\alpha\mu_g)^{-1}}T {+} \frac{3\kappa_g^2L_{f,0}^2L_{g,1}^2(1{+}\kappa_g)^2\gamma^2}{\ln(1{-}\alpha\mu_g)^{-1}}T {+} \frac{3L_{g,1}^2}{\ln(1{-}\alpha\mu_g)^{-1}}H_{2,T} \\
    =& O\left(T\log\delta^{-1} + H_{2,T}\right).
\end{align*}
\end{theorem}
\begin{proof}
It begins with Lemma~\ref{lem:begin}, shown as:
\begin{align}
    &\sum_{t=1}^{T} \left\| \mathcal{G}_\mathcal{X}\left( \mathbf{x}_t, \nabla f_t(\mathbf{x}_t, \mathbf{y}_t^*(\mathbf{x}_t)), \gamma \right) \right\|^2 \nonumber\\
    \leq& \frac{2}{\gamma\theta} \sum_{t=1}^{T} \left( f_t(\mathbf{x}_t, \mathbf{y}_t^*(\mathbf{x}_t)) - f_t(\mathbf{x}_{t+1}, \mathbf{y}_t^*(\mathbf{x}_{t+1})) \right) \nonumber\\
    & + \left(2+\frac{1}{\lambda\theta}\right) \sum_{t=1}^{T} \left\| \nabla f_t(\mathbf{x}_t, \mathbf{y}_t^*(\mathbf{x}_t)) - \widetilde{\nabla} f_t(\mathbf{x}_t, \mathbf{y}_{t+1}, \mathbf{v}_{t+1}) \right\|^2, \label{A.11_1}
\end{align}
it then follows that
\begin{align}
    &\left\| \nabla f_t(\mathbf{x}_t, \mathbf{y}_t^*(\mathbf{x}_t)) - \widetilde{\nabla} f_t(\mathbf{x}_t, \mathbf{y}_{t+1}, \mathbf{v}_{t+1}) \right\|^2 \nonumber\\
    =& \left\| \nabla_\mathbf{x} f_t(\mathbf{x}_t, \mathbf{y}_t^*(\mathbf{x}_t)) - \nabla_\mathbf{x} f_t(\mathbf{x}_t, \mathbf{y}_{t+1}) + \nabla_{\mathbf{x}\mathbf{y}}^2 g_t(\mathbf{x}_t, \mathbf{y}_{t+1})\mathbf{v}_{t+1} - \nabla_{\mathbf{x}\mathbf{y}}^2 g_t(\mathbf{x}_t, \mathbf{y}_t^*(\mathbf{x}_t))\mathbf{v}_t^*(\mathbf{x}_t) \right\|^2 \nonumber\\
    \leq& 2\left\| \nabla_\mathbf{x} f_t(\mathbf{x}_t, \mathbf{y}_t^*(\mathbf{x}_t)) - \nabla_\mathbf{x} f_t(\mathbf{x}_t, \mathbf{y}_{t+1})\right\|^2 \nonumber\\
    & + 2\left\|\nabla_{\mathbf{x}\mathbf{y}}^2 g_t(\mathbf{x}_t, \mathbf{y}_{t+1})\mathbf{v}_{t+1} - \nabla_{\mathbf{x}\mathbf{y}}^2 g_t(\mathbf{x}_t, \mathbf{y}_t^*(\mathbf{x}_t))\mathbf{v}_t^*(\mathbf{x}_t) \right\|^2 \nonumber\\
    \overset{(i)}\leq& 2L_{f,1}^2\left\|\mathbf{y}_{t+1} - \mathbf{y}_t^*(\mathbf{x}_t)\right\|^2 + 4\left\|\nabla_{\mathbf{x}\mathbf{y}}^2 g_t(\mathbf{x}_t, \mathbf{y}_{t+1})\mathbf{v}_{t+1} - \nabla_{\mathbf{x}\mathbf{y}}^2 g_t(\mathbf{x}_t, \mathbf{y}_{t+1})\mathbf{v}_t^*(\mathbf{x}_t)\right\|^2 \nonumber\\
    & + 4\left\|\nabla_{\mathbf{x}\mathbf{y}}^2 g_t(\mathbf{x}_t, \mathbf{y}_{t+1})\mathbf{v}_t^*(\mathbf{x}_t) - \nabla_{\mathbf{x}\mathbf{y}}^2 g_t(\mathbf{x}_t, \mathbf{y}_t^*(\mathbf{x}_t))\mathbf{v}_t^*(\mathbf{x}_t)\right\|^2 \nonumber\\
    \overset{(ii)}\leq& \left(2L_{f,1}^2 + \frac{4L_{f,0}^2L_{g,2}^2}{\mu_g^2}\right)\left\|\mathbf{y}_{t+1} - \mathbf{y}_t^*(\mathbf{x}_t)\right\|^2 + 4L_{g,1}^2\left\|\mathbf{v}_{t+1} - \mathbf{v}_t^*(\mathbf{x}_t)\right\|^2 \label{A.11_2} \\
    \overset{(iii)}\leq& \left(2L_{f,1}^2 {+} \frac{4L_{f,0}^2L_{g,2}^2}{\mu_g^2} {+} \frac{16L_{f,1}^2L_{g,1}^2}{\mu_g^2} {+} \frac{16L_{f,0}^2L_{g,1}^2L_{g,2}^2}{\mu_g^4}\right)\left\|\mathbf{y}_{t+1} - \mathbf{y}_t^*(\mathbf{x}_t)\right\|^2 + \frac{128L_{f,0}^2L_{g,1}^2}{\mu_g^2}\rho^{M}, \label{A.11_3}
\end{align}
where $(i)$, $(ii)$ comes from the Lipschitz continuity of $\nabla f_t(\cdot)$ in Assumption~\ref{asm2} and Lemma~\ref{lem:v_cons}, $(iii)$ comes from Lemma~\ref{lem:v_bound_1} with $M_t \equiv M$. Due to (\ref{eq:sc_2}), let $\mathbf{x} = \mathbf{x}_t$, $\mathbf{y}_1 = \mathbf{y}_{t+1}$ and $\mathbf{y}_2 = \mathbf{y}_t^*(\mathbf{x}_t)$ we have
\begin{align*}
    \mu_g\left\|\mathbf{y}_{t+1} - \mathbf{y}_t^*(\mathbf{x}_t)\right\| \leq \left\|\nabla_\mathbf{y}g_t(\mathbf{x}_t, \mathbf{y}_{t+1})\right\| \leq \delta,
\end{align*}
which implies that
\begin{align*}
    &\left\| \nabla f_t(\mathbf{x}_t, \mathbf{y}_t^*(\mathbf{x}_t)) - \widetilde{\nabla} f_t(\mathbf{x}_t, \mathbf{y}_{t+1}, \mathbf{v}_{t+1}) \right\|^2 \\
    \leq& \left(\frac{2L_{f,1}^2}{\mu_g^2} + \frac{4L_{f,0}^2L_{g,2}^2}{\mu_g^4} + \frac{16L_{f,1}^2L_{g,1}^2}{\mu_g^4} + \frac{16L_{f,0}^2L_{g,1}^2L_{g,2}^2}{\mu_g^6}\right)\delta^2 + \frac{128L_{f,0}^2L_{g,1}^2}{\mu_g^2}\rho^M,
\end{align*}
and summing over $t\in[T]$ to get
\begin{align}
    \sum_{t=1}^T\left\| \nabla f_t(\mathbf{x}_t, \mathbf{y}_t^*(\mathbf{x}_t)) - \widetilde{\nabla} f_t(\mathbf{x}_t, \mathbf{y}_{t+1}, \mathbf{v}_{t+1}) \right\|^2 \leq \kappa_F\delta^2T + 128L_{f,0}^2\kappa_g^2\rho^MT, \label{A.11_4}
\end{align}
where
\begin{align}
    \kappa_F := \frac{2L_{f,1}^2}{\mu_g^2} + \frac{4L_{f,0}^2L_{g,2}^2}{\mu_g^4} + \frac{16L_{f,1}^2L_{g,1}^2}{\mu_g^4} + \frac{16L_{f,0}^2L_{g,1}^2L_{g,2}^2}{\mu_g^6}. \label{eq:kappa_F}
\end{align}
Finally, with Lemma~\ref{lem:2M+V_T} and substituting (\ref{A.11_4}) into (\ref{A.11_1}), we obtain
\begin{align*}
    \sum_{t=1}^{T} \left\| \mathcal{G}_\mathcal{X}\left( \mathbf{x}_t, \nabla f_t(\mathbf{x}_t, \mathbf{y}_t^*(\mathbf{x}_t)), \gamma \right) \right\|^2 
    \leq& \frac{4Q}{\gamma\theta} + \frac{2}{\gamma\theta}V_T + \left(2+\frac{1}{\lambda\theta}\right)\kappa_F\delta^2T \nonumber\\
    & + \left(2+\frac{1}{\lambda\theta}\right)\frac{128L_{f,0}^2L_{g,1}^2}{\mu_g^2}\rho^{M}T,
\end{align*}
let $M = \left\lceil-\frac{\ln T}{\ln\rho}\right\rceil$ and $\gamma \leq \frac{1}{2L_F}$, we have
\begin{align*}
    \mathrm{Reg}(T) \leq \frac{16Q}{\gamma} + 6\kappa_F\delta^2T + 768L_{f,0}^2\kappa_g^2 + \frac{8}{\gamma}V_T = O\left(\delta^2T + V_T\right).
\end{align*}
Next we prove the upper bound on the total number of inner iterations. Consider the sequence $\{\mathbf{y}_t^k\}_{k=1}^{K_t}$ generated in every $t$:
\begin{align*}
    \mathbf{y}_t^{k+1} \leftarrow \mathbf{y}_t^k - \alpha\nabla_\mathbf{y} g_t(\mathbf{x}_t, \mathbf{y}_t^k).
\end{align*}
Based on existing conclusions of gradient descent in strongly convex optimization problems, it implies that
\begin{align}
    \frac{\delta^2}{L_{g,1}^2} \overset{(i)}\leq \left\|\mathbf{y}_t^{K_t} - \mathbf{y}_t^*(\mathbf{x}_t)\right\|^2 \leq \rho^{K_t-1}\left\|\mathbf{y}_t^1 - \mathbf{y}_t^*(\mathbf{x}_t)\right\|^2, \label{A.11_5}
\end{align}
where $(i)$ comes from the fact that Algorithm~\ref{alg:AOBO} maintains iteration when $\left\|\nabla_\mathbf{y}g_t(\mathbf{x}_t, \mathbf{y}_t^k)\right\| > \delta$ and if let $\mathbf{x}=\mathbf{x}_t$, $\mathbf{y}_1 = \mathbf{y}_t^{K_t}$ and $\mathbf{y}_2 = \mathbf{y}_t^*(\mathbf{x}_t)$, in (\ref{eq:sc_1}) we have
\begin{align}
    L_{g,1}\left\|\mathbf{y}_t^{K_t} - \mathbf{y}_t^*(\mathbf{x}_t)\right\| \geq \left\|\nabla_\mathbf{y} g_t(\mathbf{x}_t, \mathbf{y}_t^{K_t})\right\| > \delta, \label{eq:sc_L}
\end{align}
also $\rho = 1 - \mu_g/L_{g,1}$ if $\alpha \leq 1/L_{g,1}$. Then rearranging the terms of (\ref{A.11_5}), we have
\begin{align*}
    \rho^{K_t - 1} \geq& \frac{\delta^2}{L_{g,1}^2\left\|\mathbf{y}_t - \mathbf{y}_t^*(\mathbf{x}_t)\right\|^2} \\
    (K_t-1)\ln\rho \geq& 2\ln\delta - \ln(L_{g,1}^2\left\|\mathbf{y}_t - \mathbf{y}_t^*(\mathbf{x}_t)\right\|^2) \\
    K_t \leq& 1 - \frac{\ln(L_{g,1}^2\left\|\mathbf{y}_t - \mathbf{y}_t^*(\mathbf{x}_t)\right\|^2)}{\ln\rho} + \frac{2\ln\delta}{\ln\rho} \\
    K_t \leq& 1 + \frac{L_{g,1}^2}{\ln\rho^{-1}}\left\|\mathbf{y}_t - \mathbf{y}_t^*(\mathbf{x}_t)\right\|^2 + \frac{2\ln\delta}{\ln\rho},
\end{align*}
thus
\begin{align}
    \sum_{t=1}^TK_t \leq \left(1 + \frac{2\ln\delta}{\ln\rho}\right)T + \frac{L_{g,1}^2}{\ln\rho^{-1}}\sum_{t=1}^T\left\|\mathbf{y}_t - \mathbf{y}_t^*(\mathbf{x}_t)\right\|^2. \label{A.11_6}
\end{align}
We also have
\begin{align}
    &\sum_{t=1}^T\left\|\mathbf{y}_t - \mathbf{y}_t^*(\mathbf{x}_t)\right\|^2 \nonumber\\
    \leq& 3\sum_{t=1}^T\left\|\mathbf{y}_t - \mathbf{y}_{t-1}^*(\mathbf{x}_{t-1})\right\|^2 {+} 3\sum_{t=1}^T\left\|\mathbf{y}_{t-1}^*(\mathbf{x}_{t-1}) - \mathbf{y}_{t-1}^*(\mathbf{x}_t)\right\|^2 {+} 3\sum_{t=1}^T\left\|\mathbf{y}_{t-1}^*(\mathbf{x}_t) - \mathbf{y}_t^*(\mathbf{x}_t)\right\|^2 \nonumber\\
    \leq& \frac{3\delta^2}{\mu_g^2}T + 3\gamma^2\kappa_g^2\sum_{t=1}^T\left\|\mathcal{G}_\mathcal{X}(\mathbf{x}_{t-1}, \widetilde{\nabla}f_{t-1}(\mathbf{x}_{t-1}, \mathbf{y}_t, \mathbf{v}_t), \gamma)\right\|^2 + 3H_{2,T} \nonumber\\
    \overset{(i)}\leq& \frac{3\delta^2}{\mu_g^2}T + 3\kappa_g^2L_{f,0}^2(1+\kappa_g)^2\gamma^2T + 3H_{2,T}, \label{A.11_7}
\end{align}
where $(i)$ comes from the fact that
\begin{align*}
    \left\|\mathcal{G}_\mathcal{X} (\mathbf{x}, \widetilde{\nabla}f_t(\mathbf{x}, \mathbf{y}, \mathbf{v}), \gamma)\right\| \leq \left\|\widetilde{\nabla}f_t(\mathbf{x}, \mathbf{y}, \mathbf{v})\right\| = \left\|\nabla_\mathbf{x}f_t(\mathbf{x}, \mathbf{y}) {-} \nabla_{\mathbf{x}\mathbf{y}}^2g_t(\mathbf{x}, \mathbf{y})\mathbf{v}\right\| \leq L_{f,0} {+} L_{g,1}\frac{L_{f,0}}{\mu_g}.
\end{align*}
Finally, substituting (\ref{A.11_7}) into (\ref{A.11_6}), it holds that
\begin{align*}
    \mathcal{I}_T := \sum_{t=1}^TK_t \leq& \left(1 {+} \frac{2\ln\delta}{\ln\rho}\right)T + \frac{3\delta^2\kappa_g^2}{\ln\rho^{-1}}T + \frac{3\kappa_g^2L_{f,0}^2L_{g,1}^2(1{+}\kappa_g)^2\gamma^2}{\ln\rho^{-1}}T + \frac{3L_{g,1}^2}{\ln\rho^{-1}}H_{2,T} \\
    =& O\left(T\log\delta^{-1} + H_{2,T}\right).
\end{align*}
Thus we finish the proof.
\end{proof}

\subsection{Proof of Theorem~\ref{thm:SOGO_lower}}\label{sec:proof_SOGO_lower}

In Section~\ref{sec:3.1}, we assert that AOBO reaches an regret upper bound of $O(1+V_T)$, which matches the lower bound of the proof in \citet{guan2024on} on single-layer online nonconvex optimization and thus Theorem~\ref{thm:SOGO_lower} naturally holds. Therefore, we believe that AOBO is optimal on OBO. Although the scenarios are different, $V_T$ is defined to characterize the change of the objective function value over time, and OBO is a more difficult scenario than single-layer optimization, so it should have at least a larger lower bound. For the sake of rigor and completeness, we still provide the proof of Theorem~\ref{thm:SOGO_lower}. Remind of the online hypergradient-based algorithm class defined in Definition~\ref{dfn:alg}, which is satisfied by AOBO, OBBO \cite{bohne2024online}, SOBOW \cite{lin2023non}, FSOBO and SOGD \cite{nazari2025stochastic} without noise.

\begin{definition}[Restatement of Definition~\ref{dfn:alg}]
Suppose there are totally $T$ time steps, the iterates $\{(\mathbf{x}_t, \mathbf{y}_t)\}_{t=1}^T$ are generated according to $(\mathbf{x}_t, \mathbf{y}_t)\in \mathcal{H}_\mathbf{x}^t, \mathcal{H}_\mathbf{y}^t$ with $\mathcal{H}_\mathbf{x}^1 = \mathcal{H}_\mathbf{y}^1 = \{0\}$. 

For all $t$, the linear subspace $\mathcal{H}_\mathbf{y}^t$ is allowed to expand $K$ times with $\mathcal{H}_\mathbf{y}^{t,1} = \mathcal{H}_\mathbf{y}^t$ for $k=1,\dots,K$, shown as:
\begin{align*}
    \mathcal{H}_\mathbf{y}^{t,k+1} \leftarrow \mathrm{Span}\big\{&\mathbf{y}_i, \nabla_\mathbf{y} g_i(\widetilde{\mathbf{x}}_i, \widetilde{\mathbf{y}}_i)\big\},
\end{align*}
where $\widetilde{\mathbf{x}}_i \in \mathcal{H}_\mathbf{x}^i$, $\mathbf{y}_i, \widetilde{\mathbf{y}}_i\in\mathcal{H}_\mathbf{y}^{t,j}$ and $i\in[1,t]$, $j\in[1,k]$.

Then with $\mathcal{H}_\mathbf{y}^{t+1} = \mathcal{H}_\mathbf{y}^{t,K}$, define $S > 0$, $\mathcal{H}_\mathbf{x}^t$ expand as follows:
\begin{align}
    &\mathcal{H}_\mathbf{x}^{t+1} \leftarrow \mathrm{Span}\Big\{ \mathbf{x}_i, \nabla_\mathbf{x} f_i(\widetilde{\mathbf{x}}_i, \widetilde{\mathbf{y}}_j), \nabla_{\mathbf{x}\mathbf{y}}^2g_i(\overline{\mathbf{x}}_i, \overline{\mathbf{y}}_j)\prod_{j=1}^s\left( \mathbf{I}_{d_2} - \alpha \nabla_{\mathbf{y}\mathbf{y}}^2 g_i(\mathbf{x}_i^s, \mathbf{y}_j^s)\right) \nabla_\mathbf{y} f_i(\widehat{\mathbf{x}}_i, \widehat{\mathbf{y}}_i)\Big\},\!\! \label{dfn:span}
\end{align}
where $\mathbf{x}_i, \widetilde{\mathbf{x}}_i, \overline{\mathbf{x}}_i, \mathbf{x}_i^s, \widehat{\mathbf{x}}_i \in\mathcal{H}_\mathbf{x}^i$, $\widetilde{\mathbf{y}}_j, \overline{\mathbf{y}}_j, \mathbf{y}_j^s, \widehat{\mathbf{y}}_j \in\mathcal{H}_\mathbf{y}^j$ with any $\alpha\in \mathbb{R}$ for $i\in[1,t]$, $j\in[1,t+1]$ and $s\in[1,S]$.
\end{definition}

Our proof follows the same idea as the proof of Theorem 1 in \citet{guan2024on}. Given a total budget of $V_T$ variation, the bilevel objective function $f_t(\mathbf{x}, \mathbf{y}_t^*(\mathbf{x}))$ we constructed has $\Omega(1 + V_T)$ times of change. Thus, we divide the total time steps into $\Omega(1 + V_T)$ blocks, functions within each block remains the same while changing across a new block. Now, we construct a series of functions whose gradients are orthogonal to each other and assign them to these blocks, the algorithm will be forced to restart the learning process when entering a new block and hence the accumulated local regret is $\Omega(1 + V_T)$.

\begin{theorem}[Restatement of Theorem~\ref{thm:SOGO_lower}]
Consider any online hypergradient-based algorithm $\mathcal{A}$ that satisfies Definition~\ref{dfn:alg} with any $K, S>0$, under Assumptions~\ref{asm1}-\ref{asm3} with $d_1=d_2 \geq \Omega(1+V_T)$, there exist $\{f_t, g_t\}_{t=1}^T$ that satisfy $V_T = o(T)$ for which $\mathrm{Reg}(T) \geq \Omega (1+V_T)$.
\end{theorem}
\begin{proof}
Assume $\mathbf{x}$ and $\mathbf{y}$ have the same dimension $d$. For any vector $\mathbf{a}$, We use $[\mathbf{a}]_i$ to represent the $i$-th coordinate of $\mathbf{a}$ and $[\mathbf{a}]_{i:j}$ represent the coordinates from the $i$-th to the $j$-th coordinate of $\mathbf{a}$.

\textbf{Step 1: Constructing the functions $f_t$, $g_t$ that satisfying Assumptions~\ref{asm1}-\ref{asm3}.}

We fixedly divide the function group $\{f_t\}_{t=1}^T$ into $N$ blocks, and the size of each block is $B = \left\lceil T/N \right\rceil$. For $k\in[N]$, we have $f_{(k-1)B+1}(\mathbf{x}, \mathbf{y}) = f_{(k-1)B+2}(\mathbf{x}, \mathbf{y}) = \dots = f_{kB}(\mathbf{x}, \mathbf{y}) = \widetilde{f}_k(\mathbf{x}, \mathbf{y})$, which is defined as
\begin{align*}
    \widetilde{f}_k(\mathbf{x}, \mathbf{y}) = c\sigma([\mathbf{y}]_k) + c\sigma([\mathbf{x}]_k), \quad \text{and} \quad g_t(\mathbf{x}, \mathbf{y}) \equiv g(\mathbf{x}, \mathbf{y}) = \frac{\mu}{2}\left\|\mathbf{y}-\mathbf{x}\right\|^2,
\end{align*}
where $\sigma(\cdot)$ denotes the sigmoid function. 
For any $\mathbf{x}, \mathbf{y}\in\mathbb{R}^{d}$ and $\mathbf{x}', \mathbf{y}'\in\mathbb{R}^{d}$, it holds that
\begin{align*}
    \left\|\nabla f_t(\mathbf{x}, \mathbf{y})\right\| \leq \frac{\sqrt{2}c}{4}, \quad \left\|\nabla f_t(\mathbf{x}, \mathbf{y}) - \nabla f_t(\mathbf{x}', \mathbf{y}')\right\| \leq \frac{\sqrt{3}c}{18}\left\|(\mathbf{x}, \mathbf{y}) - (\mathbf{x}', \mathbf{y}')\right\|,
\end{align*}
thus we set $c = \min\{\frac{Q}{2}, 2\sqrt{2}L_{f,0}, 6\sqrt{3}L_{f,1}\}$ to let $f_t$ satisfies Assumptions~\ref{asm2} and~\ref{asm3} for all $t\in[T]$. 

For any $\mathbf{x}, \mathbf{y}\in\mathbb{R}^d$, $\nabla_\mathbf{y}g_t(\mathbf{x}, \mathbf{y}) = \mu(\mathbf{y} - \mathbf{x})$, it follows that $\nabla_{\mathbf{y}\mathbf{y}}^2 g_t(\mathbf{x}, \mathbf{y}) = \mu\mathbf{I}_d$ and $\nabla_{\mathbf{x}\mathbf{y}}^2g_t(\mathbf{x}, \mathbf{y}) = -\mu\mathbf{I}_d$.
$g_t(\cdot)$ is $\mu_g$-strongly convex if $\mu = \mu_g$, and thus $\left\|\nabla_{\mathbf{x}\mathbf{y}}^2g_t(\mathbf{x}, \mathbf{y})\right\| = \mu_g \leq L_{g,1}$.
$\nabla_{\mathbf{y}\mathbf{y}}^2g_t(\cdot)$ and $\nabla_{\mathbf{x}\mathbf{y}}^2g_t(\mathbf{x}, \mathbf{y})$ satisfy $L_{g,2}$-Lipschitz continuity because they are all constant matrices.

Consider the exact function-value variation $V_T^F$ we constructed, it follows that
\begin{align*}
    V_T^F = \sum_{t=2}^T\sup_{\mathbf{x}}\left|f_t(\mathbf{x}, \mathbf{y}_t^*(\mathbf{x}) - f_{t-1}(\mathbf{x}, \mathbf{y}_{t-1}^*(\mathbf{x}))\right| =& \sum_{k=2}^N\sup_\mathbf{x}\left|2c\sigma([\mathbf{x}]_k) - 2c\sigma([\mathbf{x}]_{k-1})\right| \\
    \leq& Q(N-1) \leq V_T,
\end{align*}
where the last inequality holds if set $N = \left\lceil\frac{1 + V_T}{Q}\right\rceil$ and thus $B = \left\lfloor\frac{QT}{1 + V_T}\right\rfloor$.

\textbf{Step 2: Characterizing the lower bound of bilevel local regret.}

For the first block $k=1$, with
\begin{align*}
    \nabla_\mathbf{x}\widetilde{f}_1(\mathbf{x}, \mathbf{y}) = c\sigma'([\mathbf{x}]_1), \quad \nabla_\mathbf{y}\widetilde{f}_1(\mathbf{x}, \mathbf{y}) = c\sigma([\mathbf{y}]_1)
\end{align*}
and $\nabla_{\mathbf{y}\mathbf{y}}^2g_t(\mathbf{x}, \mathbf{y}) = -\nabla_{\mathbf{x}\mathbf{y}}^2g_t(\mathbf{x}, \mathbf{y}) = \mu_g\mathbf{I}_d$, according to the algorithm's subspace update rule (\ref{dfn:span}), we have
\begin{align*}
    \mathcal{H}_\mathbf{x}^{2} = \dots = \mathcal{H}_{\mathbf{x}}^{B+1} = \mathrm{Span}\{e_1\},
\end{align*}
which means that in the first block, only the first coordinate is activated, while the rest remain at $0$, and this situation continues until it moves to the next block.
After repeating the same steps, it implies that
\begin{align}
    \mathcal{H}_\mathbf{x}^{(k-1)B+2} = \dots = \mathcal{H}_\mathbf{x}^{kB+1} = \mathrm{Span}\{e_1, \dots, e_k\}. \label{eq:lower_bound_1}
\end{align}
Now note that $\mathbf{y}_t^*(\mathbf{x}) = \mathbf{x}$ for all $t\in[T]$, we have
\begin{align*}
    f_t(\mathbf{x}, \mathbf{y}_t^*(\mathbf{x})) = 2c\sigma([\mathbf{x}]_k), \quad \nabla f_t(\mathbf{x}, \mathbf{y}_t^*(\mathbf{x})) = 2c\sigma'([\mathbf{x}]_k)e_k.
\end{align*}
When in a new block, $t = (k-1)B + 1$, due to the equation (\ref{eq:lower_bound_1}), it follows that
\begin{align*}
    \mathcal{H}_\mathbf{x}^{(k-1)B+1} = \mathrm{Span}\{e_1, \dots, e_{k-1}\},
\end{align*}
thus 
\begin{align*}
    \left\|\nabla f_t(\mathbf{x}_t, \mathbf{y}_t^*(\mathbf{x}_t))\right\| = \left\|2c\sigma'([\mathbf{x}_{(k-1)B+1}]_{k})e_k\right\| = 2c\sigma'(0) = \frac{c}{2}.
\end{align*}
Finally, we finish the proof by
\begin{align*}
    \sum_{t=1}^T\left\|\nabla f_t(\mathbf{x}_t, \mathbf{y}_t^*(\mathbf{x}_t))\right\|^2 \geq& \sum_{k=1}^N\left\|2c\sigma'([\mathbf{x}_{(k-1)B+1}]_{k})e_k\right\|^2 \\
    =& 4Nc^2\sigma'(0)^2 
    = \frac{c^2}{16}\left\lceil\frac{1 + V_T}{Q}\right\rceil = \Omega(1 + V_T).
\end{align*}
\end{proof}

\subsection{Proof of Theorem~\ref{thm:CTHO_upper}}\label{sec:proof_CTHO}


\begin{theorem}[Restatement of Theorem~\ref{thm:CTHO_upper}]
Under Assumptions~\ref{asm1}-\ref{asm3}, let $\beta \leq \frac{1}{L_{g,1}}$, $\alpha \leq \frac{2\kappa_F\mu_g^2}{(\mu_g + L_{g,1})(\kappa_F\mu_g^2 + C_\beta)}$, $\gamma < \min\left\{\frac{1}{2L_F}, \sqrt{\frac{(1-\rho)C_1}{12C_\mathbf{x}}}\right\}$, $N = 1$ and $M_t \equiv M = 1$, Algorithm~\ref{alg:FSOBO} can guarantee
\begin{align*}
    \mathrm{Reg}(T) 
    \leq& \frac{16Q}{\gamma C_2} {+} \frac{8}{\gamma C_2}V_T {+} \frac{6\Delta_1'}{C_1C_2} {+} \frac{6C_\alpha}{(1{-}\rho)C_1C_2}H_{2,T} {+} \frac{6C_P}{(1{-}\rho)C_1C_2}E_{\mathbf{y},T}^f {+} \frac{6L_{f,0}^2C_P}{\mu_g^2(1{-}\rho)C_1C_2}E_{\mathbf{y}\mathbf{y},T}^g \\
    =& O(1 + V_T + H_{2,T} + E_{2,T}),
\end{align*}
where $\Delta_1'$, $C_1$, $C_2$, $C_\alpha$, $C_\beta$, $C_\mathbf{x}$, $C_P$ are some constants, $E_{2,T} = E_{\mathbf{y}\mathbf{y},T}^g + E_{\mathbf{y},T}^f$ and
\begin{align*}
    &E_{\mathbf{y}\mathbf{y},T}^g :=  \sum_{t=2}^T \sup_{\mathbf{x},\mathbf{y}}\left\|\nabla_{\mathbf{y}\mathbf{y}}^2 g_t(\mathbf{x}, \mathbf{y}) - \nabla_{\mathbf{y}\mathbf{y}}^2 g_{t-1}(\mathbf{x}, \mathbf{y})\right\|^2, \\
    &E_{\mathbf{y},T}^f := \sum_{t=2}^T \sup_{\mathbf{x},\mathbf{y}}\left\|\nabla_\mathbf{y} f_t(\mathbf{x},\mathbf{y}) - \nabla_\mathbf{y} f_{t-1}(\mathbf{x},\mathbf{y})\right\|^2.
\end{align*}
\end{theorem}
\begin{proof}
It begins with (\ref{A.11_2}),
\begin{align}
    &\left\|\nabla f_t(\mathbf{x}_t, \mathbf{y}_t^*(\mathbf{x}_t)) - \widetilde{\nabla} f_t(\mathbf{x}_t, \mathbf{y}_{t+1}, \mathbf{v}_{t+1})\right\| \nonumber\\
    \leq& \left(2L_{f,1}^2 + \frac{4L_{f,0}^2L_{g,2}^2}{\mu_g^2}\right)\left\|\mathbf{y}_{t+1} - \mathbf{y}_t^*(\mathbf{x}_t)\right\|^2 + 4L_{g,1}^2\left\|\mathbf{v}_{t+1} - \mathbf{v}_t^*(\mathbf{x}_t, \mathbf{y}_t^*(\mathbf{x}_t))\right\|^2 \nonumber\\
    \overset{(i)}\leq& \left(2L_{f,1}^2 + \frac{4L_{f,0}^2L_{g,2}^2}{\mu_g^2} + \frac{16L_{f,1}^2L_{g,1}^2}{\mu_g^2} + \frac{16L_{f,0}^2L_{g,1}^2L_{g,2}^2}{\mu_g^4}\right)\left\|\mathbf{y}_{t+1} - \mathbf{y}_t^*(\mathbf{x}_t)\right\|^2 \nonumber\\
    & + 8L_{g,1}^2\left\|\mathbf{v}_{t+1} - \mathbf{v}_t^*(\mathbf{x}_t, \mathbf{y}_{t+1})\right\|^2 \nonumber\\
    =:& \kappa_F\mu_g^2\left\|\mathbf{y}_{t+1} - \mathbf{y}_t^*(\mathbf{x}_t)\right\|^2 + 8L_{g,1}^2\left\|\mathbf{v}_{t+1} - \mathbf{w}_t^*\right\|^2 =: \Delta_t, \label{A.14_1}
\end{align}
where $(i)$ comes from the same fact in (\ref{A.6_9}), $\kappa_F$ reminds the same as (\ref{eq:kappa_F}) and $\mathbf{w}_t^* = \mathbf{v}_t^*(\mathbf{x}_t, \mathbf{y}_{t+1})$ just for simplicity. Then substituting Lemma~\ref{lem:y_bound_2} and~\ref{lem:v_bound_2} into (\ref{A.14_1}), with $\rho_\mathbf{y} := 1 - \frac{\alpha\mu_gL_{g,1}}{\mu_g + L_{g,1}}$ and $\rho_\mathbf{v} := 1-\frac{\beta\mu_g}{2}$, we have
\begin{align*}
    \Delta_t \leq& \kappa_F\mu_g^2\rho_\mathbf{y}\left\|\mathbf{y}_t - \mathbf{y}_{t-1}^*(\mathbf{x}_{t-1})\right\|^2 + 8L_{g,1}^2\rho_\mathbf{v}\left\|\mathbf{v}_t - \mathbf{w}_{t-1}^*\right\|^2 \\
    & - \kappa_F\mu_g^2\left(\frac{2\alpha}{\mu_g + L_{g,1}} - \alpha^2\right)\left\|\nabla_\mathbf{y} g_t(\mathbf{x}_t, \mathbf{y}_t)\right\|^2 \\
    &  + 2\kappa_g^2\kappa_F\mu_g^2\left(1+\frac{\mu_g + L_{g,1}}{\alpha\mu_gL_{g,1}}\right)\left\|\mathbf{x}_t - \mathbf{x}_{t-1}\right\|^2 \\
    & + 2\kappa_F\mu_g^2\left(1+\frac{\mu_g + L_{g,1}}{\alpha\mu_gL_{g,1}}\right)\left\|\mathbf{y}_{t-1}^*(\mathbf{x}_{t-1}) - \mathbf{y}_t^*(\mathbf{x}_{t-1})\right\|^2 \\
    & + 32\kappa_g^2\left(1+\frac{2}{\beta\mu_g}\right)\left\|\nabla_\mathbf{y} f_{t-1}(\mathbf{x}_t, \mathbf{y}_{t+1}) - \nabla_\mathbf{y} f_t(\mathbf{x}_t, \mathbf{y}_{t+1})\right\|^2 \\
    & + 32\kappa_g^2\left(1+\frac{2}{\beta\mu_g}\right)\frac{L_{f,0}^2}{\mu_g^2}\left\|\nabla_{\mathbf{y}\mathbf{y}}^2 g_{t-1}(\mathbf{x}_t, \mathbf{y}_{t+1}) - \nabla_{\mathbf{y}\mathbf{y}}^2 g_t(\mathbf{x}_t, \mathbf{y}_{t+1})\right\|^2 \\
    & + 8L_{g,1}^2\left(\frac{4L_{f,1}^2}{\mu_g^2} + \frac{4L_{f,0}^2L_{g,2}^2}{\mu_g^4}\right)\left(1+\frac{2}{\beta\mu_g}\right)\left(\|\mathbf{x}_{t-1} - \mathbf{x}_t\|^2 + \|\mathbf{y}_t - \mathbf{y}_{t+1}\|^2\right), 
\end{align*}
for simplicity, let
\begin{align*}
    &C_\beta := 8L_{g,1}^2\left(\frac{4L_{f,1}^2}{\mu_g^2} + \frac{4L_{f,0}^2L_{g,2}^2}{\mu_g^4}\right)\left(1+\frac{2}{\beta\mu_g}\right), \quad C_\mathbf{x} := 2\kappa_g^2\kappa_F\mu_g^2\left(1+\frac{\mu_g + L_{g,1}}{\alpha\mu_gL_{g,1}}\right) + C_\beta, \\
    &C_\alpha := 2\kappa_F\mu_g^2\left(1+\frac{\mu_g + L_{g,1}}{\alpha\mu_gL_{g,1}}\right), \quad C_P := 32\kappa_g^2\left(1+\frac{2}{\beta\mu_g}\right),
\end{align*}
then it was shown as:
\begin{align}
    \Delta_t \leq& \kappa_F\mu_g^2\rho_\mathbf{y}\left\|\mathbf{y}_t {-} \mathbf{y}_{t-1}^*(\mathbf{x}_{t-1})\right\|^2 + 8L_{g,1}^2\rho_\mathbf{v}\left\|\mathbf{v}_t {-} \mathbf{w}_{t-1}^*\right\|^2 \nonumber\\
    & - \left(\frac{2\kappa_F\mu_g^2\alpha}{\mu_g + L_{g,1}} - \left(\kappa_F\mu_g^2{+}C_\beta\right)\alpha^2\right)\left\|\nabla_\mathbf{y} g_t(\mathbf{x}_t, \mathbf{y}_t)\right\|^2 \nonumber\\
    & + C_\alpha\left\|\mathbf{y}_{t-1}^*(\mathbf{x}_{t-1}) - \mathbf{y}_t^*(\mathbf{x}_{t-1})\right\|^2 + C_\mathbf{x}\left\|\mathbf{x}_t - \mathbf{x}_{t-1}\right\|^2 \nonumber\\
    & + C_P\left\|\nabla_\mathbf{y} f_{t-1}(\mathbf{x}_t, \mathbf{y}_{t+1}) - \nabla_\mathbf{y} f_t(\mathbf{x}_t, \mathbf{y}_{t+1})\right\|^2 \nonumber\\
    & + \frac{L_{f,0}^2C_P}{\mu_g^2}\left\|\nabla_{\mathbf{y}\mathbf{y}}^2 g_{t-1}(\mathbf{x}_t, \mathbf{y}_{t+1}) - \nabla_{\mathbf{y}\mathbf{y}}^2 g_t(\mathbf{x}_t, \mathbf{y}_{t+1})\right\|^2. \label{A.14_2}
\end{align}
Let $\alpha \leq \frac{2\kappa_F\mu_g^2}{(\mu_g + L_{g,1})(\kappa_F\mu_g^2 + C_\beta)}$, then we obtain
\begin{align}
    \Delta_t \leq& \kappa_F\mu_g^2\rho_{\mathbf{y}}\left\|\mathbf{y}_t - \mathbf{y}_{t-1}^*(\mathbf{x}_{t-1})\right\|^2 + 8L_{g,1}^2\rho_{\mathbf{v}}\left\|\mathbf{v}_t - \mathbf{w}_{t-1}^*\right\|^2 + C_\mathbf{x}\left\|\mathbf{x}_t - \mathbf{x}_{t-1}\right\|^2 \nonumber\\
    & + C_\alpha\left\|\mathbf{y}_{t-1}^*(\mathbf{x}_{t-1}) {-} \mathbf{y}_t^*(\mathbf{x}_{t-1})\right\|^2 + C_P\left\|\nabla_\mathbf{y} f_{t-1}(\mathbf{x}_t, \mathbf{y}_{t+1}) - \nabla_\mathbf{y} f_t(\mathbf{x}_t, \mathbf{y}_{t+1})\right\|^2 \nonumber\\
    & + \frac{L_{f,0}^2C_P}{\mu_g^2}\left\|\nabla_{\mathbf{y}\mathbf{y}}^2 g_{t-1}(\mathbf{x}_t, \mathbf{y}_{t+1}) - \nabla_{\mathbf{y}\mathbf{y}}^2 g_t(\mathbf{x}_t, \mathbf{y}_{t+1})\right\|^2 \nonumber\\
    \leq& \rho\Delta_{t-1} + C_\mathbf{x}\left\|\mathbf{x}_t - \mathbf{x}_{t-1}\right\|^2 + C_\alpha\sup_\mathbf{x}\left\|\mathbf{y}_{t-1}^*(\mathbf{x}) - \mathbf{y}_t^*(\mathbf{x})\right\|^2 \nonumber\\
    & + C_P\sup_\mathbf{z}\left\|\nabla_\mathbf{y} f_{t-1}(\mathbf{z}) - \nabla_\mathbf{y} f_t(\mathbf{z})\right\|^2 + \frac{L_{f,0}^2C_P}{\mu_g^2}\sup_\mathbf{z}\left\|\nabla_{\mathbf{y}\mathbf{y}}^2 g_{t-1}(\mathbf{z}) - \nabla_{\mathbf{y}\mathbf{y}}^2 g_t(\mathbf{z})\right\|^2, \nonumber
\end{align}
where in the last inequality we set $\rho := \max\{\rho_\mathbf{y}, \rho_\mathbf{v}\}$. Telescoping the above inequality $t-1$ times, we obtain
\begin{align}
    \Delta_t \leq& \rho^{t-1}\Delta_1 + C_\mathbf{x}\sum_{j=0}^{t-2}\rho^j\left\|\mathbf{x}_{t-j} - \mathbf{x}_{t-1-j}\right\|^2 + C_\alpha\sum_{j=0}^{t-2}\rho^j\sup_\mathbf{x}\left\|\mathbf{y}_{t-1-j}^*(\mathbf{x}) - \mathbf{y}_{t-j}^*(\mathbf{x})\right\|^2 \nonumber\\
    & + C_P\sum_{j=0}^{t-2}\rho^j\sup_\mathbf{z}\left\|\nabla_\mathbf{y} f_{t-1-j}(\mathbf{z}) - \nabla_\mathbf{y} f_{t-j}(\mathbf{z})\right\|^2 \nonumber\\
    & + \frac{L_{f,0}^2C_P}{\mu_g^2}\sum_{j=0}^{t-2}\rho^j\sup_\mathbf{z}\left\|\nabla_{\mathbf{y}\mathbf{y}}^2 g_{t-1-j}(\mathbf{z}) - \nabla_{\mathbf{y}\mathbf{y}}^2 g_{t-j}(\mathbf{z})\right\|^2, \label{A.14_3}
\end{align}
and substituting (\ref{A.14_3}) into (\ref{A.14_1}) and summing over $t = 2,3,\dots,T$, we have
\begin{align*}
    &\sum_{t=2}^{T} \left\| \nabla f_t(\mathbf{x}_t, \mathbf{y}_t^*(\mathbf{x}_t)) - \widetilde{\nabla} f_t(\mathbf{x}_t, \mathbf{y}_{t+1}, \mathbf{v}_{t+1}) \right\|^2 \\
    \leq& \sum_{t=2}^{T}\rho^{t-1}\Delta_1 + C_\mathbf{x}\sum_{t=2}^{T}\sum_{j=0}^{t-2}\rho^j\left\|\mathbf{x}_{t-j} - \mathbf{x}_{t-1-j}\right\|^2 + C_\alpha\sum_{t=2}^{T}\sum_{j=0}^{t-2}\rho^j\sup_\mathbf{x}\left\|\mathbf{y}_{t-1-j}^*(\mathbf{x}) - \mathbf{y}_{t-j}^*(\mathbf{x})\right\|^2 \\
    & + C_P\sum_{t=2}^{T}\sum_{j=0}^{t-2}\rho^j\sup_\mathbf{z}\left\|\nabla_\mathbf{y} f_{t-1-j}(\mathbf{z}) - \nabla_\mathbf{y} f_{t-j}(\mathbf{z})\right\|^2 \\
    & + \frac{L_{f,0}^2C_P}{\mu_g^2}\sum_{t=2}^{T}\sum_{j=0}^{t-2}\rho^j\sup_\mathbf{z}\left\|\nabla_{\mathbf{y}\mathbf{y}}^2 g_{t-1-j}(\mathbf{z}) - \nabla_{\mathbf{y}\mathbf{y}}^2 g_{t-j}(\mathbf{z})\right\|^2 \\
    \leq& \frac{\Delta_1}{1-\rho} + \frac{C_\mathbf{x}}{1-\rho}\sum_{t=2}^{T}\left\|\mathbf{x}_t - \mathbf{x}_{t-1}\right\|^2 + \frac{C_\alpha}{1-\rho}H_{2,T} + \frac{C_P}{1-\rho}E_{\mathbf{y},T}^f + \frac{L_{f,0}^2C_P}{\mu_g^2(1-\rho)}E_{\mathbf{y}\mathbf{y},T}^g, 
\end{align*}
where the last inequality comes from the fact that $\sum_{j=0}^{t-2}\rho^j < \sum_{j=0}^\infty\rho^j = \frac{1}{1-\rho}$ for $\forall\rho \in(0,1)$. Then with (\ref{eq:x_update_2}), 
\begin{align*}
    &\sum_{t=1}^{T} \left\| \nabla f_t(\mathbf{x}_t, \mathbf{y}_t^*(\mathbf{x}_t)) - \widetilde{\nabla} f_t(\mathbf{x}_t, \mathbf{y}_{t+1}, \mathbf{v}_{t+1}) \right\|^2 \\
    \leq& \frac{\Delta_1}{1-\rho} + \frac{\gamma^2C_\mathbf{x}}{1-\rho}\sum_{t=2}^{T}\left\|\mathcal{G}_\mathcal{X} (\mathbf{x}_t, \widetilde{\nabla}f_t(\mathbf{x}_t, \mathbf{y}_{t+1}, \mathbf{v}_{t+1}))\right\|^2 + \frac{C_\alpha}{1-\rho}H_{2,T} + \frac{C_P}{1-\rho}E_{\mathbf{y},T}^f \\
    & + \frac{L_{f,0}^2C_P}{\mu_g^2(1-\rho)}E_{\mathbf{y}\mathbf{y},T}^g + \left(1 - \frac{2\gamma^2C_\mathbf{x}}{1-\rho}\right)\sum_{t=2}^{T} \left\| \nabla f_t(\mathbf{x}_t, \mathbf{y}_t^*(\mathbf{x}_t)) - \widetilde{\nabla} f_t(\mathbf{x}_t, \mathbf{y}_{t+1}, \mathbf{v}_{t+1}) \right\|^2 \\
    & + \left\|\nabla f_1(\mathbf{x}_1, \mathbf{y}_1^*(\mathbf{x}_1)) -  \widetilde{\nabla}f_1(\mathbf{x}_1, \mathbf{y}_2, \mathbf{v}_2)\right\|^2 \\
    \leq& \Delta_1' + \frac{2\gamma^2C_\mathbf{x}}{1{-}\rho}\sum_{t=1}^{T}\left\|\mathcal{G}_\mathcal{X} (\mathbf{x}_t, \nabla f_t(\mathbf{x}_t, \mathbf{y}_t^*(\mathbf{x}_t)))\right\|^2 + \frac{C_\alpha}{1{-}\rho}H_{2,T} + \frac{C_P}{1{-}\rho}E_{\mathbf{y},T}^f + \frac{L_{f,0}^2C_P}{\mu_g^2(1{-}\rho)}E_{\mathbf{y}\mathbf{y},T}^g,
\end{align*}
let $\gamma \leq \sqrt{\frac{1-\rho}{2C_\mathbf{x}}}$, the constant terms are integrated into 
\begin{align*}
    \Delta_1' := \frac{\Delta_1}{1-\rho} + \left\|\nabla f_1(\mathbf{x}_1, \mathbf{y}_1^*(\mathbf{x}_1)) -  \widetilde{\nabla}f_1(\mathbf{x}_1, \mathbf{y}_2, \mathbf{v}_2)\right\|^2
\end{align*}
and $C_1 := 1 - \frac{2\gamma^2C_\mathbf{x}}{1-\rho}$, it implies that
\begin{align}
    &\sum_{t=1}^{T} \left\| \nabla f_t(\mathbf{x}_t, \mathbf{y}_t^*(\mathbf{x}_t)) - \widetilde{\nabla} f_t(\mathbf{x}_t, \mathbf{y}_{t+1}, \mathbf{v}_{t+1}) \right\|^2 \nonumber\\
    \leq& \frac{\Delta_1'}{C_1} + \frac{2\gamma^2C_\mathbf{x}}{(1-\rho)C_1}\sum_{t=1}^{T}\left\|\mathcal{G}_\mathcal{X} (\mathbf{x}_t, \nabla f_t(\mathbf{x}_t, \mathbf{y}_t^*(\mathbf{x}_t)))\right\|^2 \nonumber\\
    & + \frac{C_\alpha}{(1-\rho)C_1}H_{2,T} + \frac{C_P}{(1-\rho)C_1}E_{\mathbf{y},T}^f + \frac{L_{f,0}^2C_P}{\mu_g^2(1-\rho)C_1}E_{\mathbf{y}\mathbf{y},T}^g. \label{A.14_4}
\end{align}

Now, remind of Lemma~\ref{lem:begin}, we still have 
\begin{align}
    &\sum_{t=1}^{T} \left\| \mathcal{G}_\mathcal{X}\left( \mathbf{x}_t, \nabla f_t(\mathbf{x}_t, \mathbf{y}_t^*(\mathbf{x}_t)), \gamma \right) \right\|^2 \nonumber\\
    \leq& \frac{2}{\gamma\theta} \sum_{t=1}^{T} \left( f_t(\mathbf{x}_t, \mathbf{y}_t^*(\mathbf{x}_t)) - f_t(\mathbf{x}_{t+1}, \mathbf{y}_t^*(\mathbf{x}_{t+1})) \right) \nonumber\\
    & + \left(2+\frac{1}{\lambda\theta}\right) \sum_{t=1}^{T} \left\| \nabla f_t(\mathbf{x}_t, \mathbf{y}_t^*(\mathbf{x}_t)) - \widetilde{\nabla} f_t(\mathbf{x}_t, \mathbf{y}_{t+1}, \mathbf{v}_{t+1}) \right\|^2, \label{A.14_5}
\end{align}
after substituting (\ref{A.14_4}) into (\ref{A.14_5}), we finally obtain
\begin{align*}
    &\sum_{t=1}^{T} \left\| \mathcal{G}_\mathcal{X}\left( \mathbf{x}_t, \nabla f_t(\mathbf{x}_t, \mathbf{y}_t^*(\mathbf{x}_t)), \gamma \right) \right\|^2 \\
    \leq& \frac{2}{\gamma\theta}(2Q + V_T) + \left(2+\frac{1}{\lambda\theta}\right)\left(\frac{\Delta_1'}{C_1} {+} \frac{C_\alpha}{(1-\rho)C_1}H_{2,T} {+} \frac{C_P}{(1-\rho)C_1}E_{\mathbf{y},T}^f {+} \frac{L_{f,0}^2C_P}{\mu_g^2(1-\rho)C_1}E_{\mathbf{y}\mathbf{y},T}^g\right) \\
    & + \left(2+\frac{1}{\lambda\theta}\right)\frac{2\gamma^2C_\mathbf{x}}{(1-\rho)C_1}\sum_{t=1}^{T}\left\|\mathcal{G}_\mathcal{X} (\mathbf{x}_t, \nabla f_t(\mathbf{x}_t, \mathbf{y}_t^*(\mathbf{x}_t)))\right\|^2
\end{align*}
let $\gamma \leq \frac{1}{2L_F}$ and $\lambda = 1$, then $\theta = 1 - \frac{1}{2} - \frac{\gamma L_F}{2} \geq \frac{1}{4}$, which implies that
\begin{align*}
    &\left(1 - \frac{12\gamma^2C_\mathbf{x}}{(1-\rho)C_1}\right)\sum_{t=1}^{T} \left\| \mathcal{G}_\mathcal{X}\left( \mathbf{x}_t, \nabla f_t(\mathbf{x}_t, \mathbf{y}_t^*(\mathbf{x}_t)), \gamma \right) \right\|^2 \\
    \leq& \frac{16Q}{\gamma} + \frac{8}{\gamma}V_T + \frac{6\Delta_1'}{C_1} + \frac{6C_\alpha}{(1-\rho)C_1}H_{2,T} + \frac{6C_P}{(1-\rho)C_1}E_{\mathbf{y},T}^f + \frac{6L_{f,0}^2C_P}{\mu_g^2(1-\rho)C_1}E_{\mathbf{y}\mathbf{y},T}^g,
\end{align*}
To ensure $C_2 := 1 - \frac{12\gamma^2C_\mathbf{x}}{(1-\rho)C_1} > 0$ we set
\begin{align*}
    \gamma \leq \sqrt{\frac{(1-\rho)C_1}{12C_\mathbf{x}}} < \sqrt{\frac{1-\rho}{2C_\mathbf{x}}}.
\end{align*}
Finally, we obtain
\begin{align*}
    &\sum_{t=1}^{T} \left\| \mathcal{G}_\mathcal{X}\left( \mathbf{x}_t, \nabla f_t(\mathbf{x}_t, \mathbf{y}_t^*(\mathbf{x}_t)), \gamma \right) \right\|^2 \\
    \leq& \frac{16Q}{\gamma C_2} + \frac{8}{\gamma C_2}V_T + \frac{6\Delta_1'}{C_1C_2} + \frac{6C_\alpha}{(1-\rho)C_1C_2}H_{2,T} + \frac{6C_P}{(1-\rho)C_1C_2}E_{\mathbf{y},T}^f + \frac{6L_{f,0}^2C_P}{\mu_g^2(1-\rho)C_1C_2}E_{\mathbf{y}\mathbf{y},T}^g \\
    =& O(1 + V_T + H_{2,T} + E_{2,T}),
\end{align*}
where $E_{2,T} := E_{\mathbf{y},T}^f + E_{\mathbf{y}\mathbf{y},T}^g$ to finish the proof.
\end{proof}

\subsection{Proof of SOBOW under Standard Bilevel Local Regret}\label{sec:proof_SOBOW}

In this section, we prove that SOBOW \cite{lin2023non} can obtain an upper bound $O(1 + V_T + H_{2,T})$ under regret defined in (\ref{eq:reg}) with $w=1$.
\begin{lemma}[\citet{lin2023non}]\label{lem:y_bound_1_sobow}
Under Assumptions~\ref{asm1},~\ref{asm2}, consider the inner iteration process of $g_t(\mathbf{x}_t, \cdot)$ in Algorithm~\ref{alg:CIGO}, set $\alpha < \frac{1}{L_{g,1}}$ we can have
\begin{align*}
    \|\mathbf{y}_{t+1} - \mathbf{y}_t^*(\mathbf{x}_t)\|^2 \leq& \left(1 - \frac{\alpha\mu_g}{2}\right) \|\mathbf{y}_t - \mathbf{y}_{t-1}^*(\mathbf{x}_{t-1})\|^2 \\
    &+ 2\left(1 + \frac{2}{\alpha\mu_g}\right) (1-\alpha\mu_g) \|\mathbf{y}_{t-1}^*(\mathbf{x}_{t-1}) - \mathbf{y}_t^*(\mathbf{x}_{t-1})\|^2 \\
    & + 2\left(1 + \frac{2}{\alpha\mu_g}\right) \frac{(1 - \alpha\mu_g) L_{g,1}^2}{\mu_g^2}\|\mathbf{x}_{t-1} - \mathbf{x}_t\|^2.
\end{align*}
\end{lemma}

\begin{lemma}[Based on Lemma 5.5 in \citet{lin2023non}]\label{lem:v_bound_sobow}
Under Assumptions~\ref{asm1},~\ref{asm2}, let $\alpha \leq \frac{1}{L_{g,1}}$, $\lambda \leq \frac{1}{L_{g,1}}$ and $M_{t+1} - M_t \geq \frac{\log(1-\frac{\alpha\mu_g}{2})}{2\log(1-\lambda\mu_g)}$. We can have that
\begin{align*}
    \left\|\mathbf{v}_{t+1} - \mathbf{v}_t^*(\mathbf{x}, \mathbf{y}_t^*(\mathbf{x}_t))\right\|^2 \leq c\left\|\mathbf{y}_{t+1} - \mathbf{y}_t^*(\mathbf{x}_t)\right\|^2 + \epsilon_t^2,
\end{align*}
where $c > 0$ is some constant and the error $\epsilon_t^2$ decays with $t$, i.e., $\epsilon_{t+1}^2 \leq \left(1-\frac{\alpha\mu_g}{2}\right)\epsilon_t^2$.
\end{lemma}
\begin{remark}
\cite{lin2023non} Lemma 5.5 above was obtained under weaker assumptions. They assumed that there exists at least one point $\mathbf{x}'$ in $\mathcal{X}$ such that the gradient $\left\|\nabla_\mathbf{y} f_t(\mathbf{x}', \mathbf{y}_t^*(\mathbf{x}'))\right\|$ has an upper bound $\rho$. We do not have this assumption, but note that the upper bound of the gradient still has $L_{f,0}$, due to the lipschitz continuity in Assumption~\ref{asm2}, the above lemma still holds under our stronger assumptions.
\end{remark}

\begin{algorithm}[tb]
    \caption{SOBOW \cite{lin2023non} without window averaging}
    \label{alg:CIGO}
    \textbf{Input}: $\mathbf{x}_1$, $\mathbf{y}_1$, $\mathbf{v}_1$, $\alpha$, $\beta$, $\gamma$. \\
    \textbf{Output}: Decision sequences: $\{\mathbf{x}_t\}_{t=1}^T$, $\{\mathbf{y}_t\}_{t=1}^T$. 
    
    \begin{algorithmic}[1] 
        \FOR{\(t = 1\) to \(T\)}
        \STATE Output $\mathbf{x}_t$ and $\mathbf{y}_t$ and receive feedback $f_t$ and $g_t$
        \STATE
        \( \mathbf{y}_{t+1} \leftarrow \mathbf{y}_t - \alpha \nabla_{\mathbf{y}} g_t(\mathbf{x}_t,\mathbf{y}_t) \)
        \STATE Solve $\mathop{\arg\min}_{\mathbf{v}\in\mathcal{V}}\Phi_t(\mathbf{x}_t, \mathbf{y}_{t+1}, \mathbf{v})$ defined in (\ref{eq:Phi_t}) using $M_t$ steps of conjugate gradient starting from a fixed $\mathbf{v}^0$ with step size $\lambda$ to obtain $\mathbf{v}_{t+1}$
        \STATE Calculate $\widetilde{\nabla}f_t(\mathbf{x}_t, \mathbf{y}_{t+1}, \mathbf{v}_{t+1})$ by (\ref{eq:hypergrad}) and use it to update $\mathbf{x}_{t+1}$ by (\ref{eq:x_update})
        \ENDFOR
    \end{algorithmic}
\end{algorithm}

\begin{theorem}\label{thm:CIGO_upper}
Under Assumptions~\ref{asm1}-\ref{asm3}, let $\alpha \leq \frac{1}{L_{g,1}}$, $\lambda \leq \frac{1}{L_{g,1}}$, $M_{t+1} - M_t \geq \frac{\log(1-\frac{\alpha\mu_g}{2})}{2\log(1-\lambda\mu_g)}$ and $\gamma < \min\left\{\frac{1}{2L_F}, \frac{\alpha\mu_g}{4\kappa_g}\sqrt{\frac{C_1}{6C_\mathbf{y}}}\right\}$ where $L_F$ is the smoothness parameter of $f_t(\mathbf{x}, \mathbf{y}_t^*(\mathbf{x}))$ with respect to $\mathbf{x}$, Algorithm~\ref{alg:CIGO} can guarantee
\begin{align*}
    \mathrm{Reg}(T) \leq& \frac{16Q}{\gamma C_2} + \frac{8}{\gamma C_2}V_T + \frac{6}{C_2}\left\|\nabla f_1(\mathbf{x}_1, \mathbf{y}_1^*(\mathbf{x}_1)) {-} \widetilde{\nabla}f_1(\mathbf{x}_1, \mathbf{y}_2, \mathbf{v}_2)\right\|^2 + \frac{6\Delta_1}{(1{-}\rho)C_1C_2} \\
    & + \frac{24C_\mathbf{y}}{\alpha\mu_g(1{-}\rho)C_1C_2}H_{2,T}= O(1 + V_T + H_{2,T}),
\end{align*}
where $\rho=1-\frac{\alpha\mu_g}{2}$, $C_1$, $C_2$ and $C_\mathbf{y}$ are some constants. 
\end{theorem}
\begin{proof}
It begins with (\ref{A.11_2}),
\begin{align}
    &\left\| \nabla f_t(\mathbf{x}_t, \mathbf{y}_t^*(\mathbf{x}_t)) - \widetilde{\nabla} f_t(\mathbf{x}_t, \mathbf{y}_{t+1}, \mathbf{v}_{t+1}) \right\|^2 \nonumber\\
    \leq& \left(2L_{f,1}^2 + \frac{4L_{f,0}^2L_{g,2}^2}{\mu_g^2}\right)\left\|\mathbf{y}_{t+1} - \mathbf{y}_t^*(\mathbf{x}_t)\right\|^2 + 4L_{g,1}^2\left\|\mathbf{v}_{t+1} - \mathbf{v}_t^*(\mathbf{x}_t)\right\|^2, \label{A.12_1}
\end{align}
with substituting Lemma~\ref{lem:v_bound_sobow} into (\ref{A.12_1}), we have
\begin{align*}
    &\left\| \nabla f_t(\mathbf{x}_t, \mathbf{y}_t^*(\mathbf{x}_t)) - \widetilde{\nabla} f_t(\mathbf{x}_t, \mathbf{y}_{t+1}, \mathbf{v}_{t+1}) \right\|^2 \\
    \leq& \left(2L_{f,1}^2 + \frac{4L_{f,0}^2L_{g,2}^2}{\mu_g^2} + c\right)\left\|\mathbf{y}_{t+1} - \mathbf{y}_t^*(\mathbf{x}_t)\right\|^2 + 4L_{g,1}^2\epsilon_t^2 =: \Delta_t,
\end{align*}
for simplicity, define $C_\mathbf{y} := 2L_{f,1}^2 + \frac{4L_{f,0}^2L_{g,2}^2}{\mu_g^2} + c$. With Lemma~\ref{lem:y_bound_1_sobow}, it holds that
\begin{align}
    \Delta_t \leq& C_\mathbf{y}\rho\left\|\mathbf{y}_t - \mathbf{y}_{t-1}^*(\mathbf{x}_{t-1})\right\|^2 + 4\rho L_{g,1}^2\epsilon_{t-1}^2 \nonumber\\
    & + 2C_\mathbf{y}\left(1 + \frac{2}{\alpha\mu_g}\right)(1 - \alpha\mu_g)\left\|\mathbf{y}_{t-1}^*(\mathbf{x}_{t-1}) - \mathbf{y}_t^*(\mathbf{x}_{t-1})\right\|^2 \nonumber\\
    & + 2\kappa_g^2C_\mathbf{y}\left(1 + \frac{2}{\alpha\mu_g}\right)(1 - \alpha\mu_g)\|\mathbf{x}_{t-1} - \mathbf{x}_t\|^2 \nonumber\\
    \leq& \rho\Delta_{t-1} + \frac{4C_\mathbf{y}}{\alpha\mu_g} \left\|\mathbf{y}_{t-1}^*(\mathbf{x}_{t-1}) - \mathbf{y}_t^*(\mathbf{x}_{t-1})\right\|^2 + \frac{4\kappa_g^2C_\mathbf{y}}{\alpha\mu_g} \|\mathbf{x}_{t-1} - \mathbf{x}_t\|^2 \nonumber\\
    \leq& \rho^{t-1}\Delta_1 + \frac{4C_\mathbf{y}}{\alpha\mu_g} \sum_{j=0}^{t-2} \rho^j \left\|\mathbf{y}_{t-1-j}^*(\mathbf{x}_{t-1-j}) - \mathbf{y}_{t-j}^*(\mathbf{x}_{t-1-j})\right\|^2 \nonumber\\
    & + \frac{4\kappa_g^2C_\mathbf{y}}{\alpha\mu_g} \sum_{j=0}^{t-2} \rho^j \|\mathbf{x}_{t-1-j} - \mathbf{x}_{t-j}\|^2. \label{A.12_2}
\end{align}
Finally, substituting (\ref{A.12_2}) into (\ref{A.12_1}) we have
\begin{align}
    &\|\nabla f_t (\mathbf{x}_t, \mathbf{y}^*_t(\mathbf{x}_t)) - \widetilde{\nabla} f_t(\mathbf{x}_t, \mathbf{y}_{t+1}, \mathbf{v}_{t+1})\|^2 \nonumber\\
    \leq& \rho^{t-1}\Delta_1 + \frac{4C_\mathbf{y}}{\alpha\mu_g} \sum_{j=0}^{t-2} \rho^j \left\|\mathbf{y}_{t-1-j}^*(\mathbf{x}_{t-1-j}) {-} \mathbf{y}_{t-j}^*(\mathbf{x}_{t-1-j})\right\|^2 + \frac{4\kappa_g^2C_\mathbf{y}}{\alpha\mu_g} \sum_{j=0}^{t-2} \rho^j \|\mathbf{x}_{t-1-j} {-} \mathbf{x}_{t-j}\|^2. \label{A.12_3}
\end{align}
Then, summing (\ref{A.12_3}) over $t=2,\dots,T$ to obtain
\begin{align}
    &\sum_{t=2}^T\|\nabla f_t (\mathbf{x}_t, \mathbf{y}^*_t(\mathbf{x}_t)) - \widetilde{\nabla} f_t(\mathbf{x}_t, \mathbf{y}_{t+1}, \mathbf{v}_{t+1})\|^2 \nonumber\\ 
    \leq& \sum_{t=2}^T\rho^{t-1} \Delta_1 + \frac{4C_\mathbf{y}}{\alpha\mu_g} \sum_{t=2}^T\sum_{j=0}^{t-2} \rho^j \|\mathbf{y}_{t-1-j}^*(\mathbf{x}_{t-1-j}) - \mathbf{y}_{t-j}^*(\mathbf{x}_{t-1-j})\|^2 \nonumber\\
    & + \frac{4\kappa_g^2C_\mathbf{y}}{\alpha\mu_g} \sum_{t=2}^T\sum_{j=0}^{t-2} \rho^j \|\mathbf{x}_{t-1-j} - \mathbf{x}_{t-j}\|^2 \nonumber\\
    \overset{(i)}\leq& \frac{\Delta_1}{1-\rho} + \frac{4C_\mathbf{y}}{\alpha\mu_g(1-\rho)} \sum_{t=2}^T \|\mathbf{y}_{t-1}^*(\mathbf{x}_{t-1}) - \mathbf{y}_t^*(\mathbf{x}_{t-1})\|^2 + \frac{4\kappa_g^2C_\mathbf{y}}{\alpha\mu_g(1-\rho)} \sum_{t=2}^T\|\mathbf{x}_{t-1} - \mathbf{x}_t\|^2 \nonumber\\
    \overset{(ii)}\leq& \frac{\Delta_1}{1-\rho} + \frac{4C_\mathbf{y}}{\alpha\mu_g(1-\rho)} H_{2,T} + \frac{4\kappa_g^2C_\mathbf{y}}{\alpha\mu_g(1-\rho)} \sum_{t=1}^T \left\|\mathbf{x}_{t-1} - \mathbf{x}_t\right\|^2, \label{A.12_4}
\end{align}
where $(i)$ comes from the fact that $\sum_{j=0}^{t-2}\rho^j < \sum_{j=0}^\infty\rho^j = \frac{1}{1-\rho}$ for $\forall\rho \in(0,1)$, and $(ii)$ follows the definition of $H_{2,T}$ in (\ref{pathV}). Then note that
\begin{align}
    \mathbf{x}_t - \mathbf{x}_{t-1} = \gamma\mathcal{G}_\mathcal{X} (\mathbf{x}_{t-1}, \widetilde{\nabla}f_{t-1}(\mathbf{x}_{t-1}, \mathbf{y}_t, \mathbf{v}_t), \gamma), \label{eq:x_update_2}
\end{align}
and substitute it into (\ref{A.12_4}) we have
\begin{align*}
    &\sum_{t=2}^T\|\nabla f_t (\mathbf{x}_t, \mathbf{y}^*_t(\mathbf{x}_t)) - \widetilde{\nabla} f_t(\mathbf{x}_t, \mathbf{y}_{t+1}, \mathbf{v}_{t+1})\|^2 \\ 
    \leq& \frac{\Delta_1}{1-\rho} + \frac{4C_\mathbf{y}}{\alpha\mu_g(1-\rho)} H_{2,T} + \frac{4\kappa_g^2C_\mathbf{y}}{\alpha\mu_g(1-\rho)} \sum_{t=1}^T \left\|\mathcal{G}_\mathcal{X}(\mathbf{x}_t, \widetilde{\nabla} f_t(\mathbf{x}_t, \mathbf{y}_{t+1}, \mathbf{v}_{t+1}), \gamma)\right\|^2 \\
    &\left(1 - \frac{8\gamma^2\kappa_g^2C_\mathbf{y}}{\alpha\mu_g(1-\rho)} \right)\sum_{t=2}^T\|\nabla f_t (\mathbf{x}_t, \mathbf{y}^*_t(\mathbf{x}_t)) - \widetilde{\nabla} f_t(\mathbf{x}_t, \mathbf{y}_{t+1}, \mathbf{v}_{t+1})\|^2 \\ 
    \leq& \frac{\Delta_1}{1-\rho} + \frac{4C_\mathbf{y}}{\alpha\mu_g(1-\rho)} H_{2,T} + \frac{8\gamma^2\kappa_g^2C_\mathbf{y}}{\alpha\mu_g(1-\rho)} \sum_{t=1}^T \left\|\mathcal{G}_\mathcal{X}(\mathbf{x}_t, \nabla f_t(\mathbf{x}_t, \mathbf{y}_t^*(\mathbf{x}_t)), \gamma)\right\|^2,
\end{align*}
let $C_1 := 1 - \frac{8\gamma^2\kappa_g^2C_\mathbf{y}}{\alpha\mu_g(1-\rho)} > 0$, set $\gamma < \frac{\alpha\mu_g}{4\kappa_g}\sqrt{\frac{1}{C_\mathbf{y}}}$, it then holds that
\begin{align}
    &\sum_{t=2}^T\|\nabla f_t (\mathbf{x}_t, \mathbf{y}^*_t(\mathbf{x}_t)) - \widetilde{\nabla} f_t(\mathbf{x}_t, \mathbf{y}_{t+1}, \mathbf{v}_{t+1})\|^2 \nonumber\\
    \leq& \frac{\Delta_1}{(1-\rho)C_1} + \frac{4C_\mathbf{y}}{\alpha\mu_g(1-\rho)C_1}H_{2,T} + \frac{8\gamma^2\kappa_g^2C_\mathbf{y}}{\alpha\mu_g(1-\rho)C_1}\sum_{t=1}^T \left\|\mathcal{G}_\mathcal{X}(\mathbf{x}_t, \nabla f_t(\mathbf{x}_t, \mathbf{y}_t^*(\mathbf{x}_t)), \gamma)\right\|^2. \label{A.12_5}
\end{align}

Now, remind of Lemma~\ref{lem:begin},
\begin{align}
    &\sum_{t=1}^{T} \left\| \mathcal{G}_\mathcal{X}\left( \mathbf{x}_t, \nabla f_t(\mathbf{x}_t, \mathbf{y}_t^*(\mathbf{x}_t)), \gamma \right) \right\|^2 \nonumber\\
    \leq& \frac{2}{\gamma\theta} \sum_{t=1}^{T} \left( f_t(\mathbf{x}_t, \mathbf{y}_t^*(\mathbf{x}_t)) - f_t(\mathbf{x}_{t+1}, \mathbf{y}_t^*(\mathbf{x}_{t+1})) \right) \nonumber\\
    & + \left(2+\frac{1}{\lambda\theta}\right) \sum_{t=1}^{T} \left\| \nabla f_t(\mathbf{x}_t, \mathbf{y}_t^*(\mathbf{x}_t)) - \widetilde{\nabla} f_t(\mathbf{x}_t, \mathbf{y}_{t+1}, \mathbf{v}_{t+1}) \right\|^2, \label{A.12_6}
\end{align}
and substituting (\ref{A.12_5}) into it, we have
\begin{align}
    &\sum_{t=1}^{T} \left\| \mathcal{G}_\mathcal{X}\left( \mathbf{x}_t, \nabla f_t(\mathbf{x}_t, \mathbf{y}_t^*(\mathbf{x}_t)), \gamma \right) \right\|^2 \nonumber\\
    \leq& \frac{4Q}{\gamma\theta} + \frac{2}{\gamma\theta}V_T + \left(2+\frac{1}{\lambda\theta}\right)\left\|\nabla f_1(\mathbf{x}_1, \mathbf{y}_1^*(\mathbf{x}_1)) - \widetilde{\nabla}f_1(\mathbf{x}_1, \mathbf{y}_2, \mathbf{v}_2)\right\|^2 + \left(2+\frac{1}{\lambda\theta}\right)\frac{\Delta_1}{(1-\rho)C_1} \nonumber\\
    & + \left(2+\frac{1}{\lambda\theta}\right)\frac{4C_\mathbf{y}}{\alpha\mu_g(1-\rho)C_1}H_{2,T} \nonumber\\
    & + \left(2+\frac{1}{\lambda\theta}\right)\frac{8\gamma^2\kappa_g^2C_\mathbf{y}}{\alpha\mu_g(1-\rho)C_1}\sum_{t=1}^T\left\|\mathcal{G}_\mathcal{X}(\mathbf{x}_t, \nabla f_t(\mathbf{x}_t, \mathbf{y}_t^*(\mathbf{x}_t)), \gamma)\right\|^2,
\end{align}
let $\lambda=1$ and $\gamma\leq \frac{1}{2L_F}$, thus $\theta = 1 - \frac{1}{2} - \frac{\gamma L_F}{2} \geq \frac{1}{4}$, we obtain
\begin{align*}
    &\left(1 - \frac{48\gamma^2\kappa_g^2C_\mathbf{y}}{\alpha\mu_g(1-\rho)C_1}\right)\sum_{t=1}^{T} \left\| \mathcal{G}_\mathcal{X}\left( \mathbf{x}_t, \nabla f_t(\mathbf{x}_t, \mathbf{y}_t^*(\mathbf{x}_t)), \gamma \right) \right\|^2 \\
    \leq& \frac{16Q}{\gamma} + \frac{8}{\gamma}V_T + 6\left\|\nabla f_1(\mathbf{x}_1, \mathbf{y}_1^*(\mathbf{x}_1)) - \widetilde{\nabla}f_1(\mathbf{x}_1, \mathbf{y}_2, \mathbf{v}_2)\right\|^2 + \frac{6\Delta_1}{(1-\rho)C_1} + \frac{24C_\mathbf{y}
    }{\alpha\mu_g(1-\rho)C_1}H_{2,T},
\end{align*}
with $\gamma \leq \frac{\alpha\mu_g}{4\kappa_g}\sqrt{\frac{C_1}{6C_\mathbf{y}}}$ and $C_2 := 1 - \frac{48\gamma^2\kappa_FL_{g,1}^2}{(1-\rho)C_1}$, which implies that
\begin{align*}
    &\sum_{t=1}^{T} \left\| \mathcal{G}_\mathcal{X}\left( \mathbf{x}_t, \nabla f_t(\mathbf{x}_t, \mathbf{y}_t^*(\mathbf{x}_t)), \gamma \right) \right\|^2 \\
    \leq& \frac{16}{\gamma C_2} + \frac{8}{\gamma C_2}V_T + \frac{6}{C_2}\left\|\nabla f_1(\mathbf{x}_1, \mathbf{y}_1^*(\mathbf{x}_1)) - \widetilde{\nabla}f_1(\mathbf{x}_1, \mathbf{y}_2, \mathbf{v}_2)\right\|^2 + \frac{6\Delta_1}{(1-\rho)C_1C_2} \\
    & + \frac{24C_\mathbf{y}}{\alpha\mu_g(1-\rho)C_1C_2}H_{2,T} = O(1 + V_T + H_{2,T}).
\end{align*}
Thus we finish the proof.
\end{proof}

\subsection{Proof of OBBO under Standard Bilevel Local Regret}\label{sec:proof_OBBO}

\begin{algorithm}[tb]
    \caption{OBBO \cite{bohne2024online} without window averaging}
    \label{alg:OBBO}
    \textbf{Input}: $\mathbf{x}_1$, $\mathbf{y}_1$, $\eta$, $\alpha$. \\
    \textbf{Output}: Decision sequences: $\{\mathbf{x}_t\}_{t=1}^T$, $\{\mathbf{y}_t\}_{t=1}^T$. 
    
    \begin{algorithmic}[1] 
        \FOR{\(t = 1\) to \(T\)}
        \STATE Output $\mathbf{x}_t$ and $\mathbf{y}_t$ and receive feedback $f_t$ and $g_t$
        \STATE \(\mathbf{y}_t^0 \leftarrow \mathbf{y}_t\)
        \FOR{\(k = 1\) to \(K\)}
        \STATE \( \mathbf{y}_t^k \leftarrow \mathbf{y}_t^{k-1} - \alpha \nabla_{\mathbf{y}} g_t(\mathbf{x}_t,\mathbf{y}_t) \)
        \ENDFOR
        \STATE Calculate $\widetilde{\nabla}f_t(\mathbf{x}_t, \mathbf{y}_t^K) := \frac{\partial f_t(\mathbf{x}_t, \mathbf{y}_t^K)}{\partial\mathbf{x}}$ by (\ref{eq:ITD}) and set $\mathbf{y}_{t+1} = \mathbf{y}_t^K$
        \STATE use $\widetilde{\nabla}f_t(\mathbf{x}_t, \mathbf{y}_{t+1})$ to update $\mathbf{x}_{t+1}$ by (\ref{eq:x_update})
        \ENDFOR
    \end{algorithmic}
\end{algorithm}

In this section, we prove that OBBO \cite{bohne2024online} can obtain an upper bound $O(1 + V_T + H_{2,T})$ under regret defined in (\ref{eq:reg}) with $w=1$. The following lemma shows the ITD approach approximating the hypergradient used by Algorithm~\ref{alg:OBBO}.
\begin{lemma}[Lemma 4.2 in \citet{bohne2024online}]
The partial $\frac{\partial f_t(\mathbf{x}_t, \mathbf{y}_t^K)}{\partial\mathbf{x}}$ takes an analytical form of 
\begin{align}
    \frac{\partial f_t(\mathbf{x}_t, \mathbf{y}_t^K)}{\partial\mathbf{x}} = \nabla_\mathbf{x}f_t(\mathbf{x}_t, \mathbf{y}_t^K) - \eta\sum_{k=0}^{K-1}\nabla_{\mathbf{x}\mathbf{y}}^2g_t(\mathbf{x}_t, \mathbf{y}_t^k)\mathbf{H}_{\mathbf{y}\mathbf{y}}\nabla_\mathbf{y}f_t(\mathbf{x}_t, \mathbf{y}_t^K), \label{eq:ITD}
\end{align}
where $\mathbf{H}_{\mathbf{y}\mathbf{y}} := \prod_{j=k+1}^{K-1}\left(\mathbf{I}_{d_2} - \eta\nabla_{\mathbf{y}\mathbf{y}}^2g_t(\mathbf{x}_t, \mathbf{y}_t^j)\right)$.
\end{lemma}

The following lemma gives the hypergradient estimation error upper bound at time step $t$ in OBBO.
\begin{lemma}[Lemma 5.1 in \citet{bohne2024online}]\label{lem:obbo}
Under Assumptions~\ref{asm1}-\ref{asm3}, let $\eta\leq\min\{\frac{1}{\mu_g}, \frac{1}{L_{g,1}}\}$ and $K\geq1$, we then obtain
\begin{align*}
    &\left\|\frac{\partial f_t(\mathbf{x}_t, \mathbf{y}_t^K)}{\partial \mathbf{x}} - \nabla F_t(\mathbf{x}_t)\right\|^2 \\
    \leq& \delta_t + A\sum_{j=0}^{t-2}\nu^j\left\|\frac{\partial f_{t-1-j,w}(\mathbf{x}_{t-1-j}, \mathbf{y}_{t-1-j}^K)}{\partial \mathbf{x}} - \nabla F_{t-1-j,w}(\mathbf{x}_{t-1-j})\right\|^2 \\
    & + B\sum_{j=0}^{t-2}\nu^j\left\|\mathcal{G}_\mathcal{X}(\mathbf{x}_{t-1-j}, \nabla F_{t-1-j,w}(\mathbf{x}_{t-1-j}), \alpha)\right\|^2 \\
    & + C\sum_{j=0}^{t-2}\nu^j\left\|\mathbf{y}_{t-j}^*(\mathbf{x}_{t-1-j}) - \mathbf{y}_{t-1-j}^*(\mathbf{x}_{t-1-j})\right\|^2,
\end{align*}
where $\delta_t = 3L_3^2(1-\eta\mu_g)^{2K} + 3L_\beta\nu^{t-1}\Delta_\beta$, $\nu = (1-\frac{\eta\mu_g}{2})(1-\eta\mu_g)^{K-1}$ with $L_3$, $L_\beta$, $\Delta_\beta$, $A$, $B$ and $C$ are some constants.
\end{lemma}

\begin{theorem}\label{thm:improved_OBBO}
Under Assumptions~\ref{asm1}-\ref{asm3}, let $\eta\leq\min\{\frac{1}{\mu_g}, \frac{1}{L_{g,1}}\}$, $\alpha\leq\min\left\{\frac{1}{2L_F}, \frac{\rho\sqrt{1-\nu}}{\kappa_g\sqrt{108C_{\mu_g}L_\beta}}, \frac{\sqrt{1-\nu}}{9\kappa_g\sqrt{C_{\mu_g}L_\beta}}\right\}$ and $K = \frac{\log T}{\log ((1-\eta\mu_g)^{-1})} + 1$ where $L_F$ is the smoothness parameter of $f_t(\mathbf{x}, \mathbf{y}_t^*(\mathbf{x}))$ with respect to $\mathbf{x}$, Algorithm~\ref{alg:OBBO} can guarantee
\begin{align*}
    \mathrm{Reg}(T) \leq& \frac{16Q + 8V_T}{c\gamma} + \frac{51}{4c}\left(\frac{\Delta_\beta L_\beta}{1-\nu} + L_3^2\right) + \frac{27CH_{2,T}}{4c(1-\nu)} + \frac{6}{c}\left\|\nabla F_t(\mathbf{x}_t) - \frac{\partial f_1(\mathbf{x}_1, \mathbf{y}_1^K)}{\partial\mathbf{x}}\right\|^2 \\
    =& O(1 + V_T + H_{2,T}),
\end{align*}
where $c$, $C$, $\Delta_\beta$, $L_\beta$ are some constants, $\nu = (1-\frac{\eta\mu_g}{2})(1-\eta\mu_g)$.
\end{theorem}
\begin{proof}
It begins with Lemma~\ref{lem:obbo} with $w=1$, we have
\begin{align*}
    &\left\|\frac{\partial f_t(\mathbf{x}_t, \mathbf{y}_t^K)}{\partial \mathbf{x}} - \nabla F_t(\mathbf{x}_t)\right\|^2 \\
    \leq& \delta_t + A\sum_{j=0}^{t-2}\nu^j\left\|\frac{\partial f_{t-1-j}(\mathbf{x}_{t-1-j}, \mathbf{y}_{t-1-j}^K)}{\partial \mathbf{x}} - \nabla F_{t-1-j}(\mathbf{x}_{t-1-j})\right\|^2 \\
    & + B\sum_{j=0}^{t-2}\nu^j\left\|\mathcal{G}_\mathcal{X}(\mathbf{x}_{t-1-j}, \nabla F_{t-1-j}(\mathbf{x}_{t-1-j}), \alpha)\right\|^2 \\
    & + C\sum_{j=0}^{t-2}\nu^j\left\|\mathbf{y}_{t-j}^*(\mathbf{x}_{t-1-j}) - \mathbf{y}_{t-1-j}^*(\mathbf{x}_{t-1-j})\right\|^2,
\end{align*}
summing the above inequality, we then obtain
\begin{align*}
    &\sum_{t=2}^T\left\|\frac{\partial f_t(\mathbf{x}_t, \mathbf{y}_t^K)}{\partial \mathbf{x}} - \nabla F_t(\mathbf{x}_t)\right\|^2 \\
    \leq& 3L_3^2T(1-\eta\mu_g)^{2K} + 3L_\beta\Delta_\beta\sum_{t=2}^T\nu^{t-1} \\
    & + A\sum_{t=2}^T\sum_{j=0}^{t-2}\nu^j\left\|\frac{\partial f_{t-1-j}(\mathbf{x}_{t-1-j}, \mathbf{y}_{t-1-j}^K)}{\partial \mathbf{x}} - \nabla F_{t-1-j}(\mathbf{x}_{t-1-j})\right\|^2 \\
    & + B\sum_{t=2}^T\sum_{j=0}^{t-2}\nu^j\left\|\mathcal{G}_\mathcal{X}(\mathbf{x}_{t-1-j}, \nabla F_{t-1-j}(\mathbf{x}_{t-1-j}), \alpha)\right\|^2 \\
    & + C\sum_{t=2}^T\sum_{j=0}^{t-2}\nu^j\left\|\mathbf{y}_{t-j}^*(\mathbf{x}_{t-1-j}) - \mathbf{y}_{t-1-j}^*(\mathbf{x}_{t-1-j})\right\|^2 \\
    \leq& 3L_3^2 + \frac{3L_\beta\Delta_\beta}{1-\nu} + \frac{A}{1-\nu}\sum_{t=2}^T\left\|\frac{\partial f_t(\mathbf{x}_t, \mathbf{y}_t^K)}{\partial \mathbf{x}} - \nabla F_t(\mathbf{x}_t)\right\|^2 \\
    & + \frac{B}{1-\nu}\sum_{t=1}^T\left\|\mathcal{G}_\mathcal{X}(\mathbf{x}_{t}, \nabla F_{t}(\mathbf{x}_{t}), \alpha)\right\|^2 \\
    & + \frac{C}{1-\nu}\sum_{t=2}^T\left\|\mathbf{y}_{t}^*(\mathbf{x}_{t-1}) - \mathbf{y}_{t-1}^*(\mathbf{x}_{t-1})\right\|^2,
\end{align*}
where the last inequality comes from the fact that $\sum_{j=0}^{t-2}\nu^j < \sum_{j=0}^\infty\nu^j = \frac{1}{1-\nu}$ for $\forall\rho \in(0,1)$. After rearranging the terms, it holds that
\begin{align}
    \left(1 - \frac{A}{1-\nu}\right)\sum_{t=2}^T\left\|\frac{\partial f_t(\mathbf{x}_t, \mathbf{y}_t^K)}{\partial \mathbf{x}} - \nabla F_t(\mathbf{x}_t)\right\|^2 \leq& 3L_3^2 + \frac{B}{1-\nu}\sum_{t=1}^T\left\|\mathcal{G}_\mathcal{X}(\mathbf{x}_{t}, \nabla F_{t}(\mathbf{x}_{t}), \alpha)\right\|^2 \nonumber\\
    & + \frac{3L_\beta\Delta_\beta}{1-\nu}  + \frac{CH_{2,T}}{1-\nu}, \label{OBBO_1}
\end{align}
with the expression of $A$ in Lemma B.3 \citet{bohne2024online}, 
\begin{align*}
    A = \frac{12\alpha^2C_{\mu_g}L_{\beta}\kappa_g^2}{\rho^2}, \quad 0<\alpha\leq\frac{\rho\sqrt{1-\nu}}{\kappa_g\sqrt{108C_{\mu_g}L_\beta}},
\end{align*}
substituting the above inequality into (\ref{OBBO_1}), it follows that
\begin{align}
    \sum_{t=2}^T\left\|\frac{\partial f_t(\mathbf{x}_t, \mathbf{y}_t^K)}{\partial \mathbf{x}} - \nabla F_t(\mathbf{x}_t)\right\|^2 \leq& \frac{27}{8}\left(\frac{\Delta_\beta L_\beta}{1-\nu} + L_3^2\right) + \frac{9B}{8(1-\nu)}\sum_{t=1}^T\left\|\mathcal{G}_\mathcal{X}(\mathbf{x}_{t}, \nabla F_{t}(\mathbf{x}_{t}), \alpha)\right\|^2 \nonumber\\
    & + \frac{9CH_{2,T}}{8(1-\nu)}. \label{OBBO_2}
\end{align}

Now remind of Lemma~\ref{lem:begin}, with $\widetilde{\nabla}f_t(\mathbf{x}_t, \mathbf{y}_{t+1}, \mathbf{v}_{t+1}) = \frac{\partial f_t(\mathbf{x}_t, \mathbf{y}_t^K)}{\partial\mathbf{x}}$, we naturally have
\begin{align}
    \sum_{t=1}^{T} \left\| \mathcal{G}_\mathcal{X}\left( \mathbf{x}_t, \nabla f_t(\mathbf{x}_t, \mathbf{y}_t^*(\mathbf{x}_t)), \gamma \right) \right\|^2 
    \leq& \frac{2}{\gamma\theta} \sum_{t=1}^{T} \left( f_t(\mathbf{x}_t, \mathbf{y}_t^*(\mathbf{x}_t)) - f_t(\mathbf{x}_{t+1}, \mathbf{y}_t^*(\mathbf{x}_{t+1})) \right) \nonumber\\
    & + \left(2+\frac{1}{\lambda\theta}\right) \sum_{t=1}^{T} \left\| \nabla f_t(\mathbf{x}_t, \mathbf{y}_t^*(\mathbf{x}_t)) - \frac{\partial f_t(\mathbf{x}_t, \mathbf{y}_t^K)}{\partial\mathbf{x}} \right\|^2. \label{OBBO_3}
\end{align}
we still set $\alpha \leq \frac{1}{2L_F}$ and $\lambda=1$, $2 + \frac{1}{\lambda\theta} \leq 6$, the same as in Theorem~\ref{thm:CIGO_upper}, also with Lemma~\ref{lem:2M+V_T}, it then holds that
\begin{align*}
    \sum_{t=1}^{T} \left\| \mathcal{G}_\mathcal{X}\left( \mathbf{x}_t, \nabla F_t(\mathbf{x}_t), \gamma \right) \right\|^2 \leq& \frac{4Q + 2V_T}{\gamma\theta} + \frac{51}{4}\left(\frac{\Delta_\beta L_\beta}{1-\nu} + L_3^2\right) \\
    & + \frac{27B}{4(1-\nu)}\sum_{t=1}^T\left\|\mathcal{G}_\mathcal{X}(\mathbf{x}_{t}, \nabla F_{t}(\mathbf{x}_{t}), \alpha)\right\|^2 \\
    & + \frac{27CH_{2,T}}{4(1-\nu)} + 6\left\|\nabla F_t(\mathbf{x}_t) - \frac{\partial f_1(\mathbf{x}_1, \mathbf{y}_1^K)}{\partial\mathbf{x}}\right\|^2,
\end{align*}
with $B = 12\alpha^2C_{\mu_g}L_\beta\kappa_g^2$ in Lemma B.3 \cite{bohne2024online}, 
\begin{align}
    &\left(1 - \frac{81\alpha^2C_{\mu_g}L_\beta\kappa_g^2}{1-\nu}\right)\sum_{t=1}^{T} \left\| \mathcal{G}_\mathcal{X}\left( \mathbf{x}_t, \nabla F_t(\mathbf{x}_t), \gamma \right) \right\|^2 \\
    \leq& \frac{4Q + 2V_T}{\gamma\theta} + \frac{51}{4}\left(\frac{\Delta_\beta L_\beta}{1-\nu} + L_3^2\right) + \frac{27CH_{2,T}}{4(1-\nu)} + 6\left\|\nabla F_t(\mathbf{x}_t) - \frac{\partial f_1(\mathbf{x}_1, \mathbf{y}_1^K)}{\partial\mathbf{x}}\right\|^2. \label{OBBO_4}
\end{align}
We set
\begin{align*}
    \alpha\leq \frac{\sqrt{1-\nu}}{9\kappa_g\sqrt{C_{\mu_g}L_\beta}} \quad \text{and} \quad c := 1 - \frac{81\alpha^2C_{\mu_g}L_\beta\kappa_g^2}{1-\nu},
\end{align*}
and substitute into (\ref{OBBO_4}) to get
\begin{align*}
    \sum_{t=1}^{T} \left\| \mathcal{G}_\mathcal{X}\left( \mathbf{x}_t, \nabla F_t(\mathbf{x}_t), \gamma \right) \right\|^2 \leq& \frac{4Q + 2V_T}{c\gamma\theta} + \frac{51}{4c}\left(\frac{\Delta_\beta L_\beta}{1-\nu} + L_3^2\right) + \frac{27CH_{2,T}}{4c(1-\nu)} \\
    & + \frac{6}{c}\left\|\nabla F_t(\mathbf{x}_t) - \frac{\partial f_1(\mathbf{x}_1, \mathbf{y}_1^K)}{\partial\mathbf{x}}\right\|^2 \\
    =& O(1 + V_T + H_{2,T}).
\end{align*}
Thus we finish the proof.
\end{proof}

\subsection{Proof of SOGD in Deterministic Online Bilevel Optimization}\label{sec:proof_SOGD}

In the work of \citet{nazari2025stochastic}, the objective functions $f_t(\cdot)$, $g_t(\cdot)$ in problem (\ref{eq:obo}) they studied had the following random noise:
\begin{align*}
    f_t(\mathbf{x}, \mathbf{y}_t^*(\mathbf{x})) = \mathbb{E}_{\xi_t\sim\mathcal{D}_f}\left[f_t(\mathbf{x}, \mathbf{y}_t^*(\mathbf{x});\xi_t)\right], \quad g_t(\mathbf{x}, \mathbf{y}) = \mathbb{E}_{\zeta_t\sim\mathcal{D}_g}\left[g_t(\mathbf{x}, \mathbf{y};\zeta_t)\right],
\end{align*}
where $\mathcal{D}_f, \mathcal{D}_g$ are data distributions. The unbiased stochastic gradient is revealed, and the variance follows the bounded assumption as follows:
\begin{align*}
    &\mathbb{E} \left[ \| \nabla_\mathbf{x} f_t(\mathbf{x}, \mathbf{y}; \xi_t) - \nabla_\mathbf{x} f_t(\mathbf{x}, \mathbf{y}) \|^2 \right] \leq \sigma_{f\mathbf{x}}^2, \quad \mathbb{E} \left[ \| \nabla_\mathbf{y} g_t(\mathbf{x}, \mathbf{y}; \zeta) - \nabla_\mathbf{y} g_t(\mathbf{x}, \mathbf{y}) \|^2 \right] \leq \sigma_{g\mathbf{y}}^2, \\
    &\mathbb{E} \left[ \| \nabla^2_{\mathbf{y}\mathbf{y}} g_t(\mathbf{x}, \mathbf{y}; \zeta) - \nabla^2_{\mathbf{y}\mathbf{y}} g_t(\mathbf{x}, \mathbf{y}) \|^2 \right] \leq \sigma_{g\mathbf{y}\mathbf{y}}^2, \quad \mathbb{E} \left[ \| \nabla^2_{\mathbf{x}\mathbf{y}} g_t(\mathbf{x}, \mathbf{y}; \zeta) - \nabla^2_{\mathbf{x}\mathbf{y}} g_t(\mathbf{x}, \mathbf{y}) \|^2 \right] \leq \sigma_{g\mathbf{x}\mathbf{y}}^2, \\
    &\mathbb{E} \left[ \| \nabla_\mathbf{y} f_t(\mathbf{x}, \mathbf{y}; \xi_t) - \nabla_\mathbf{y} f_t(\mathbf{x}, \mathbf{y}) \|^2 \right] \leq \sigma_{f\mathbf{y}}^2,
\end{align*}
they also define $\sigma^2 := \sigma_{g\mathbf{y}}^2 + \sigma_{f\mathbf{y}}^2 + \sigma_{g\mathbf{y}\mathbf{y}}^2 + \sigma_{g\mathbf{x}\mathbf{y}}^2 + \sigma_{f\mathbf{x}}^2$.

We now briefly demonstrate the improved bound of SOGD in a deterministic environment, i.e., $\sigma = 0$.
\begin{theorem}\label{thm:SOGD}
Under Assumptions~\ref{asm1}-\ref{asm3}, let $\alpha_t \equiv\alpha$, $\beta_t\equiv\beta=c_\beta\alpha$, $\delta_t\equiv\delta=c_\delta\alpha$, $\gamma_t\equiv\gamma=c_\gamma\alpha^2$, $\eta_t\equiv\eta=c_\eta\alpha^2$ and $\lambda_t\equiv\lambda=c_\lambda\alpha^2$ with some constants $c$, $c_\beta$, $c_\delta$, $c_\gamma$, $c_\eta$ and $c_\lambda$, SOGD \cite{nazari2025stochastic} can guarantees the following bound on the bilevel local regret defined in (\ref{eq:reg}) for problem (\ref{eq:obo}):
\begin{align*}
    \mathrm{Reg}(T) \leq O\left(1 + V_T + H_{2,T} + E_{2,T} + P_T\right),
\end{align*}
where $P_T$ is defined in (\ref{eq:P_T}).
\end{theorem}
\begin{proof}
We begin with equation (122) in their proof of Section C.6, shown as:
\begin{align*}
    &\sum_{t=1}^{T} \frac{\alpha_t}{2} \mathbb{E} \left[ \left\| \mathcal{G}_{\mathcal{X}} \left( \mathbf{x}_t, \nabla f_t (\mathbf{x}_t, \mathbf{y}^*_t(\mathbf{x}_t)), \alpha_t \right) \right\|^2 \right] + \Lambda \\
    \leq& O \left( V_T {+} \frac{H_{2,T}}{\alpha_T} {+} \frac{\sigma^2}{b} \sum_{t=1}^{T} \alpha_t^3 {+} \frac{G_{\mathbf{y},T} {+} G_{\mathbf{y}\mathbf{y},T} {+} G_{\mathbf{x}\mathbf{y},T} {+} D_{\mathbf{x},T} {+} D_{\mathbf{y}, T}}{\alpha_T}\right), 
\end{align*}
where 
\begin{align}
    &D_{\mathbf{x}, T} := \sum_{t=2}^T\sup_{\mathbf{x},\mathbf{y}}\left\|\nabla_\mathbf{x}f_{t-1}(\mathbf{x}, \mathbf{y}) - \nabla_\mathbf{x}f_t(\mathbf{x}, \mathbf{y})\right\|^2, \nonumber\\
    &D_{\mathbf{y},T} := \sum_{t=2}^T\sup_{\mathbf{x},\mathbf{y}}\left\|\nabla_\mathbf{y}f_{t-1}(\mathbf{x}, \mathbf{y}) - \nabla_\mathbf{y}f_t(\mathbf{x}, \mathbf{y})\right\|^2, \nonumber\\
    &G_{\mathbf{y}\mathbf{y},T} := \sum_{t=2}^T \left\|\nabla_{\mathbf{y}\mathbf{y}}^2g_{t-1}(\mathbf{x}_t, \mathbf{y}_t) - \nabla_{\mathbf{y}\mathbf{y}}^2g_t(\mathbf{x}_t, \mathbf{y}_t)\right\|^2, \nonumber\\
    &G_{\mathbf{x}\mathbf{y},T} := \sum_{t=2}^T \left\|\nabla_{\mathbf{x}\mathbf{y}}^2g_{t-1}(\mathbf{x}_t, \mathbf{y}_t) - \nabla_{\mathbf{x}\mathbf{y}}^2g_t(\mathbf{x}_t, \mathbf{y}_t)\right\|^2, \nonumber\\
    &G_{\mathbf{y},T} := \sum_{t=2}^T \left\|\nabla_{\mathbf{y}}g_{t-1}(\mathbf{x}_t, \mathbf{y}_t) - \nabla_{\mathbf{y}}g_t(\mathbf{x}_t, \mathbf{y}_t)\right\|^2, \nonumber\\
    &\Psi_T := G_{\mathbf{y},T} + G_{\mathbf{y}\mathbf{y},T} + G_{\mathbf{x}\mathbf{y},T} + D_{\mathbf{x},T} + D_{\mathbf{y}, T}. \label{eq:Psi_T}
\end{align}
With $\Lambda \geq - \frac{\sigma^2}{\alpha_0}$ and $\alpha_t = \frac{1}{\sqrt[3]{c+t}}$ for some constant $c$ is set, which means $\sum_{t=1}^T\alpha_t^3 \leq \log(T+1)$, the stochastic bilevel local regret bound is achieved as 
\begin{align}
    \sum_{t=1}^{T} \mathbb{E} \left[ \left\| \mathcal{G}_{\mathcal{X}} \left( \mathbf{x}_t, \nabla f_t (\mathbf{x}_t, \mathbf{y}^*_t(\mathbf{x}_t)), \alpha_t \right) \right\|^2 \right] \leq& O\left(\frac{1}{\alpha_T}(\sigma^2 + V_T) + \frac{1}{\alpha_T^2}\Psi_T\right) \label{A.15_1}\\
    =& O\left(T^{1/3}(\sigma^2 + V_T) + T^{2/3}\Psi_T\right). \nonumber
\end{align}
Now consider the deterministic situation, which means $\sigma = 0$ and (\ref{A.15_1}) becomes
\begin{align*}
    &\sum_{t=1}^{T} \left\| \mathcal{G}_{\mathcal{X}} \left( \mathbf{x}_t, \nabla f_t (\mathbf{x}_t, \mathbf{y}^*_t(\mathbf{x}_t)), \alpha_t \right) \right\|^2 \\
    \leq& O \left(\frac{1}{\alpha_T}(1 + V_T) + \frac{1}{\alpha_T^2}(H_{2,T} + G_{\mathbf{y},T} + G_{\mathbf{y}\mathbf{y},T} + G_{\mathbf{x}\mathbf{y},T} + D_{\mathbf{x},T} + D_{\mathbf{y}, T})\right) \\
    \leq& O\left(\frac{1}{\alpha_T}(1 + V_T) + \frac{1}{\alpha_T^2}(H_{2,T} + E_{2,T} + P_T\right),
\end{align*}
where $E_{2,T}$ is defined in Theorem~\ref{thm:CTHO_upper} and
\begin{align}
    P_T := D_{\mathbf{x},T} + G_{\mathbf{y},T} + G_{\mathbf{x}\mathbf{y},T}. \label{eq:P_T}
\end{align}
Note that we can now set $\alpha_t \equiv \alpha = \frac{1}{c}$, since term $\sum_{t=1}^T\alpha_t^3$ is already zero, and the step size no longer needs to decay with $t$ to ensure this term is sublinear. We thus obtain the following upper-bound guarantee for SOGD in the noise-free setting:
\begin{align*}
    \mathrm{Reg}(T) \leq O\left(1 + V_T + H_{2,T} + E_{2,T} + P_T\right).
\end{align*}
\end{proof}

\section{Proofs of Section~\ref{sec:4} under Window-Averaged Bilevel Local Regret}\label{sec:proof_B}

In this section, we provide proofs of Theorems~\ref{thm:unres_win_dl},~\ref{thm:unres_win_lower} and~\ref{thm:unres_win_sl} under the window averaged bilevel local regret presented in Section~\ref{sec:4}. 
In Section~\ref{sec:proof_B_lem}, we present the key Lemmas~\ref{lem:y_bound_win_3} and~\ref{lem:v_bound_win_3} that characterize how the time variations of $\vy$ and $\vv$ are averaged by the window.

 \subsection{Preliminary Lemmas}\label{sec:proof_B_lem}

Now we consider the window-averaged hypergradient.
Note that $\widehat{f}_{t,w}(\cdot)$ and $\widehat{g}_{t,w}(\cdot)$ are linear combinations of $f_t(\cdot)$ and $g_t(\cdot)$ for some $t\in[T]$, and therefore also satisfy assumptions~\ref{asm1}-\ref{asm3}. We similarily have
\begin{align*}
    \nabla\mathbf{y}_{t,w}^*(\mathbf{x})\nabla_{\mathbf{y}\mathbf{y}}^2\widehat{g}_{t,w}&(\mathbf{x}, \mathbf{y}_{t,w}^*(\mathbf{x})) + \nabla_{\mathbf{x}\mathbf{y}}^2\widehat{g}_{t,w}(\mathbf{x}, \mathbf{y}_{t,w}^*(\mathbf{x}))=0,
\end{align*}
which follows from $\nabla_\mathbf{y}\widehat{g}_{t,w}(\mathbf{x}, \mathbf{y}_{t,w}^*(\mathbf{x})) = 0$ and reveals the closed form of window-averaged implicit gradient $\nabla\mathbf{y}_{t,w}^*(\mathbf{x})$. Similar to (\ref{eq:hypergrad}), $\nabla \widehat{f}_{t,w}(\mathbf{x}, \mathbf{y}_{t,w}^*(\mathbf{x}))$ is then obtained by
\begin{align*}
    \nabla \widehat{f}_{t,w}(\mathbf{x}, \mathbf{y}_{t,w}^*(\mathbf{x})) = \nabla_\mathbf{x}\widehat{f}_{t,w}(\mathbf{x}, \mathbf{y}_{t,w}^*(\mathbf{x})) - \nabla_{\mathbf{x}\mathbf{y}}^2\widehat{g}_{t,w}(\mathbf{x}, \mathbf{y}_{t,w}^*(\mathbf{x}))\mathbf{v}_{t,w}^*(\mathbf{x}, \mathbf{y}_{t,w}^*(\vx)).
\end{align*}
Here $\mathbf{v}_{t,w}^*(\mathbf{x}):= \mathop{\arg\min}_{\mathbf{v}\in\mathcal{V}}\widehat{\Phi}_{t,w}(\mathbf{x}, \mathbf{y}_{t,w}^*(\mathbf{x}), \mathbf{v})$, where
\begin{align*}
    \widehat{\Phi}_{t,w}(\mathbf{x}, \mathbf{y}, \mathbf{v}) := \frac{1}{2}\mathbf{v}^\top\nabla_{\mathbf{y}\mathbf{y}}^2\widehat{g}_{t,w}(\mathbf{x}, \mathbf{y})\mathbf{v} - \mathbf{v}^\top\nabla_\mathbf{y}\widehat{f}_t(\mathbf{x},\mathbf{y}). 
\end{align*}
Note that we also have
\begin{align*}
    \widehat{\Phi}_{t,w}(\mathbf{x}, \mathbf{y}, \mathbf{v}) =& \frac{1}{2}\mathbf{v}^{\top}\nabla_{\mathbf{y}\mathbf{y}}^2\widehat{g}_{t,w}(\mathbf{x}, \mathbf{y})\mathbf{v} - \mathbf{v}^{\top}\nabla_\mathbf{y}\widehat{f}_{t,w}(\mathbf{x}, \mathbf{y}) \\
    =& \frac{1}{2W}\sum_{i=0}^{w-1}\eta^i \mathbf{v}^\top \nabla_{\mathbf{y}\mathbf{y}}^2g_{t-i}(\mathbf{x}, \mathbf{y}) - \frac{1}{W}\sum_{i=0}^{w-1}\eta^i \mathbf{v}^\top\nabla_\mathbf{y} f_{t-i}(\mathbf{x}, \mathbf{y}),
\end{align*}
and
\begin{align*}
    \nabla_\mathbf{v}\widehat{\Phi}_{t,w}(\mathbf{x}, \mathbf{y}, \mathbf{v}) =& \frac{1}{W} \sum_{i=0}^{w-1}\eta^i\nabla_{\mathbf{y}\mathbf{y}}^2g_{t-i}(\mathbf{x}, \mathbf{y})\mathbf{v} - \frac{1}{W}\sum_{i=0}^{w-1}\eta^i\nabla_\mathbf{y} f_{t-i}(\mathbf{x}, \mathbf{y}) \\
    =& \frac{1}{W}\sum_{i=0}^{w-1}\eta^i\nabla_\mathbf{v}\Phi_{t-i}(\mathbf{x}, \mathbf{y}, \mathbf{v}).
\end{align*}

Due to the $\mu_g$-strongly convex and $L_{g,1}$-smooth of $\widehat{g}_{t,w}(\mathbf{x}, \cdot)$ for any $\mathbf{x}\in\mathcal{X}$, we have
\begin{align}
    & \left\|\nabla_\mathbf{y}\widehat{g}_{t,w}(\mathbf{x}, \mathbf{y}_1) - \nabla_\mathbf{y}\widehat{g}_{t,w}(\mathbf{x}, \mathbf{y}_2)\right\| \leq L_{g,1}\left\|\mathbf{y}_1 - \mathbf{y}_2\right\| \label{eq:sc_4} \\
    & \left\|\nabla_\mathbf{y}\widehat{g}_{t,w}(\mathbf{x}, \mathbf{y}_1) - \nabla_\mathbf{y}\widehat{g}_{t,w}(\mathbf{x}, \mathbf{y}_2)\right\| \geq \mu_g\left\|\mathbf{y}_1 - \mathbf{y}_2\right\| \label{eq:sc_5}.
\end{align}

\begin{lemma}\label{lem:y_bound_win_1}
Under Assumptions~\ref{asm1},~\ref{asm2} and~\ref{asm4}, for all $t\in[T]$ and any $\mathbf{x}\in\mathcal{X}$ we have
\begin{align*}
    \left\|\mathbf{y}_{t-1,w}^*(\mathbf{x}) - \mathbf{y}_{t,w}^*(\mathbf{x})\right\| \leq \frac{(1+\eta^w)L_{g,1}D}{\mu_gW}.
\end{align*}
\end{lemma}
\begin{proof}
It begins (\ref{eq:sc_5}), with $\mathbf{y}_2 = \mathbf{y}_{t,w}^*(\mathbf{x})$, we have
\begin{align*}
    \mu_g\left\|\mathbf{y} - \mathbf{y}_{t,w}^*(\mathbf{x})\right\| \leq \left\|\nabla_\mathbf{y}\widehat{g}_{t,w}(\mathbf{x}, \mathbf{y}) - \nabla_\mathbf{y}\widehat{g}_{t,w}(\mathbf{x}, \mathbf{y}_{t,w}^*(\mathbf{x}))\right\| = \left\|\nabla_\mathbf{y}\widehat{g}_{t,w}(\mathbf{x}, \mathbf{y})\right\|,
\end{align*}
let $\mathbf{y} = \mathbf{y}_{t-1,w}^*(\mathbf{x})$, it implies that
\begin{align*}
    \left\|\mathbf{y}_{t-1,w}^*(\mathbf{x}) - \mathbf{y}_{t,w}^*(\mathbf{x})\right\| \leq& \frac{1}{\mu_g}\left\|\nabla_\mathbf{y}\widehat{g}_{t,w}(\mathbf{x}, \mathbf{y}_{t-1,w}^*(\mathbf{x}))\right\|
    \\
    =& \frac{1}{\mu_g}\left\|\nabla_\mathbf{y}\widehat{g}_{t,w}(\mathbf{x}, \mathbf{y}_{t-1,w}^*(\mathbf{x})) - \eta\nabla_\mathbf{y}\widehat{g}_{t-1,w}(\mathbf{x}, \mathbf{y}_{t-1,w}^*(\mathbf{x}))\right\|,
\end{align*}
then it natually holds that
\begin{align*}
    &\left\|\nabla_\mathbf{y} \widehat{g}_{t,w}(\mathbf{x}, \mathbf{y}_{t-1,w}^*(\mathbf{x})) - \eta\nabla_\mathbf{y}\widehat{g}_{t-1,w}^*(\mathbf{x}, \mathbf{y}_{t-1,w}^*(\mathbf{x}))\right\| \\
    \leq& \frac{1}{W}\left\|\nabla_\mathbf{y} g_t(\mathbf{x}, \mathbf{y}_{t-1,w}^*(\mathbf{x})) - \eta^w\nabla_\mathbf{y} g_{t-w}(\mathbf{x}, \mathbf{y}_{t-1,w}^*(\mathbf{x}))\right\| \\
    & + \frac{1}{W}\left\|\sum_{i=1}^{w-1}\eta^i\left(\nabla_\mathbf{y} g_{t-i}(\mathbf{x}, \mathbf{y}_{t-1,w}^*(\mathbf{x})) - \nabla_\mathbf{y} g_{t-i}(\mathbf{x}, \mathbf{y}_{t-1,w}^*(\mathbf{x}))\right)\right\| \\
    =& \frac{1}{W}\left\|\nabla_\mathbf{y} g_t(\mathbf{x}, \mathbf{y}_{t-1,w}^*(\mathbf{x})) - \eta^w\nabla_\mathbf{y} g_{t-w}(\mathbf{x}, \mathbf{y}_{t-1,w}^*(\mathbf{x}))\right\|
    \leq \frac{(1+\eta^w)L_{g,1}D}{W},
\end{align*}
where in the last inequality, notice that for any $\mathbf{y}\in\mathcal{Y}$ defined in Assumption~\ref{asm4}, with $\mathbf{y}_2 = \mathbf{y}_t^*(\mathbf{x})$ in (\ref{eq:sc_1}),
\begin{align*}
    \left\|\nabla_\mathbf{y} g_t(\mathbf{x}, \mathbf{y})\right\| = \left\|\nabla_\mathbf{y} g_t(\mathbf{x}, \mathbf{y}) - \nabla_\mathbf{y} g_t(\mathbf{x}, \mathbf{y}_t^*(\mathbf{x}))\right\| \leq L_{g,1}\left\|\mathbf{y} - \mathbf{y}_t^*(\mathbf{x})\right\| \leq L_{g,1}D.
\end{align*}
Thus we finish the proof.
\end{proof}

\begin{lemma}\label{lem:v_bound_win_1}
Under Assumptions~\ref{asm1},~\ref{asm2} and~\ref{asm4}, for all $t\in[T]$ and any $\mathbf{x}\in\mathcal{X}$ we have
\begin{align*}
    \left\|\mathbf{v}_{t-1,w}^*(\mathbf{x}, \mathbf{y}) - \mathbf{v}_{t,w}^*(\mathbf{x}, \mathbf{y})\right\| \leq \frac{(1+\eta^w)L_{f,0}L_{g,1}}{\mu_g^2W}.
\end{align*}
\end{lemma}
\begin{proof}
The proof is the same as Lemma~\ref{lem:y_bound_win_1}, and note that $\mathbf{v}_{t,w}^*(\mathbf{x}, \mathbf{y})\in\mathcal{V}$ for any $\mathbf{x}\in\mathcal{X}$, $\mathbf{y}\in\mathbb{R}_{d_2}$, thus
\begin{align*}
    &\left\|\mathbf{v}_{t-1,w}^*(\mathbf{x}, \mathbf{y}) - \mathbf{v}_{t,w}^*(\mathbf{x}, \mathbf{y})\right\| \\
    \leq& \frac{1}{\mu_gW}\left\|\nabla_\mathbf{v}\Phi_t(\mathbf{x}, \mathbf{y}, \mathbf{v}_{t-1,w}^*(\mathbf{x}, \mathbf{y})) - \eta^w\nabla_\mathbf{v}\Phi_{t-w}(\mathbf{x}, \mathbf{y}, \mathbf{v}_{t-1,w}^*(\mathbf{x}, \mathbf{y}))\right\| \\
    \leq& \frac{(1+\eta^w)L_{g,1}L_{f,0}}{\mu_g^2W}.
\end{align*}
\end{proof}

\begin{lemma}\label{lem:F-F_win}
    Under Assumptions~\ref{asm1}-\ref{asm3} and~\ref{asm4}, we can obtain
    \begin{align*}
        \sum_{t=1}^T \left(\widehat{f}_{t,w}(\mathbf{x}_t, \mathbf{y}_{t,w}^*(\mathbf{x}_t)) - \widehat{f}_{t,w}(\mathbf{x}_{t+1}, \mathbf{y}_{t,w}^*(\mathbf{x}_{t+1}))\right) \leq 2Q + \frac{(1+\eta^w)L_{f,0}L_{g,1}DT}{\mu_gW} + \frac{2QT}{W}.
    \end{align*}
\end{lemma}
\begin{proof}
It follows that
\begin{align}
    &\sum_{t=1}^T \left(\widehat{f}_{t,w}(\mathbf{x}_t, \mathbf{y}_{t,w}^*(\mathbf{x}_t)) - \widehat{f}_{t,w}(\mathbf{x}_{t+1}, \mathbf{y}_{t,w}^*(\mathbf{x}_{t+1}))\right) \nonumber\\
    =& f_1(\vx_1, y_1^*(\vx_1)) - \widehat{f}_{T,w}(\vx_{T+1}, \vy_{T,w}^*(\vx_{T+1})) \nonumber\\
    & + \sum_{t=2}^T\left(\widehat{f}_{t,w}(\vx_t, \vy_{t,w}^*(\vx_t)) - \widehat{f}_{t-1,w}(\vx_t, \vy_{t-1,w}^*(\vx_t))\right) \nonumber\\
    \leq& 2Q + \sum_{t=2}^T\left(\widehat{f}_{t,w}(\vx_t, \vy_{t,w}^*(\vx_t)) - \widehat{f}_{t,w}(\vx_t, \vy_{t-1,w}^*(\vx_t)) \right. \nonumber\\
    & \left. + \widehat{f}_{t,w}(\vx_t, \vy_{t-1,w}^*(\vx_t)) - \widehat{f}_{t-1,w}(\vx_t, \vy_{t-1,w}^*(\vx_t))\right) \nonumber\\
    \overset{(i)}{\leq}& 2Q + L_{f,0}\sum_{t=2}^T \left\|\vy_{t,w}^*(\vx_t) - \vy_{t-1,w}^*(\vx_t)\right\| + \frac{2QT}{W} \nonumber\\
    \overset{(ii)}{\leq}& 2Q + \frac{(1+\eta^w)L_{f,0}L_{g,1}DT}{\mu_gW} + \frac{2QT}{W}, \label{C.4_13}
\end{align}
where $(i)$ follows from the fact that, for any $\mathbf{z}=(\mathbf{x},\mathbf{y})\in\mathcal{X}\times\mathbb{R}^{d_2}$,
\begin{align*}
    \widehat{f}_{t,w}(\vz) - \widehat{f}_{t-1,w}(\vz) 
    =& \frac{1}{W}\left( f_t(\vz) + \eta f_{t-1}(\vz) + \eta^2 f_{t-2}(\vz) + \cdots \eta^{w-1}f_{t-w+1}(\vz) \right. \\
    & \left. - f_{t-1}(\vz) - \eta f_{t-2}(\vz) - \cdots - \eta^{w-2} f_{t-w+1}(\vz) - \eta^{w-1}f_{t-w}(\vz) \right)\\
    \leq& \frac{1}{W}\left( |f_t(\vz)| + \eta^{w-1}|f_{t-w}(\vz)| + (1-\eta)\sum_{i=1}^{w-1}\eta^{i-1}|f_{t-i}(\vz)| \right) \\
    \leq& \frac{(1 + \eta^{w-1} + 1 - \eta^{w-1})Q}{W} = \frac{2Q}{W},
\end{align*}
and $(ii)$ comes from Lemma~\ref{lem:y_bound_win_1}. 

\end{proof}

\subsection{Proof of Theorem~\ref{thm:unres_win_sl}}

In this section, we proof the window-averaged regret upper bound of Algorithm~\ref{alg:WOBO} under single-loop structure. For clearity, we repeatly describle our algorithm in Algorithm~\ref{alg:WOBO-SL}.

\begin{algorithm}[tb]
    \caption{Window-averaged Online Bilevel Optimizer under Single-Loop structure (WOBO-SL)}
    \label{alg:WOBO-SL}
    \textbf{Input}: Initializations: $\mathbf{x}_1$, $\mathbf{y}_1$, $\mathbf{v}_1$, $\alpha$, $\beta$, $\gamma$, $w$, $\eta$. \\
    \textbf{Output}: Decision sequences: $\{\mathbf{x}_t\}_{t=1}^T$, $\{\mathbf{y}_t\}_{t=1}^T$. 
    
    \begin{algorithmic}[1] 
        \FOR{\(t = 1\) to \(T\)}
        \STATE Output $\mathbf{x}_t$ and $\mathbf{y}_t$ and receive feedback $f_t$ and $g_t$
        \STATE \(
        \mathbf{y}_{t+1} \leftarrow \mathbf{y}_t - \alpha\nabla_\mathbf{y}\widehat{g}_{t,w}(\mathbf{x}_t, \mathbf{y}_t)
        \)
        \STATE \(
        \mathbf{v}_{t+1} \leftarrow \mathbf{v}_t - \beta\nabla_\mathbf{v}\widehat{\Phi}_{t,w}(\mathbf{x}_t, \mathbf{y}_t, \mathbf{v}_t)
        \)
        \STATE Calculate \(
        \widetilde{\nabla}\widehat{f}_{t,w}(\mathbf{x}_t, \mathbf{y}_t, \mathbf{v}_t) = \nabla_\vx \widehat{f}_{t,w}(\vx_t, \vy_t) - \nabla_{\vx\vy}^2 \widehat{g}_{t,w}(\vx_t, \vy_t)\vv_t
        \)
        \STATE \(
        \mathbf{x}_{t+1} \leftarrow \mathop{\rm argmin}_{\mathbf{u}\in\mathcal{X}}\left\{\left\langle\widetilde{\nabla}\widehat{f}_{t,w}(\mathbf{x}_t, \mathbf{y}_t, \mathbf{v}_t), \mathbf{u}\right\rangle + \frac{1}{2\gamma}\left\|\mathbf{u} - \mathbf{x}_t\right\|^2\right\}
        \)
        \ENDFOR
    \end{algorithmic}
\end{algorithm}

For notation simplicity, here we define $\nabla F_{t,w}(\vx) := \nabla \widehat{f}_{t,w}(\vx, \vy_{t,w}^*(\vx))$, and $\widetilde\nabla F_{t,w}(\vx_t) := \widetilde\nabla \widehat{f}_{t,w}(\vx_t, \vy_t, \vv_t)$ in WOBO-SL as we represented in Algorithm~\ref{alg:WOBO-SL}.

\begin{lemma}\label{lem:begin_win}
    Under Assumptions~\ref{asm1}, \ref{asm2}, for all $t\in[T]$, we can obtain
    \begin{align*}
        \sum_{t=1}^T\left\|\mathcal{G}_\mathcal{X}\left(\vx_t, \nabla F_{t,w}(\vx_t), \gamma\right)\right\|^2 \leq& \frac{8}{\gamma} \sum_{t=1}^T\left(F_{t,w}(\vx_t) - F_{t,w}(\vx_{t+1})\right) \\
        & + 6\sum_{t=1}^T\left\|\widetilde{\nabla} F_{t,w}(\vx_t) - \nabla F_{t,w}(\vx_t)\right\|^2.
\end{align*}
\end{lemma}
\begin{proof}
Based on Lemma~\ref{lem:L_F}, it still holds that $\widehat{f}_{t,w}(\vx, \vy_{t,w}^*(\vx))$ is $L_F$-smooth w.r.t. $\vx$, thus we begins with
\begin{align}
    &F_{t,w}(\vx_{t+1}) - F_{t,w}(\vx_t) \nonumber\\
    \leq& \left\langle \nabla F_{t,w}(\vx_t), \vx_{t+1} - \vx_t \right\rangle + \frac{L_F}{2}\left\|\vx_{t+1} - \vx_t\right\|^2 \nonumber\\
    \leq& -\gamma\left\langle \nabla F_{t,w}(\vx_t), \mathcal{G}_\mathcal{X}\left(\vx_t, \widetilde{\nabla} F_{t,w}(\vx_t), \gamma\right)  \right\rangle + \frac{\gamma^2L_F}{2} \left\|\mathcal{G}_\mathcal{X}\left(\vx_t, \widetilde{\nabla} F_{t,w}(\vx_t), \gamma\right)\right\|^2 \nonumber\\
    =& -\gamma\left\langle \widetilde{\nabla} F_{t,w}(\vx_t), \mathcal{G}_\mathcal{X}\left(\vx_t, \widetilde{\nabla} F_{t,w}(\vx_t), \gamma\right) \right\rangle + \frac{\gamma^2L_F}{2}\left\|\mathcal{G}_\mathcal{X}\left(\vx_t, \widetilde{\nabla} F_{t,w}(\vx_t), \gamma\right)\right\|^2 \nonumber\\
    & + \gamma\left\langle \widetilde{\nabla} F_{t,w}(\vx_t) - \nabla F_{t,w}(\vx_t), \mathcal{G}_\mathcal{X}\left(\vx_t, \widetilde{\nabla} F_{t,w}(\vx_t), \gamma\right) \right\rangle \nonumber\\
    \leq& - \frac{\gamma}{2} \left(1 - \gamma L_F\right) \left\|\mathcal{G}_\mathcal{X}\left(\vx_t, \widetilde{\nabla} F_{t,w}(\vx_t), \gamma\right)\right\|^2 + \frac{\gamma}{2} \left\|\widetilde{\nabla} F_{t,w}(\vx_t) - \nabla F_{t,w}(\vx_t)\right\|^2. \nonumber
\end{align}
Let $\gamma \leq \frac{1}{2L_F}$ such that $1 - \gamma L_F \geq \frac{1}{2}$, it implies that
\begin{align}
    \left\|\mathcal{G}_\mathcal{X}\left(\vx_t, \widetilde{\nabla} F_{t,w}(\vx_t), \gamma\right)\right\|^2 \leq \frac{4}{\gamma}\left(F_{t,w}(\vx_t) - F_{t,w}(\vx_{t+1})\right) + 2\left\|\widetilde{\nabla} F_{t,w}(\vx_t) - \nabla F_{t,w}(\vx_t)\right\|^2. \label{D.1_1}
\end{align}
Substituting
\begin{align*}
    \frac{1}{2}\left\|\mathcal{G}_\mathcal{X}\left(\vx_t, \nabla F_{t,w}(\vx_t), \gamma\right)\right\|^2 \leq \left\|\mathcal{G}_\mathcal{X}\left(\vx_t, \widetilde{\nabla} F_{t,w}(\vx_t), \gamma\right)\right\|^2 + \left\|\widetilde{\nabla} F_{t,w}(\vx_t) - \nabla F_{t,w}(\vx_t)\right\|^2
\end{align*}
into (\ref{D.1_1}), we finish the proof by
\begin{align*}
    \sum_{t=1}^T\left\|\mathcal{G}_\mathcal{X}\left(\vx_t, \nabla F_{t,w}(\vx_t), \gamma\right)\right\|^2 \leq& \frac{8}{\gamma} \sum_{t=1}^T\left(F_{t,w}(\vx_t) - F_{t,w}(\vx_{t+1})\right) \\
    & + 6\sum_{t=1}^T\left\|\widetilde{\nabla} F_{t,w}(\vx_t) - \nabla F_{t,w}(\vx_t)\right\|^2.
\end{align*}
\end{proof}

\begin{lemma}\label{lem:y_bound_win_3}
    Under Assumptions~\ref{asm1}, \ref{asm2} and \ref{asm4}, consider the inner iteration process of $\widehat{g}_{t,w}(\vx_t, \cdot)$ in Algorithm~\ref{alg:WOBO-SL} and $\alpha > 0$, we can obtain
    \begin{align*}
        \left\|\vy_t - \vy_{t,w}^*(\vx_t)\right\|^2 
        \leq& \rho_\vy\left\|\vy_{t-1} - \vy_{t-1,w}^*(\vx_{t-1})\right\|^2 + \frac{2\kappa_g^2(\mu_g {+} L_{g,1})}{\alpha\mu_gL_{g,1}}\left\|\vx_{t-1} - \vx_t\right\|^2 \\
        & - \frac{2(\mu_g + L_{g,1})}{\alpha\mu_gL_{g,1}}\left(\frac{2\alpha}{\mu_g + L_{g,1}} - \alpha^2\right)\left\|\nabla_\vy \widehat{g}_{t-1,w}(\vx_{t-1}, \vy_{t-1})\right\|^2 \\
        & + \frac{2(\mu_g + L_{g,1})(1+\eta^w)^2L_{g,1}D^2}{\alpha\mu_g^3W^2},
    \end{align*}
    where $\rho_\vy = 1 - \frac{\alpha\mu_gL_{g,1}}{\mu_g + L_{g,1}}$.
\end{lemma}
\begin{proof}
It begins with
\begin{align}
    \left\|\vy_t - \vy_{t,w}^*(\vx_t)\right\|^2 
    \leq& (1+\theta)\left\|\vy_t - \vy_{t-1,w}^*(\vx_{t-1})\right\|^2 + \left(1+\frac{1}{\theta}\right)\left\|\vy_{t-1,w}^*(\vx_{t-1}) - \vy_{t,w}^*(\vx_t)\right\|^2 \nonumber\\
    \leq& (1+\theta)\left\|\vy_t - \vy_{t-1,w}^*(\vx_{t-1})\right\|^2 + 2\kappa_g^2\left(1+\frac{1}{\theta}\right)\left\|\vx_{t-1} - \vx_t\right\|^2 \nonumber\\
    & + 2\left(1+\frac{1}{\theta}\right)\left\|\vy_{t-1,w}^*(\vx_t) - \vy_{t,w}^*(\vx_t)\right\|^2. \label{D.2_1}
\end{align}
For the first term, we have
\begin{align}
    &\left\|\vy_t - \vy_{t-1,w}^*(\vx_{t-1})\right\|^2 \nonumber\\
    =&\left\|\vy_{t-1} - \alpha\nabla_\vy \widehat{g}_{t-1,w}(\vx_{t-1}, \vy_{t-1}) - \vy_{t-1,w}^*(\vx_{t-1})\right\|^2 \nonumber\\
    =& \left\|\vy_{t-1} - \vy_{t-1,w}^*(\vx_{t-1})\right\|^2 + \alpha^2\left\|\nabla_\vy \widehat{g}_{t-1,w}(\vx_{t-1}, \vy_{t-1})\right\|^2 \nonumber\\
    & - 2\alpha\left\langle \nabla_\vy \widehat{g}_{t-1,w}(\vx_{t-1}, \vy_{t-1}), \vy_{t-1} - \vy_{t-1,w}^*(\vx_{t-1}) \right\rangle \nonumber\\
    \leq& \left(1 - \frac{2\alpha\mu_gL_{g,1}}{\mu_g + L_{g,1}}\right)\left\|\vy_{t-1} - \vy_{t-1,w}^*(\vx_{t-1})\right\|^2 - \left(\frac{2\alpha}{\mu_g + L_{g,1}} - \alpha^2\right)\left\|\nabla_\vy \widehat{g}_{t-1,w}(\vx_{t-1}, \vy_{t-1})\right\|^2, \label{D.2_2}
\end{align}
where in the last inequality, we follow the fact that, for any $\vx\in\mathcal{X}$ and $\vy\in\mathbb{R}^{d_2}$,
\begin{align}
    \left\langle \nabla_\mathbf{y} \widehat{g}_{t,w}(\mathbf{x}, \mathbf{y}), \mathbf{y} - \mathbf{y}_{t,w}^*(\mathbf{x}) \right\rangle \geq \frac{\mu_gL_{g,1}}{\mu_g + L_{g,1}}\left\|\mathbf{y} - \mathbf{y}_{t,w}^*(\mathbf{x})\right\|^2 + \frac{1}{\mu_g + L_{g,1}}\left\|\nabla_\mathbf{y} \widehat{g}_{t,w}(\mathbf{x}, \mathbf{y})\right\|^2. \nonumber
\end{align}
Substituting (\ref{D.2_2}) into (\ref{D.2_1}) and with $\theta = \frac{\alpha\mu_gL_{g,1}}{\mu_g + L_{g,1}}$, we finish the proof by
\begin{align*}
    \left\|\vy_t - \vy_{t,w}^*(\vx_t)\right\|^2 
    \leq& \left(1 - \frac{\alpha\mu_gL_{g,1}}{\mu_g {+} L_{g,1}}\right)\left\|\vy_{t-1} - \vy_{t-1,w}^*(\vx_{t-1})\right\|^2 + \frac{2\kappa_g^2(\mu_g {+} L_{g,1})}{\alpha\mu_gL_{g,1}}\left\|\vx_{t-1} - \vx_t\right\|^2 \\
    & - \frac{2(\mu_g + L_{g,1})}{\alpha\mu_gL_{g,1}}\left(\frac{2\alpha}{\mu_g + L_{g,1}} - \alpha^2\right)\left\|\nabla_\vy \widehat{g}_{t-1,w}(\vx_{t-1}, \vy_{t-1})\right\|^2 \\
    & + \frac{2(\mu_g + L_{g,1})(1+\eta^w)^2L_{g,1}D^2}{\alpha\mu_g^3W^2},
\end{align*}
where the last term comes from Lemma~\ref{lem:y_bound_win_3}.
\end{proof}

\begin{lemma}\label{lem:v_bound_win_3}
    Under Assumptions~\ref{asm1}, \ref{asm2} and \ref{asm4}, consider the inner iteration process of $\widehat{\Phi}_{t,w}(\vx_t, \vy_{t+1}, \cdot)$ in Algorithm~\ref{alg:WOBO-SL}, if set $\beta \leq \frac{1}{L_{g,1}}$, we can obtain
    \begin{align*}
        \left\|\vv_t - \vv_{t,w}^*(\vx_t, \vy_t)\right\|^2 \leq& \rho_\vv\left\|\vv_{t-1} - \vv_{t-1,w}^*(\vx_{t-1}, \vy_{t-1})\right\|^2 + c_\beta\left\|\vx_{t-1} - \vx_t\right\|^2 \\
        & + \alpha^2c_\beta\left\|\nabla_\vy \widehat{g}_{t-1,w}(\vx_{t-1}, \vy_{t-1})\right\|^2 +  + \frac{4(1+\eta^w)^2L_{f,0}^2L_{g,1}^2}{\beta\mu_g^5W^2},
    \end{align*}
    where $\rho_\vv = 1 - \frac{\beta\mu_g}{2}$, $c_\beta$ is some positive constant.
\end{lemma}
\begin{proof}
With $\beta \leq \frac{1}{L_{g,1}}$, it follows that
\begin{align*}
    &\left\|\vv_t - \vv_{t,w}^*(\vx_t, \vy_t)\right\|^2 \\
    \leq& (1+\theta)\left\|\vv_t - \vv_{t-1,w}^*(\vx_{t-1}, \vy_{t-1})\right\|^2 + \left(1+\frac{1}{\theta}\right)\left\|\vv_{t-1,w}^*(\vx_{t-1}, \vy_{t-1}) - \vv_{t,w}^*(\vx_t, \vy_t)\right\|^2 \\
    \overset{(i)}{\leq}& \left(1 - \frac{\beta\mu_g}{2}\right)\left\|\vv_{t-1} - \vv_{t-1,w}^*(\vx_{t-1}, \vy_{t-1})\right\|^2 + \frac{4}{\beta\mu_g}\left\|\vv_{t-1,w}^*(\vx_t, \vy_t) - \vv_{t,w}^*(\vx_t, \vy_t)\right\|^2\\
    & + \frac{4}{\beta\mu_g}\left(\frac{2L_{f,1}^2}{\mu_g^2} + \frac{2L_{f,0}^2L_{g,2}^2}{\mu_g^4}\right)\left(\left\|\vx_{t-1} - \vx_t\right\|^2 + \left\|\vy_{t-1} - \vy_t\right\|^2\right) \\
    \overset{(ii)}{\leq}& \rho_\vv\left\|\vv_{t-1} - \vv_{t-1,w}^*(\vx_{t-1}, \vy_{t-1})\right\|^2 + c_\beta\left\|\vx_{t-1} - \vx_t\right\|^2 + \alpha^2c_\beta\left\|\nabla_\vy \widehat{g}_{t-1,w}(\vx_{t-1}, \vy_{t-1})\right\|^2,
\end{align*}
where in $(i)$, we set $\theta = \frac{\beta\mu_g}{2}$, and for any $\vz, \vz'\in(\vx, \vy)\in\mathcal{X}\times\mathbb{R}^{d_2}$,
\begin{align*}
    \left\|\vv_{t,w}^*(\vz) - \vv_{t,w}^*(\vz')\right\|^2 
    =& \left\|\left[\nabla_{\vy\vy}^2 \widehat{g}_{t,w}(\vz)\right]^{-1} \nabla_\vy \widehat{f}_{t,w}(\vz) - \left[\nabla_{\vy\vy}^2 \widehat{g}_{t,w}(\vz')\right]^{-1} \nabla_\vy \widehat{f}_{t,w}(\vz')\right\|^2 \\
    \leq& 2\left\|\left[\nabla_{\vy\vy}^2 \widehat{g}_{t,w}(\vz)\right]^{-1} \nabla_\vy \widehat{f}_{t,w}(\vz) - \left[\nabla_{\vy\vy}^2 \widehat{g}_{t,w}(\vz)\right]^{-1} \nabla_\vy \widehat{f}_{t,w}(\vz')\right\|^2 \\
    & + 2\left\|\left[\nabla_{\vy\vy}^2 \widehat{g}_{t,w}(\vz)\right]^{-1} \nabla_\vy \widehat{f}_{t,w}(\vz') - \left[\nabla_{\vy\vy}^2 \widehat{g}_{t,w}(\vz')\right]^{-1} \nabla_\vy \widehat{f}_{t,w}(\vz')\right\|^2 \\
    \leq& \frac{2L_{f,1}^2}{\mu_g^2}\left\|\vz - \vz'\right\|^2 + \frac{2L_{f,0}^2L_{g,2}^2}{\mu_g^4}\left\|\vz - \vz'\right\|^2.
\end{align*}
In $(ii)$, for simplicity, we define $\rho_\vv := 1 - \frac{\beta\mu_g}{2}$ and
\begin{align}
    c_\beta := \frac{4}{\beta\mu_g}\left(\frac{2L_{f,1}^2}{\mu_g^2} + \frac{2L_{f,0}^2L_{g,2}^2}{\mu_g^4}\right). \label{eq:c_beta}
\end{align}
\end{proof}

\begin{theorem}[Restatement of Theorem~\ref{thm:unres_win_sl}]
    Under Assumptions~\ref{asm1}-\ref{asm3} and~\ref{asm4}, let $\beta\leq \frac{1}{L_{g,1}}$, $\alpha \leq \frac{4\kappa_F\mu_gL_{g,1}}{\kappa_g\mu_g(2\kappa_F\mu_g(\mu_g + L_{g,1}) + 8c_\beta L_{g,1}^3)}$ and $\gamma \leq \min\{\frac{1}{2L_F}, \sqrt{\frac{1-\rho}{48c_\vx}}\}$, Algorithm~\ref{alg:WOBO-SL} can obtain
    \begin{align*}
        \mathrm{Reg}_w(T) \leq& \frac{32Q}{\gamma} + \frac{16(1+\eta^w)L_{f,0}L_{g,1}DT}{\gamma\mu_gW} + \frac{32QT}{\gamma W} + \frac{24\Delta_1'}{1-\rho} + \frac{24C_wT}{(1-\rho)W^2} = O\left(\frac{T}{W}\right),
    \end{align*}
    where $\kappa_F$, $c_\beta$, $c_\vx$, $\Delta_1'$ and $C_w$ are some constants.
\end{theorem}
\begin{proof}
Remind of (\ref{A.14_1}), it still holds by replacing $f_t$, $g_t$ with $\widehat{f}_{t,w}$, $\widehat{g}_{t,w}$. Thus we have
\begin{align*}
    &\left\|\nabla \widehat{f}_{t,w}(\vx_t, \vy_{t,w}^*(\vx_t)) - \nabla \widehat{f}_{t,w}(\vx_t, \vy_t, \vv_t)\right\|^2 \\
    \leq& \left(2L_{f,1}^2 + \frac{4L_{f,0}^2L_{g,2}^2}{\mu_g^2}\right)\left\|\vy_t - \vy_{t,w}^*(\vx_t)\right\|^2 + 4L_{g,1}^2\left\|\vv_t - \vv_{t,w}^*(\vx_t, \vy_{t,w}^*(\vx_t))\right\|^2 \\
    \leq& \kappa_F\mu_g^2\left\|\vy_t - \vy_{t,w}^*(\vx_t)\right\|^2 + 8L_{g,1}^2\left\|\vv_t - \vv_{t,w}^*(\vx_t, \vy_t)\right\|^2 =: \Delta_t,
\end{align*}
where $\kappa_F$ is the same as in (\ref{eq:kappa_F}). Substituting Lemmas~\ref{lem:y_bound_win_3} and~\ref{lem:v_bound_win_3} into the above inequality, we obtain
\begin{align}
    \Delta_t \leq& \rho\Delta_{t-1} - \left(\frac{2\kappa_F\mu_g^2(\mu_g+L_{g,1})}{\alpha\mu_gL_{g,1}}\left(\frac{2\alpha}{\mu_g+L_{g,1}} - \alpha^2\right) - 8\alpha^2c_\beta L_{g,1}^2\right)\left\|\nabla_\vy \widehat{g}_{t,w}(\vx_t, \vy_t)\right\|^2 \nonumber\\
    & + c_\vx\left\|\vx_{t-1} - \vx_t\right\|^2 + \frac{2\kappa_F\mu_g^2(\mu_g + L_{g,1})(1+\eta^w)^2L_{g,1}D^2}{\alpha\mu_g^3W^2} + \frac{32(1+\eta^w)^2L_{f,0}^2L_{g,1}^4}{\beta\mu_g^5W^2} \\
    \leq& \rho^{t-1}\Delta_1 + c_\vx\sum_{j=0}^{t-2}\rho^j\left\|\vx_{t-1-j} - \vx_{t-j}\right\|^2 + \frac{C_w}{(1-\rho)W^2}, \label{D.1_2}
\end{align}
where $\rho := \max\{\rho_\vy, \rho_\vv\}$. In the last inequality, we assume that
\begin{align*}
    \frac{2\kappa_F\mu_g^2(\mu_g+L_{g,1})}{\alpha\mu_gL_{g,1}}\left(\frac{2\alpha}{\mu_g+L_{g,1}} - \alpha^2\right) - 8\alpha^2c_\beta L_{g,1}^2 &\geq 0 \\
    \frac{4\kappa_F\mu_g}{L_{g,1}} - \frac{2\alpha\kappa_F\mu_g^2(\mu_g+L_{g,1})}{\mu_gL_{g,1}} - 8\alpha^2c_\beta L_{g,1}^2 \geq \frac{4\kappa_F\mu_g}{L_{g,1}} - \frac{2\alpha\kappa_F\mu_g^2(\mu_g+L_{g,1})}{\mu_gL_{g,1}} - 8\alpha c_\beta L_{g,1}^2 &\geq 0.
\end{align*}
Therefore, we choose
\begin{align*}
    \alpha \leq \frac{4\kappa_F\mu_gL_{g,1}}{\kappa_g\mu_g(2\kappa_F\mu_g(\mu_g + L_{g,1}) + 8c_\beta L_{g,1}^3)} 
\end{align*}
to ensure that the above condition holds. For notational simplicity, we define
\begin{align}
    c_\vx &:= \frac{2\kappa_F\kappa_g^2\mu_g^2(\mu_g + L_{g,1})}{\alpha\mu_gL_{g,1}} + 8c_\beta L_{g,1}^2, \label{eq:c_x}\\
    C_w &:= \frac{2\kappa_F\mu_g^2(\mu_g + L_{g,1})(1+\eta^w)^2L_{g,1}D^2}{\alpha\mu_g^3} + \frac{32(1+\eta^w)^2L_{f,0}^2L_{g,1}^4}{\beta\mu_g^5}. \nonumber
\end{align}
Summing (\ref{D.1_2}) from $t=1,\dots,T$ yields 
\begin{align*}
    \sum_{t=1}^T\left\|\widetilde{\nabla} F_{t,w}(\vx_t) - \nabla F_{t,w}(\vx_t)\right\|^2 \leq& \frac{\Delta_1'}{1-\rho} + c_\vx\sum_{t=2}^T\sum_{j=0}^{t-2}\rho^j \left\|\vx_{t-1-j} - \vx_{t-j}\right\|^2 + \frac{C_wT}{(1-\rho)W^2} \\
    \leq& \frac{\Delta_1'}{1-\rho} + \frac{c_\vx}{1-\rho}\sum_{t=2}^T \left\|\vx_{t-1} - \vx_t\right\|^2 + \frac{C_wT}{(1-\rho)W^2} \\
    \leq& \frac{\Delta_1'}{1-\rho} + \frac{2\gamma^2c_\vx}{1-\rho}\sum_{t-1}^T\left\|\widetilde{\nabla} F_{t,w}(\vx_t) - \nabla F_{t,w}(\vx_t)\right\|^2 \\
    & + \frac{2\gamma^2c_\vx}{1-\rho}\mathrm{Reg}_w(T) + \frac{C_wT}{(1-\rho)W^2}.
\end{align*}
Here $\Delta_1' = \Delta_1 + \|\widetilde{\nabla}F_1(\vz_1) - \nabla F_1(\vx_1)\|^2$.
Let $\gamma \leq \sqrt{\frac{1-\rho}{4c_\vx}}$ such that $1 - \frac{2\gamma^2c_\vx}{1-\rho} \geq \frac{1}{2}$, we obtain
\begin{align*}
    \sum_{t=2}^T\left\|\widetilde{\nabla} F_{t,w}(\vx_t) - \nabla F_{t,w}(\vx_t)\right\|^2 \leq& \frac{2\Delta_1'}{1-\rho} + \frac{4\gamma^2c_\vx}{1-\rho}\mathrm{Reg}_w(T) + \frac{2C_wT}{(1-\rho)W^2}.
\end{align*}
Substituting the above inequality into Lemma~\ref{lem:begin_win} and recalling Lemma~\ref{lem:F-F_win}, we have
\begin{align*}
    \mathrm{Reg}_w(T) \leq& \frac{8}{\gamma}\left(2Q {+} \frac{(1{+}\eta^w)L_{f,0}L_{g,1}DT}{\mu_gW} {+} \frac{2QT}{W}\right) + \frac{12\Delta_1'}{1{-}\rho} + \frac{24\gamma^2c_\vx}{1{-}\rho}\mathrm{Reg}_w(T) + \frac{12C_wT}{(1{-}\rho)W^2}.
\end{align*}
Let $\gamma \leq \sqrt{\frac{1-\rho}{48c_\vx}}$ to gaurantee $1 - \frac{24\gamma^2c_\vx}{1-\rho} \geq \frac{1}{2}$, we finally finish the proof by
\begin{align*}
    \mathrm{Reg}_w(T) \leq& \frac{32Q}{\gamma} + \frac{16(1+\eta^w)L_{f,0}L_{g,1}DT}{\gamma\mu_gW} + \frac{32QT}{\gamma W} + \frac{24\Delta_1'}{1-\rho} + \frac{24C_wT}{(1-\rho)W^2} = O\left(\frac{T}{W}\right).
\end{align*}
\end{proof}

\subsection{Proof of Theorem~\ref{thm:unres_win_dl}}\label{sec:proof_win_upperbound}

In this section, we proof the window-averaged regret upper bound of Algorithm~\ref{alg:WOBO}. For clearity, we repeatly describe our algorithm in Algorithm~\ref{alg:WOBO-DL}.

\begin{algorithm}[tb]
    \caption{Window-averaged Online Bilevel Optimizer (WOBO)}
    \label{alg:WOBO-DL}
    \textbf{Input}: $\mathbf{x}_1$, $\mathbf{y}_1$, $\mathbf{v}_1$, $\alpha$, $\beta$ $\gamma$, $w$, $\eta$, $\delta$. \\
    \textbf{Output}: Decision sequences: $\{\mathbf{x}_t\}_{t=1}^T$, $\{\mathbf{y}_t\}_{t=1}^T$. 
    
    \begin{algorithmic}[1] 
        \FOR{\(t = 1\) to \(T\)}
        \STATE Output $\mathbf{x}_t$ and $\mathbf{y}_t$ and receive feedback $f_t$ and $g_t$
        \STATE Set $(\vx_{t+1}, \vy_{t+1}, \vv_{t+1}) \leftarrow (\vx_t, \vy_t, \vv_t)$
        \REPEAT
        \STATE \(
        \mathbf{y}_{t+1} \leftarrow \mathbf{y}_{t+1} - \alpha\nabla_\mathbf{y}\widehat{g}_{t,w}(\mathbf{x}_{t+1}, \mathbf{y}_{t+1})
        \)
        \STATE \(
        \mathbf{v}_{t+1} \leftarrow \mathbf{v}_{t+1} - \beta\nabla_\mathbf{v}\widehat{\Phi}_{t,w}(\mathbf{x}_{t+1}, \mathbf{y}_{t+1}, \mathbf{v}_{t+1})
        \)
        \STATE Calculate \(
        \widetilde{\nabla}\widehat{f}_{t,w}(\mathbf{x}_{t+1}, \mathbf{y}_{t+1}, \mathbf{v}_{t+1}) = \nabla_\vx \widehat{f}_{t,w}(\vx_{t+1}, \vy_{t+1}) - \nabla_{\vx\vy}^2 \widehat{g}_{t,w}(\vx_{t+1}, \vy_{t+1})\vv_{t+1}
        \)
        \STATE \(
        \mathbf{x}_{t+1} \leftarrow \mathop{\rm argmin}_{\mathbf{u}\in\mathcal{X}}\left\{\left\langle\widetilde{\nabla}\widehat{f}_{t,w}(\mathbf{x}_{t+1}, \mathbf{y}_{t+1}, \mathbf{v}_{t+1}), \mathbf{u}\right\rangle + \frac{1}{2\gamma}\left\|\mathbf{u} - \mathbf{x}_{t+1}\right\|^2\right\}
        \)
        \UNTIL the following condition is met \\
        \(
        \left\| \widetilde{\nabla} \widehat{f}_{t,w}(\mathbf{x}_{t+1}, \mathbf{y}_{t+1}, \mathbf{v}_{t+1}) \right\|^2 + \kappa_F \left\| \nabla_\mathbf{y} \widehat{g}_{t,w}(\mathbf{x}_{t+1}, \mathbf{y}_{t+1}) \right\|^2
        \) \\
        \(
        \phantom{\left\| \widetilde{\nabla} \widehat{f}_{t,w}(\mathbf{x}_{t+1}, \mathbf{y}_{t+1}, \mathbf{v}_{t+1}) \right\|^2} + 8\kappa_g^2 \left\| \nabla_\mathbf{v} \widehat{\Phi}_{t,w}(\mathbf{x}_{t+1}, \mathbf{y}_{t+1}, \mathbf{v}_{t+1}) \right\|^2 \leq \frac{\delta^2}{W^2}
        \)
        \ENDFOR
    \end{algorithmic}
\end{algorithm}

\begin{lemma}[Proposition 2.4 in \cite{hazan2017efficient}]\label{lem:2017}
Let $\mathbf{x}$ be any point in $\mathcal{X} \subsetneq \mathbb{R}^{d_1}$, and let $f$, $g$ be differentiable functions $\mathcal{X} \rightarrow \mathbb{R}$. Then for any $\gamma > 0$, we have
\begin{align*}
    \left\|\mathcal{G}_\mathcal{X}(\mathbf{x}, \nabla f(\mathbf{x}) + \nabla g(\mathbf{x}), \gamma)\right\| \leq \left\|\mathcal{G}_\mathcal{X} (\mathbf{x}, \nabla f(\mathbf{x}), \gamma)\right\| + \left\|\nabla g(\mathbf{x})\right\|.
\end{align*}
\end{lemma}

\begin{lemma}\label{lem:L_v}
Under Assumptions~\ref{asm1},~\ref{asm2}, for all $t\in[T]$, any $\mathbf{x}\in\mathcal{X}$, $\mathbf{y}\in\mathbb{R}^{d_2}$, we have
\begin{align*}
    \left\|\nabla_\mathbf{y}\mathbf{v}_{t,w}^*(\mathbf{x}, \mathbf{y})\right\| \leq L_\mathbf{v} \quad \text{and} \quad \left\|\nabla_\mathbf{y}\mathbf{v}_{t,w}^*(\mathbf{x}, \mathbf{y})\right\| \leq  L_\mathbf{v},
\end{align*}
where $L_\mathbf{v} := \frac{L_{f,1}}{\mu_g} + \frac{L_{f,0}L_{g,2}}{\mu_g^2}$.
\end{lemma}
\begin{proof}
Note that
\begin{align*}
    \nabla_\mathbf{v}\widehat{\Phi}_{t,w}(\mathbf{x}, \mathbf{y}, \mathbf{v}_{t,w}^*(\mathbf{x}, \mathbf{y})) = \nabla_{\mathbf{y}\mathbf{y}}^2\widehat{g}_{t,w}(\mathbf{x}, \mathbf{y})\mathbf{v}_{t,w}^*(\mathbf{x}, \mathbf{y}) - \nabla_\mathbf{y}\widehat{f}_{t,w}(\mathbf{x}, \mathbf{y}) = 0,
\end{align*}
taking the derivative of the above equation with respect to $\mathbf{x}$, we obtain
\begin{align*}
    \nabla_{\mathbf{y}\mathbf{y}\mathbf{x}}^3\widehat{g}_{t,w}(\mathbf{x}, \mathbf{y})\mathbf{v}_{t,w}^*(\mathbf{x}, \mathbf{y}) + \nabla_{\mathbf{y}\mathbf{y}}^2\widehat{g}_{t,w}(\mathbf{x}, \mathbf{y})\nabla_\mathbf{x}\mathbf{v}_{t,w}^*(\mathbf{x}, \mathbf{y}) = \nabla_{\mathbf{y}\mathbf{x}}^2\widehat{f}_{t,w}(\mathbf{x}, \mathbf{y}), 
\end{align*}
thus
\begin{align*}
    \nabla_\mathbf{x} \mathbf{v}_{t,w}^*(\mathbf{x}, \mathbf{y}) =& \left[\nabla_{\mathbf{y}\mathbf{y}}^2\widehat{g}_{t,w}(\mathbf{x}, \mathbf{y})\right]^{-1}\left[\nabla_{\mathbf{y}\mathbf{x}}^2\widehat{f}_{t,w}(\mathbf{x}, \mathbf{y}) - \nabla_{\mathbf{y}\mathbf{y}\mathbf{x}}^3\widehat{g}_{t,w}(\mathbf{x}, \mathbf{y})\mathbf{v}_{t,w}^*(\mathbf{x}, \mathbf{y})\right] \\
    \left\|\nabla_\mathbf{x}\mathbf{v}_{t,w}^*(\mathbf{x}, \mathbf{y})\right\| \leq& \frac{1}{\mu_g}\left(L_{f,1} + \frac{L_{g,2}L_{f,0}}{\mu_g}\right) = \frac{L_{f,1}}{\mu_g} + \frac{L_{f,0}L_{g,2}}{\mu_g^2} =:  L_\mathbf{v},
\end{align*}
where the last inequality comes from Assumption~\ref{asm1} and~\ref{asm2}. Similarily, we have $\left\|\nabla_\mathbf{y}\mathbf{v}_{t,w}^*(\mathbf{x}, \mathbf{y})\right\| \leq  L_\mathbf{v}$.
\end{proof}

\begin{lemma}\label{lem:y_bound_win_2}
Under Assumptions~\ref{asm1},~\ref{asm2}, for all $t\in[T]$, consider the sequences $\{\mathbf{x}_t^k, \mathbf{y}_t^k, \mathbf{v}_t^k\}_{k=0}^{K_t-1}$ generated by Algorithm~\ref{alg:WOBO-DL}, we have
\begin{align*}
    \sum_{k=0}^{K_t-1}\left\|\mathbf{y}_{t,w}^*(\mathbf{x}_t^k) - \mathbf{y}_t^k\right\|^2
    \leq& \frac{2}{\alpha\mu_g}\left\|\mathbf{y}_{t,w}^*(\mathbf{x}_t^0) - \mathbf{y}_t^0\right\|^2 \\
    & + \left(1+\frac{2}{\alpha\mu_g}\right)\frac{2\gamma^2\kappa_g^2}{\alpha\mu_g}\sum_{k=0}^{K_t-1} \left\|\mathcal{G}_\mathcal{X} (\mathbf{x}_t^k, \widetilde{\nabla}\widehat{f}_{t,w}(\mathbf{x}_t^k, \mathbf{y}_t^k, \mathbf{v}_t^k), \gamma))\right\|^2,
\end{align*}
and
\begin{align*}
    &\sum_{k=0}^{K_t-1}\left\|\nabla_\mathbf{y}\widehat{g}_{t,w}(\mathbf{x}_t^k, \mathbf{y}_t^k)\right\|^2 \\
    \leq& \frac{8-4\alpha\mu_g}{\alpha^3\mu_g}\left\|\mathbf{y}_{t,w}^*(\mathbf{x}_t^0) - \mathbf{y}_t^0\right\|^2 \\
    & + \left(1+\frac{2}{\alpha\mu_g}\right)\frac{2\gamma^2\kappa_g^2(4-2\alpha\mu_g)}{\alpha^3\mu_g}\sum_{k=0}^{K_t-1} \left\|\mathcal{G}_\mathcal{X} (\mathbf{x}_t^k, \widetilde{\nabla}\widehat{f}_{t,w}(\mathbf{x}_t^k, \mathbf{y}_t^k, \mathbf{v}_t^k), \gamma))\right\|^2.
\end{align*}
\end{lemma}
\begin{proof}
Let $\alpha \leq \frac{1}{L_{g,1}}$, due to the strongly convexity of $\widehat{g}_{t,w}(\mathbf{x}, \cdot)$ we have
\begin{align}
    \left\|\mathbf{y}_t^{k+1} - \mathbf{y}_{t,w}^*(\mathbf{x}_t^k)\right\|^2 \leq& \left(1-\alpha \mu_g\right)\left\|\mathbf{y}_t^k - \mathbf{y}_{t,w}^*(\mathbf{x}_t^k)\right\|^2 \nonumber\\
    \left\|\mathbf{y}_t^{k+1} - \mathbf{y}_t^k\right\|^2 \leq& 2\left\|\mathbf{y}_t^{k+1} - \mathbf{y}_{t,w}^*(\mathbf{x}_t^k)\right\|^2 + 2\left\|\mathbf{y}_{t,w}^*(\mathbf{x}_t^k) - \mathbf{y}_t^k\right\|^2 \nonumber\\
    \leq& \left(4 - 2\alpha \mu_g\right)\left\|\mathbf{y}_t^k - \mathbf{y}_{t,w}^*(\mathbf{x}_t^k)\right\|^2, \label{B.4_1}
\end{align}
additionally, it follows that
\begin{align}
    \left\|\mathbf{y}_{t,w}^*(\mathbf{x}_t^k) - \mathbf{y}_t^k\right\|^2 
    \leq& (1+\theta)\left\|\mathbf{y}_{t,w}^*(\mathbf{x}_t^{k-1}) - \mathbf{y}_t^k\right\|^2 + \left(1+\frac{1}{\theta}\right)\left\|\mathbf{y}_{t,w}^*(\mathbf{x}_t^k) - \mathbf{y}_{t,w}^*(\mathbf{x}_t^{k-1})\right\|^2 \nonumber\\
    \overset{(i)}\leq& (1+\theta)(1-\alpha\mu_g)\left\|\mathbf{y}_{t,w}^*(\mathbf{x}_t^{k-1}) - \mathbf{y}_t^{k-1}\right\|^2 \nonumber\\
    & + \left(1+\frac{1}{\theta}\right)\gamma^2\kappa_g^2\left\|\mathcal{G}_\mathcal{X} (\mathbf{x}_t^{k-1}, \widetilde{\nabla}\widehat{f}_{t,w}(\mathbf{x}_t^{k-1}, \mathbf{y}_t^{k-1}, \mathbf{v}_t^{k-1}), \gamma)\right\|^2 \nonumber\\
    \overset{(ii)}\leq& \left(1-\frac{\alpha\mu_g}{2}\right)\left\|\mathbf{y}_{t,w}^*(\mathbf{x}_t^{k-1}) - \mathbf{y}_t^{k-1}\right\|^2 \nonumber\\
    & + \left(1+\frac{2}{\alpha\mu_g}\right)\gamma^2\kappa_g^2\left\|\mathcal{G}_\mathcal{X} (\mathbf{x}_t^{k-1}, \widetilde{\nabla}\widehat{f}_{t,w}(\mathbf{x}_t^{k-1}, \mathbf{y}_t^{k-1}, \mathbf{v}_t^{k-1}), \gamma)\right\|^2 \nonumber\\
    \leq& \rho_\mathbf{y}^k\left\|\mathbf{y}_{t,w}^*(\mathbf{x}_t^0) - \mathbf{y}_t^0\right\|^2 \nonumber\\
    & + \left(1+\frac{2}{\alpha\mu_g}\right)\gamma^2\kappa_g^2\sum_{j=0}^{k-1}\rho_\mathbf{y}^{k-1-j}\left\|\mathcal{G}_\mathcal{X} (\mathbf{x}_t^j, \widetilde{\nabla}\widehat{f}_{t,w}(\mathbf{x}_t^j, \mathbf{y}_t^j, \mathbf{v}_t^j), \gamma))\right\|^2, \label{B.4_2}
\end{align}
where $(i)$ comes from the fact that $\left\|\nabla\mathbf{y}_{t,w}^*(\mathbf{x})\right\| \leq \kappa_g$ following Lemma~\ref{lem:L_y}, and in $(ii)$ set $\theta = \frac{\alpha\mu_g}{2}$, then $\left(1+\frac{\alpha\mu_g}{2}\right)(1-\alpha\mu_g) \leq 1 - \frac{\alpha\mu_g}{2}$, $\rho_\mathbf{y} := 1 - \frac{\alpha\mu_g}{2}$ and summing (\ref{B.4_2}) over $k$, we have
\begin{align}
    &\sum_{k=0}^{K_t-1} \left\|\mathbf{y}_{t,w}^*(\mathbf{x}_t^k) - \mathbf{y}_t^k\right\|^2 \nonumber\\
    \leq& \sum_{k=0}^{K_t-1} \rho_\mathbf{y}^k\left\|\mathbf{y}_{t,w}^*(\mathbf{x}_t^0) - \mathbf{y}_t^0\right\|^2 \nonumber\\
    & + \left(1{+}\frac{2}{\alpha\mu_g}\right)\gamma^2\kappa_g^2 \sum_{k=0}^{K_t-1} \sum_{j=0}^{k-1}\rho_\mathbf{y}^{k-1-j}\left\|\mathcal{G}_\mathcal{X} (\mathbf{x}_t^j, \widetilde{\nabla}\widehat{f}_{t,w}(\mathbf{x}_t^j, \mathbf{y}_t^j, \mathbf{v}_t^j), \gamma))\right\|^2 \nonumber\\
    \leq& \frac{1}{1-\rho_\mathbf{y}}\left\|\mathbf{y}_{t,w}^*(\mathbf{x}_t^0) - \mathbf{y}_t^0\right\|^2 \nonumber\\
    & + \left(1+\frac{2}{\alpha\mu_g}\right)\frac{\gamma^2\kappa_g^2}{1-\rho_\mathbf{y}}\sum_{k=0}^{K_t-1} \left\|\mathcal{G}_\mathcal{X} (\mathbf{x}_t^k, \widetilde{\nabla}\widehat{f}_{t,w}(\mathbf{x}_t^k, \mathbf{y}_t^k, \mathbf{v}_t^k), \gamma))\right\|^2, \label{B.4_3}
\end{align}
substituting (\ref{B.4_3}) into (\ref{B.4_1}), we have
\begin{align*}
    &\sum_{k=0}^{K_t-1} \left\|\mathbf{y}_t^{k+1} - \mathbf{y}_t^k\right\|^2 \\
    \leq& \frac{8-4\alpha\mu_g}{\alpha\mu_g}\left\|\mathbf{y}_{t,w}^*(\mathbf{x}_t^0) - \mathbf{y}_t^0\right\|^2 \\
    & + \left(1+\frac{2}{\alpha\mu_g}\right)\frac{2\gamma^2\kappa_g^2(4-2\alpha\mu_g)}{\alpha\mu_g}\sum_{k=0}^{K_t-1} \left\|\mathcal{G}_\mathcal{X} (\mathbf{x}_t^k, \widetilde{\nabla}\widehat{f}_{t,w}(\mathbf{x}_t^k, \mathbf{y}_t^k, \mathbf{v}_t^k), \gamma))\right\|^2 \\
    &\sum_{k=0}^{K_t-1}\left\|\nabla_\mathbf{y}\widehat{g}_{t,w}(\mathbf{x}_t^k, \mathbf{y}_t^k)\right\|^2 \\
    \leq& \frac{8-4\alpha\mu_g}{\alpha^3\mu_g}\left\|\mathbf{y}_{t,w}^*(\mathbf{x}_t^0) - \mathbf{y}_t^0\right\|^2 \\
    & + \left(1+\frac{2}{\alpha\mu_g}\right)\frac{2\gamma^2\kappa_g^2(4-2\alpha\mu_g)}{\alpha^3\mu_g}\sum_{k=0}^{K_t-1} \left\|\mathcal{G}_\mathcal{X} (\mathbf{x}_t^k, \widetilde{\nabla}\widehat{f}_{t,w}(\mathbf{x}_t^k, \mathbf{y}_t^k, \mathbf{v}_t^k), \gamma))\right\|^2,
\end{align*}
where the second inequality comes from $\mathbf{y}_t^{k+1} = \mathbf{y}_t^k - \alpha\nabla_\mathbf{y}\widehat{g}_{t,w}(\mathbf{x}_t^k, \mathbf{y}_t^k)$.
\end{proof}

\begin{lemma}\label{lem:v_bound_win_2}
Under Assumptions~\ref{asm1},~\ref{asm2}, for all $t\in[T]$, consider the sequences $\{\mathbf{x}_t^k, \mathbf{y}_t^k, \mathbf{v}_t^k\}_{k=0}^{K_t-1}$ generated by Algorithm~\ref{alg:WOBO-DL}, we have
\begin{align*}
    &\sum_{k=0}^{K_t-1} \left\|\mathbf{v}_{t,w}^*(\mathbf{x}_t^k, \mathbf{y}_t^k) - \mathbf{v}_t^k\right\|^2 \\
    \leq& \frac{2}{\beta\mu_g}\left\|\mathbf{v}_{t,w}^*(\mathbf{x}_t^0, \mathbf{y}_t^0) - \mathbf{v}_t^0\right\|^2 + \left(1+\frac{2}{\beta\mu_g}\right)\frac{2(1+\lambda) L_\mathbf{v}^2}{\beta\mu_g}\frac{8-4\alpha\mu_g}{\alpha\mu_g}\left\|\mathbf{y}_{t,w}^*(\mathbf{x}_t^0) - \mathbf{y}_t^0\right\|^2 \\
    & + \left(1{+}\frac{2}{\beta\mu_g}\right)\frac{2(1{+}\lambda) L_\mathbf{v}^2}{\beta\mu_g}\left(1{+}\frac{2}{\alpha\mu_g}\right)\frac{2\gamma^2\kappa_g^2(4{-}2\alpha\mu_g)}{\alpha\mu_g}\sum_{k=0}^{K_t-1}\left\|\mathcal{G}_\mathcal{X} (\mathbf{x}_t^k, \widetilde{\nabla}\widehat{f}_{t,w}(\mathbf{x}_t^k, \mathbf{y}_t^k, \mathbf{v}_t^k), \gamma)\right\|^2 \\
    & + \left(1+\frac{2}{\beta\mu_g}\right)\left(1+\frac{1}{\lambda}\right) \frac{2\gamma^2L_\mathbf{v}^2}{\beta\mu_g} \sum_{k=0}^{K_t-1}\left\|\mathcal{G}_\mathcal{X} (\mathbf{x}_t^k, \widetilde{\nabla}\widehat{f}_{t,w}(\mathbf{x}_t^k, \mathbf{y}_t^k, \mathbf{v}_t^k), \gamma))\right\|^2,
\end{align*}
and
\begin{align*}
    &\sum_{k=0}^{K_t-1} \left\|\nabla_\mathbf{v} \widehat{\Phi}_{t,w}(\mathbf{x}_t^k, \mathbf{y}_t^k, \mathbf{v}_t^k)\right\|^2 \\
    \leq& \frac{8-4\beta\mu_g}{\beta^3\mu_g}\left\|\mathbf{v}_{t,w}^*(\mathbf{x}_t^0, \mathbf{y}_t^0) - \mathbf{v}_t^0\right\|^2 + \left(\frac{16}{\beta\mu_g} - 4\beta\mu_g\right)\frac{2(8-4\alpha\mu_g)L_\mathbf{v}^2}{\alpha\beta^3\mu_g^2} \left\|\mathbf{y}_{t,w}^*(\mathbf{x}_t^0) - \mathbf{y}_t^0\right\|^2 \\
    & + \left(\frac{16}{\beta\mu_g} - 4\beta\mu_g\right)\left(\frac{16}{\alpha\mu_g} - 4\alpha\mu_g\right)\frac{2\gamma^2\kappa_g^2L_\mathbf{v}^2}{\alpha\beta^3\mu_g^2} \sum_{k=0}^{K_t-1}\left\|\mathcal{G}_\mathcal{X} (\mathbf{x}_t^k, \widetilde{\nabla}\widehat{f}_{t,w}(\mathbf{x}_t^k, \mathbf{y}_t^k, \mathbf{v}_t^k), \gamma)) \right\|^2 \\
    & + \left(\frac{16}{\beta\mu_g} - 4\beta\mu_g\right) \frac{2 \gamma^2L_\mathbf{v}^2}{\beta^3\mu_g}\sum_{k=0}^{K_t-1}\left\|\mathcal{G}_\mathcal{X} (\mathbf{x}_t^k, \widetilde{\nabla}\widehat{f}_{t,w}(\mathbf{x}_t^k, \mathbf{y}_t^k, \mathbf{v}_t^k), \gamma))  \right\|^2.
\end{align*}
\end{lemma}
\begin{proof}
Let $\beta \leq \frac{1}{L_{g,1}}$, due to the strongly convexity of $\widehat{\Phi}_{t,w}(\mathbf{x}, \mathbf{y}, \cdot)$ we have
\begin{align}
    \left\|\mathbf{v}_t^{k+1} - \mathbf{v}_{t,w}^*(\mathbf{x}_t^k, \mathbf{y}_t^k)\right\|^2 \leq& \left(1-\beta \mu_g\right)\left\|\mathbf{v}_t^k - \mathbf{v}_{t,w}^*(\mathbf{x}_t^k, \mathbf{y}_t^k)\right\|^2 \nonumber\\
    \left\|\mathbf{v}_t^{k+1} - \mathbf{v}_t^k\right\|^2 \leq& \left(4 - 2\beta \mu_g\right)\left\|\mathbf{v}_t^k - \mathbf{v}_{t,w}^*(\mathbf{x}_t^k, \mathbf{y}_t^k)\right\|^2, \label{B.5_1}
\end{align}
additionally, it follows that
\begin{align}
    &\left\|\mathbf{v}_{t,w}^*(\mathbf{x}_t^k, \mathbf{y}_t^k) - \mathbf{v}_t^k\right\|^2 \nonumber\\
    \leq& (1+\theta)\left\|\mathbf{v}_{t,w}^*(\mathbf{x}_t^{k-1}, \mathbf{y}_t^{k-1}) - \mathbf{v}_t^k\right\|^2 + \left(1+\frac{1}{\theta}\right)\left\|\mathbf{v}_{t,w}^*(\mathbf{x}_t^k, \mathbf{y}_t^k) - \mathbf{v}_{t,w}^*(\mathbf{x}_t^{k-1}, \mathbf{y}_t^{k-1})\right\|^2 \nonumber\\
    \leq& (1+\theta)(1-\beta\mu_g)\left\|\mathbf{v}_{t,w}^*(\mathbf{x}_t^{k-1}, \mathbf{y}_t^{k-1}) - \mathbf{v}_t^{k-1}\right\|^2 \nonumber\\
    & + \left(1+\frac{1}{\theta}\right)(1+\lambda)\left\|\mathbf{v}_{t,w}^*(\mathbf{x}_t^k, \mathbf{y}_t^k) - \mathbf{v}_{t,w}^*(\mathbf{x}_t^k, \mathbf{y}_t^{k-1})\right\|^2 \nonumber\\
    & + \left(1+\frac{1}{\theta}\right)\left(1+\frac{1}{\lambda}\right)\left\|\mathbf{v}_{t,w}^*(\mathbf{x}_t^k, \mathbf{y}_t^{k-1}) - \mathbf{v}_{t,w}^*(\mathbf{x}_t^{k-1}, \mathbf{y}_t^{k-1})\right\|^2 \nonumber\\
    \overset{(i)}\leq& \left(1 - \frac{\beta\mu_g}{2}\right)\left\|\mathbf{v}_{t,w}^*(\mathbf{x}_t^{k-1}, \mathbf{y}_t^{k-1}) - \mathbf{v}_t^{k-1}\right\|^2 + \left(1+\frac{2}{\beta\mu_g}\right)(1+\lambda) L_\mathbf{v}^2\left\|\mathbf{y}_t^k - \mathbf{y}_t^{k-1}\right\|^2 \nonumber\\
    & + \left(1+\frac{2}{\beta\mu_g}\right)\left(1+\frac{1}{\lambda}\right) L_\mathbf{v}^2\left\|\mathbf{x}_t^k - \mathbf{x}_t^{k-1}\right\|^2 \nonumber
\end{align}
where the last inequality comes from Lemma~\ref{lem:L_v}, then define $\rho_\mathbf{v} := 1-\frac{\beta\mu_g}{2}$ and substitute the update rule of $\mathbf{x}$ and $\mathbf{y}$ shown in (\ref{nowin_update}), after telescoping we obtain
\begin{align}
    &\left\|\mathbf{v}_{t,w}^*(\mathbf{x}_t^k, \mathbf{y}_t^k) - \mathbf{v}_t^k\right\|^2 \nonumber\\
    \leq&  \rho_\mathbf{v}^k\left\|\mathbf{v}_{t,w}^*(\mathbf{x}_t^0, \mathbf{y}_t^0) - \mathbf{v}_t^0\right\|^2 + \left(1+\frac{2}{\beta\mu_g}\right)(1+\lambda) \alpha^2L_\mathbf{v}^2\sum_{j=0}^{k-1}\rho_\mathbf{v}^{k-1-j}\left\|\nabla_\mathbf{y}\widehat{g}_{t,w}(\mathbf{x}_t^j, \mathbf{y}_t^j)\right\|^2 \nonumber\\
    & + \left(1+\frac{2}{\beta\mu_g}\right)\left(1+\frac{1}{\lambda}\right) \gamma^2L_\mathbf{v}^2\sum_{j=0}^{k-1}\rho_\mathbf{v}^{k-1-j}\left\|\mathcal{G}_\mathcal{X} (\mathbf{x}_t^j, \widetilde{\nabla}\widehat{f}_{t,w}(\mathbf{x}_t^j, \mathbf{y}_t^j, \mathbf{v}_t^j), \gamma))\right\|^2, \label{B.5_2}
\end{align}
and summing (\ref{B.5_2}) over $k$, we have
\begin{align}
    &\sum_{k=0}^{K_t-1} \left\|\mathbf{v}_{t,w}^*(\mathbf{x}_t^k, \mathbf{y}_t^k) - \mathbf{v}_t^k\right\|^2 \nonumber\\
    \leq& \frac{1}{1{-}\rho_\mathbf{v}}\left\|\mathbf{v}_{t,w}^*(\mathbf{x}_t^0, \mathbf{y}_t^0) - \mathbf{v}_t^0\right\|^2 + \left(1+\frac{2}{\beta\mu_g}\right)\frac{(1{+}\lambda) \alpha^2L_\mathbf{v}^2}{1-\rho_\mathbf{v}}\sum_{k=0}^{K_t-1} \left\|\nabla_\mathbf{y}\widehat{g}_{t,w}(\mathbf{x}_t^k, \mathbf{y}_t^k)\right\|^2 \nonumber\\
    & + \left(1+\frac{2}{\beta\mu_g}\right)\left(1+\frac{1}{\lambda}\right) \frac{\gamma^2L_\mathbf{v}^2}{1-\rho_\mathbf{v}} \sum_{k=0}^{K_t-1}\left\|\mathcal{G}_\mathcal{X} (\mathbf{x}_t^k, \widetilde{\nabla}\widehat{f}_{t,w}(\mathbf{x}_t^k, \mathbf{y}_t^k, \mathbf{v}_t^k), \gamma))\right\|^2. \label{B.5_3}
\end{align}
Remind of Lemma~\ref{lem:y_bound_win_2}, substitute its first inequality into (\ref{B.5_3}), we then obtain
\begin{align}
    &\sum_{k=0}^{K_t-1} \left\|\mathbf{v}_{t,w}^*(\mathbf{x}_t^k, \mathbf{y}_t^k) - \mathbf{v}_t^k\right\|^2 \nonumber\\
    \leq& \frac{2}{\beta\mu_g}\left\|\mathbf{v}_{t,w}^*(\mathbf{x}_t^0, \mathbf{y}_t^0) - \mathbf{v}_t^0\right\|^2 + \left(1+\frac{2}{\beta\mu_g}\right)\frac{2(1+\lambda) \alpha^2L_\mathbf{v}^2}{\beta\mu_g}\frac{8-4\alpha\mu_g}{\alpha^3\mu_g}\left\|\mathbf{y}_{t,w}^*(\mathbf{x}_t^0) - \mathbf{y}_t^0\right\|^2 \nonumber\\
    & + \left(1{+}\frac{2}{\beta\mu_g}\right)\frac{2(1{+}\lambda) \alpha^2L_\mathbf{v}^2}{\beta\mu_g}\left(1{+}\frac{2}{\alpha\mu_g}\right)\frac{2\gamma^2\kappa_g^2(4{-}2\alpha\mu_g)}{\alpha^3\mu_g}\sum_{k=0}^{K_t-1}\left\|\mathcal{G}_\mathcal{X} (\mathbf{x}_t^k, \widetilde{\nabla}\widehat{f}_{t,w}(\mathbf{x}_t^k, \mathbf{y}_t^k, \mathbf{v}_t^k), \gamma)\right\|^2 \nonumber\\
    & + \left(1+\frac{2}{\beta\mu_g}\right)\left(1+\frac{1}{\lambda}\right) \frac{2\gamma^2L_\mathbf{v}^2}{\beta\mu_g} \sum_{k=0}^{K_t-1}\left\|\mathcal{G}_\mathcal{X} (\mathbf{x}_t^k, \widetilde{\nabla}\widehat{f}_{t,w}(\mathbf{x}_t^k, \mathbf{y}_t^k, \mathbf{v}_t^k), \gamma))\right\|^2, \label{B.5_4}
\end{align}
finally substituting (\ref{B.5_4}) into (\ref{B.5_1}), it implies that
\begin{align}
    &\sum_{k=0}^{K_t-1}\left\|\mathbf{v}_t^{k+1} {-} \mathbf{v}_t^k\right\|^2 \nonumber\\
    \leq& \frac{8{-}4\beta\mu_g}{\beta\mu_g}\left\|\mathbf{v}_{t,w}^*(\mathbf{x}_t^0, \mathbf{y}_t^0) - \mathbf{v}_t^0\right\|^2 + \left(\frac{16}{\beta\mu_g} - 4\beta\mu_g\right)\frac{(1{+}\lambda) L_\mathbf{v}^2}{\beta\mu_g}\frac{8{-}4\alpha\mu_g}{\alpha\mu_g}\left\|\mathbf{y}_{t,w}^*(\mathbf{x}_t^0) - \mathbf{y}_t^0\right\|^2 \nonumber\\
    & {+} \left(\frac{16}{\beta\mu_g} {-} 4\beta\mu_g\right)\frac{(1{+}\lambda) L_\mathbf{v}^2}{\beta\mu_g}\left(\frac{16}{\alpha\mu_g} {-} 4\alpha\mu_g\right)\frac{\gamma^2\kappa_g^2}{\alpha\mu_g}\sum_{k=0}^{K_t-1}\left\|\mathcal{G}_\mathcal{X} (\mathbf{x}_t^k, \widetilde{\nabla}\widehat{f}_{t,w}(\mathbf{x}_t^k, \mathbf{y}_t^k, \mathbf{v}_t^k), \gamma)) \right\|^2 \nonumber\\
    & {+} \left(\frac{16}{\beta\mu_g} - 4\beta\mu_g\right)\left(1+\frac{1}{\lambda}\right)\frac{ \gamma^2L_\mathbf{v}^2}{\beta\mu_g}\sum_{k=0}^{K_t-1}\left\|\mathcal{G}_\mathcal{X} (\mathbf{x}_t^k, \widetilde{\nabla}\widehat{f}_{t,w}(\mathbf{x}_t^k, \mathbf{y}_t^k, \mathbf{v}_t^k), \gamma))  \right\|^2, \label{B.5_5}
\end{align}
let $\lambda = 1$, and substitute $\mathbf{v}_t^{k+1} = \mathbf{v}_t^k - \beta\nabla_\mathbf{v}\widehat{\Phi}_{t,w}(\mathbf{x}_t^k, \mathbf{y}_t^k, \mathbf{v}_t^k)$ into (\ref{B.5_5}), the proof is finished by
\begin{align*}
    &\sum_{k=0}^{K_t-1}\left\|\nabla_\mathbf{v}\widehat{\Phi}_{t,w}(\mathbf{x}_t^k, \mathbf{y}_t^k, \mathbf{v}_t^k)\right\|^2 \\ \leq& \frac{8-4\beta\mu_g}{\beta^3\mu_g}\left\|\mathbf{v}_{t,w}^*(\mathbf{x}_t^0, \mathbf{y}_t^0) - \mathbf{v}_t^0\right\|^2 + \left(\frac{16}{\beta\mu_g} - 4\beta\mu_g\right)\frac{2(8-4\alpha\mu_g)L_\mathbf{v}^2}{\alpha\beta^3\mu_g^2} \left\|\mathbf{y}_{t,w}^*(\mathbf{x}_t^0) - \mathbf{y}_t^0\right\|^2 \\
    & + \left(\frac{16}{\beta\mu_g} - 4\beta\mu_g\right)\left(\frac{16}{\alpha\mu_g} - 4\alpha\mu_g\right)\frac{2\gamma^2\kappa_g^2L_\mathbf{v}^2}{\alpha\beta^3\mu_g^2} \sum_{k=0}^{K_t-1}\left\|\mathcal{G}_\mathcal{X} (\mathbf{x}_t^k, \widetilde{\nabla}\widehat{f}_{t,w}(\mathbf{x}_t^k, \mathbf{y}_t^k, \mathbf{v}_t^k), \gamma)) \right\|^2 \\
    & + \left(\frac{16}{\beta\mu_g} - 4\beta\mu_g\right) \frac{2 \gamma^2L_\mathbf{v}^2}{\beta^3\mu_g}\sum_{k=0}^{K_t-1}\left\|\mathcal{G}_\mathcal{X} (\mathbf{x}_t^k, \widetilde{\nabla}\widehat{f}_{t,w}(\mathbf{x}_t^k, \mathbf{y}_t^k, \mathbf{v}_t^k), \gamma))  \right\|^2.
\end{align*}
\end{proof}

\begin{theorem}[Restatement of Theorem~\ref{thm:unres_win_dl}]
Under Assumptions~\ref{asm1}-\ref{asm3} and~\ref{asm4}, let $\beta = \frac{1}{L_{g,1}}$, $\alpha = \frac{1}{L_{g,1}}$ and $\gamma \leq \min\left\{\frac{1}{4L_F}, \frac{1}{456\kappa_g^4L_{g,1}L_\mathbf{v}\sqrt{\kappa_F}}\right\}$, Algorithm~\ref{alg:WOBO-DL} can obtain
\begin{align*}
    \mathrm{Reg}_w(T) \leq \left(8\delta^2 + 4(1+\eta^w)^2L_{f,1}^2\kappa_g^2D^2 + 2(1+\eta^w)^2L_{f,0}^2\right)\frac{T}{W^2} = O\left(\frac{T}{W^2}\right),
\end{align*}
with $L_F$, $L_\mathbf{v}$ and $\kappa_F$ are some constants, the total number of inner iterations $\mathcal{I}_T$ satisfies
\begin{align*}
    \mathcal{I}_T \leq& \frac{8QW^2}{\gamma\delta^2} + \frac{4(1{+}\eta^w)L_{f,0}L_{g,1}DWT}{\gamma\delta^2\mu_g} + \frac{8QWT}{\gamma\delta^2} + 24\kappa_g^3L_{g,1}^2 \left(\frac{8(1{+}\eta^w)L_{f,0}^2L_{g,1}^2}{\delta^2\mu_g^4} {+} \frac
    {1}{L_{g,1}^2}\right)T\\
    & + (3\kappa_F\kappa_gL_{g,1}^2 + 1024\kappa_g^5L_{g,1}^2L_\mathbf{v}^2)\left(\frac{8(1+\eta^w)^2L_{g,1}^2D^2}{\delta^2\mu_g^2} + \frac{8}{\mu_g^2\kappa_F}\right) T = O\left(WT\right).
\end{align*}
\end{theorem}
\begin{proof}

It begins with
\begin{align}
    &\left\|\mathcal{G}_\mathcal{X} (\mathbf{x}_t, \nabla \widehat{f}_{t,w}(\mathbf{x}_t, \mathbf{y}_{t,w}^*(\mathbf{x}_t)), \gamma)\right\|^2 \nonumber\\
    =& \left\|\mathcal{G}_\mathcal{X} (\mathbf{x}_t, \eta\nabla \widehat{f}_{t-1,w}(\mathbf{x}_t, \mathbf{y}_{t,w}^*(\mathbf{x}_t)) {+} \frac{1}{W}\left(\nabla f_t(\mathbf{x}_t, \mathbf{y}_{t,w}^*(\mathbf{x}_t)) {-} \eta^w\nabla f_{t-w}(\mathbf{x}_t, \mathbf{y}_{t,w}^*(\mathbf{x}_t))\right), \gamma)\right\|^2 \nonumber\\
    \overset{(i)}\leq& \left(\left\|\mathcal{G}_\mathcal{X} (\mathbf{x}_t, \eta\nabla \widehat{f}_{t-1,w}(\mathbf{x}_t, \mathbf{y}_{t,w}^*(\mathbf{x}_t)), \gamma)\right\| {+} \frac{1}{W}\left\|\nabla f_t(\mathbf{x}_t, \mathbf{y}_{t,w}^*(\mathbf{x}_t)) {-} \eta^w\nabla f_{t-w}(\mathbf{x}_t, \mathbf{y}_{t,w}^*(\mathbf{x}_t))\right\|\right)^2 \nonumber\\
    \leq& 2\left\|\mathcal{G}_\mathcal{X} (\mathbf{x}_t, \nabla \widehat{f}_{t-1,w}(\mathbf{x}_t, \mathbf{y}_{t,w}^*(\mathbf{x}_t)), \gamma)\right\|^2 {+} \frac{2}{W^2}\left\|\nabla f_t(\mathbf{x}_t, \mathbf{y}_{t,w}^*(\mathbf{x}_t)) {-} \eta^w\nabla f_{t-w}(\mathbf{x}_t, \mathbf{y}_{t,w}^*(\mathbf{x}_t))\right\|^2 \nonumber\\
    \overset{(ii)}\leq& 2\left\|\mathcal{G}_\mathcal{X} (\mathbf{x}_t, \nabla \widehat{f}_{t-1,w}(\mathbf{x}_t, \mathbf{y}_{t,w}^*(\mathbf{x}_t)), \gamma)\right\|^2 + \frac{2(1+\eta^w)^2L_{f,0}^2}{W^2}, \label{C.4_1}
\end{align}
where $(i)$ comes from Lemma~\ref{lem:2017} and $(ii)$ comes from the $L_{f,0}$-Lipschitz Continuity of $f_t(\cdot)$ in Assumption~\ref{asm2}. Then
\begin{align}
    &\left\|\mathcal{G}_\mathcal{X} (\mathbf{x}_t, \nabla \widehat{f}_{t-1,w}(\mathbf{x}_t, \mathbf{y}_{t,w}^*(\mathbf{x}_t)), \gamma)\right\|^2 \nonumber\\
    =& \left\|\mathcal{G}_\mathcal{X} (\mathbf{x}_t, \nabla \widehat{f}_{t-1,w}(\mathbf{x}_t, \mathbf{y}_{t-1,w}^*(\mathbf{x}_t)) {+} \nabla \widehat{f}_{t-1,w}(\mathbf{x}_t, \mathbf{y}_{t,w}^*(\mathbf{x}_t)) {-} \nabla \widehat{f}_{t-1,w}(\mathbf{x}_t, \mathbf{y}_{t-1,w}^*(\mathbf{x}_t)), \gamma)\right\|^2 \nonumber\\
    \overset{(i)}\leq& \left(\left\|\mathcal{G}_\mathcal{X} (\mathbf{x}_t, \nabla \widehat{f}_{t-1,w}(\mathbf{x}_t, \mathbf{y}_{t-1,w}^*(\mathbf{x}_t)), \gamma)\right\| {+} \left\|\nabla \widehat{f}_{t-1,w}(\mathbf{x}_t, \mathbf{y}_{t,w}^*(\mathbf{x}_t)) {-} \nabla \widehat{f}_{t-1,w}(\mathbf{x}_t, \mathbf{y}_{t-1,w}^*(\mathbf{x}_t))\right\|\right)^2 \nonumber\\
    \leq& 2\left\|\mathcal{G}_\mathcal{X} (\mathbf{x}_t, \nabla \widehat{f}_{t-1,w}(\mathbf{x}_t, \mathbf{y}_{t-1,w}^*(\mathbf{x}_t)), \gamma)\right\|^2 {+} 2L_{f,1}^2\left\|\mathbf{y}_{t,w}^*(\mathbf{x}_t) - \mathbf{y}_{t-1,w}^*(\mathbf{x}_t)\right\|^2 \nonumber\\
    \overset{(ii)}\leq& 2\left\|\mathcal{G}_\mathcal{X} (\mathbf{x}_t, \nabla \widehat{f}_{t-1,w}(\mathbf{x}_t, \mathbf{y}_{t-1,w}^*(\mathbf{x}_t)), \gamma)\right\|^2 + \frac{2(1+\eta^w)^2L_{f,1}^2L_{g,1}^2D^2}{\mu_g^2W^2}, \label{C.4_2}
\end{align}
where $(i)$ comes from Lemma~\ref{lem:2017} and $(ii)$ from Lemma~\ref{lem:y_bound_win_1}, substituting (\ref{C.4_2}) into (\ref{C.4_1}) to get
\begin{align}
    &\left\|\mathcal{G}_\mathcal{X} (\mathbf{x}_t, \nabla \widehat{f}_{t,w}(\mathbf{x}_t, \mathbf{y}_{t,w}^*(\mathbf{x}_t)), \gamma)\right\|^2 \nonumber\\
    \leq& 4\left\|\mathcal{G}_\mathcal{X} (\mathbf{x}_t, \nabla \widehat{f}_{t-1,w}(\mathbf{x}_t, \mathbf{y}_{t-1,w}^*(\mathbf{x}_t)), \gamma)\right\|^2 + \frac{4(1+\eta^w)^2L_{f,1}^2L_{g,1}^2D^2}{\mu_g^2W^2} + \frac{2(1+\eta^w)^2L_{f,0}^2}{W^2} \nonumber\\
    \leq& 8\left\|\mathcal{G}_\mathcal{X} (\mathbf{x}_t, \widetilde{\nabla} \widehat{f}_{t-1,w}(\mathbf{x}_t, \mathbf{y}_t, \mathbf{v}_t), \gamma)\right\|^2 + 8\left\|\nabla \widehat{f}_{t-1,w}(\mathbf{x}_t, \mathbf{y}_{t-1,w}^*(\mathbf{x}_t)) - \widetilde{\nabla}\widehat{f}_{t-1,w}(\mathbf{x}_t, \mathbf{y}_t, \mathbf{v}_t)\right\|^2 \nonumber\\
    & + \frac{4(1+\eta^w)^2L_{f,1}^2L_{g,1}^2D^2}{\mu_g^2W^2} + \frac{2(1+\eta^w)^2L_{f,0}^2}{W^2}. \label{C.4_3}
\end{align}
We also have
\begin{align}
    &\left\|\nabla \widehat{f}_{t-1,w}(\mathbf{x}_t, \mathbf{y}_{t-1,w}^*(\mathbf{x}_t)) - \widetilde{\nabla}\widehat{f}_{t-1,w}(\mathbf{x}_t, \mathbf{y}_t, \mathbf{v}_t)\right\|^2 \nonumber\\
    \leq& 2\left\|\nabla_\mathbf{x}\widehat{f}_{t-1,w}(\mathbf{x}_t,\mathbf{y}_{t-1,w}^*(\mathbf{x}_t)) - \nabla_\mathbf{x}\widehat{f}_{t-1,w}(\mathbf{x}_t, \mathbf{y}_t)\right\|^2 \nonumber\\
    & + 2\left\|\nabla_{\mathbf{x}\mathbf{y}}^2\widehat{g}_{t-1,w}(\mathbf{x}_t, \mathbf{y}_{t-1,w}^*(\mathbf{x}_t))\mathbf{v}_{t-1,w}^*(\mathbf{x}_t, \mathbf{y}_{t-1,w}^*(\mathbf{x}_t)) - \nabla_{\mathbf{x}\mathbf{y}}^2\widehat{g}_{t-1,w}(\mathbf{x}_t, \mathbf{y}_t)\mathbf{v}_t\right\|^2 \nonumber\\
    \leq& 2L_{f,1}^2\left\|\mathbf{y}_{t-1,w}^*(\mathbf{x}_t) - \mathbf{y}_t\right\|^2 + 4L_{g,2}^2 \left\|\mathbf{v}_{t-1,w}^*(\mathbf{x}_t, \mathbf{y}_{t-1,w}^*(\mathbf{x}_t))\right\|^2\left\|\mathbf{y}_{t-1,w}^*(\mathbf{x}_t) - \mathbf{y}_t\right\|^2 \nonumber\\
    & + 4L_{g,1}^2\left\|\mathbf{v}_{t-1.w}^*(\mathbf{x}_t, \mathbf{y}_{t-1,w}^*(\mathbf{x}_t)) - \mathbf{v}_t\right\|^2, \nonumber
\end{align}
note that due to Lemma~\ref{lem:v_cons}, we still have $\left\|\mathbf{v}_{t,w}^*(\cdot)\right\| \leq \frac{L_{f,0}}{\mu_g}$ for all $t\in[T]$, and then it follows that
\begin{align}
    &\left\|\nabla \widehat{f}_{t-1,w}(\mathbf{x}_t, \mathbf{y}_{t-1,w}^*(\mathbf{x}_t)) - \widetilde{\nabla}\widehat{f}_{t-1,w}(\mathbf{x}_t, \mathbf{y}_t, \mathbf{v}_t)\right\|^2 \nonumber\\
    \leq& \left(2L_{f,1}^2 {+} \frac{4L_{f,0}^2L_{g,2}^2}{\mu_g^2}\right)\left\|\mathbf{y}_{t-1,w}^*(\mathbf{x}_t) - \mathbf{y}_t\right\|^2 {+} 8L_{g,1}^2\left\|\mathbf{v}_{t-1,w}^*(\mathbf{x}_t, \mathbf{y}_{t-1,w}^*(\mathbf{x}_t)) {-} \mathbf{v}_{t-1,w}^*(\mathbf{x}_t, \mathbf{y}_t)\right\|^2 \nonumber\\
    & + 8L_{g,1}^2\left\|\mathbf{v}_{t-1,w}^*(\mathbf{x}_t, \mathbf{y}_t) - \mathbf{v}_t\right\|^2 \nonumber\\
    \overset{(i)}\leq& \left(2L_{f,1}^2 {+} \frac{4L_{g,1}^2L_{g,2}^2}{\mu_g^2} {+} \frac{16L_{f,1}^2L_{g,1}^2}{\mu_g^2} {+} \frac{16L_{f,0}^2L_{f,1}^2L_{g,2}^2}{\mu_g^4}\right)\left\|\mathbf{y}_{t-1,w}^*(\mathbf{x}_t) - \mathbf{y}_t\right\|^2 \nonumber\\
    & + 8L_{g,1}^2\left\|\mathbf{v}_{t-1,w}^*(\mathbf{x}_t, \mathbf{y}_t) - \mathbf{v}_t\right\|^2, \nonumber\\
    =:& \kappa_F\mu_g^2\left\|\mathbf{y}_{t-1,w}^*(\mathbf{x}_t) - \mathbf{y}_t\right\|^2 + 8L_{g,1}^2\left\|\mathbf{v}_{t-1,w}^*(\mathbf{x}_t, \mathbf{y}_t) - \mathbf{v}_t\right\|^2, \label{C.4_4} 
\end{align}
where $(i)$ comes from
\begin{align*}
    \left\|\mathbf{v}_{t-1,w}^*(\mathbf{x}_t, \mathbf{y}_{t-1,w}^*(\mathbf{x}_t)) - \mathbf{v}_{t-1,w}^*(\mathbf{x}_t, \mathbf{y}_t)\right\|^2 \leq \left(\frac{2L_{f,1}^2}{\mu_g^2} + \frac{2L_{f,0}^2L_{g,2}^2}{\mu_g^4}\right)\left\|\mathbf{y}_t - \mathbf{y}_{t-1,w}^*(\mathbf{x}_t)\right\|^2,
\end{align*}
as the similar result in (\ref{A.6_9}), $\kappa_F$ is defined in (\ref{eq:kappa_F}).

Substituting (\ref{C.4_4}) into (\ref{C.4_2}), we have
\begin{align}
    &\left\|\mathcal{G}_\mathcal{X} (\mathbf{x}_t, \nabla \widehat{f}_{t-1,w}(\mathbf{x}_t, \mathbf{y}_{t-1,w}^*(\mathbf{x}_t)), \gamma)\right\|^2 \nonumber\\
    \leq& 2\left\|\nabla \widehat{f}_{t-1,w}(\mathbf{x}_t, \mathbf{y}_{t-1,w}^*(\mathbf{x}_t)) - \widetilde{\nabla}\widehat{f}_{t-1,w}(\mathbf{x}_t, \mathbf{y}_t, \mathbf{v}_t)\right\|^2 + 2\left\|\mathcal{G}_\mathcal{X} (\mathbf{x}_t, \widetilde{\nabla}\widehat{f}_{t-1,w}(\mathbf{x}_t, \mathbf{y}_t, \mathbf{v}_t), \gamma)\right\|^2 \nonumber\\
    \leq& 2\kappa_F\mu_g^2\left\|\mathbf{y}_{t-1,w}^*(\mathbf{x}_t) - \mathbf{y}_t\right\|^2 + 16L_{g,1}^2\left\|\mathbf{v}_{t-1,w}^*(\mathbf{x}_t, \mathbf{y}_t) - \mathbf{v}_t\right\|^2 \nonumber\\
    & + 2\left\|\mathcal{G}_\mathcal{X} (\mathbf{x}_t, \widetilde{\nabla}\widehat{f}_{t-1,w}(\mathbf{x}_t, \mathbf{y}_t, \mathbf{v}_t), \gamma)\right\|^2 \nonumber\\
    \overset{(i)}\leq& 2\kappa_F\left\|\nabla_\mathbf{y}\widehat{g}_{t-1,w}(\mathbf{x}_t, \mathbf{y}_t)\right\|^2 + 16\kappa_g^2\left\|\nabla\widehat{\Phi}_{t-1,w}(\mathbf{x}_t, \mathbf{y}_t)\right\|^2 \nonumber\\
    & + 2\left\|\mathcal{G}_\mathcal{X} (\mathbf{x}_t, \widetilde{\nabla}\widehat{f}_{t-1,w}(\mathbf{x}_t, \mathbf{y}_t, \mathbf{v}_t), \gamma)\right\|^2 \overset{(ii)}\leq \frac{2\delta^2}{W^2}, \label{C.4_5}
\end{align}
where $(i)$ comes from the strongly convexity of $\widehat{g}_{t-1,w}(\mathbf{x}, \cdot)$ and $\widehat{\Phi}_{t-1,w}(\mathbf{x}, \mathbf{y}, \cdot)$ and $(ii)$ comes from the stop condition defined in (\ref{eq:condition}). Finally, substituting (\ref{C.4_5}) into (\ref{C.4_3}), we obtain
\begin{align*}
    &\left\|\mathcal{G}_\mathcal{X} (\mathbf{x}_t, \nabla \widehat{f}_{t,w}(\mathbf{x}_t, \mathbf{y}_{t,w}^*(\mathbf{x}_t)), \gamma)\right\|^2 \\
    \leq& 4\left\|\mathcal{G}_\mathcal{X} (\mathbf{x}_t, \nabla \widehat{f}_{t-1,w}(\mathbf{x}_t, \mathbf{y}_{t-1,w}^*(\mathbf{x}_t)), \gamma)\right\|^2 + \frac{4(1+\eta^w)^2L_{f,1}^2L_{g,1}^2D^2}{\mu_g^2W^2} + \frac{2(1+\eta^w)^2L_{f,0}^2}{W^2} \\
    \leq& \frac{8\delta^2}{W^2} + \frac{4(1+\eta^w)^2L_{f,1}^2L_{g,1}^2D^2}{\mu_g^2W^2} + \frac{2(1+\eta^w)^2L_{f,0}^2}{W^2},
\end{align*}
and thus
\begin{align*}
    \sum_{t=1}^T \left\|\mathcal{G}_\mathcal{X} (\mathbf{x}_t, \nabla \widehat{f}_{t,w}(\mathbf{x}_t, \mathbf{y}_{t,w}^*(\mathbf{x}_t)), \gamma)\right\|^2 \leq& \left(8\delta^2 + 4(1+\eta^w)^2L_{f,1}^2\kappa_g^2D^2 + 2(1+\eta^w)^2L_{f,0}^2\right)\frac{T}{W^2}.
\end{align*}

Then, we prove the upper bound of the inner iteration count $\mathcal{I}_T$ over the total $T$ rounds. Consider in step $t$, we have
\begin{align}
    \mathbf{y}_t^0 = \mathbf{y}_t,& \quad \mathbf{y}_t^{k+1} \leftarrow \mathbf{y}_t^k - \alpha\nabla_\mathbf{y}\widehat{g}_{t,w}(\mathbf{x}_t^k, \mathbf{y}_t^k) \nonumber\\
    \mathbf{v}_t^0 = \mathbf{v}_t,& \quad \mathbf{v}_t^{k+1} \leftarrow \mathbf{v}_t^k - \beta\nabla_\mathbf{v}\widehat{\Phi}_{t,w}(\mathbf{x}_t^k, \mathbf{y}_t^k, \mathbf{v}_t^k) \nonumber\\
    \mathbf{x}_t^0 = \mathbf{x}_t,& \quad \mathbf{x}_t^{k+1} \leftarrow \mathcal{P}_\mathcal{X}\left(\mathbf{x}_t^k - \gamma\nabla \widehat{f}_{t,w}(\mathbf{x}_t^k, \mathbf{y}_t^k, \mathbf{v}_t^k)\right), \label{nowin_update}
\end{align}
following the similarity proof process in Lemma~\ref{lem:begin}, it holds that
\begin{align*}
    &\widehat{f}_{t,w}(\mathbf{x}^+, \mathbf{y}_{t,w}^*(\mathbf{x}^+)) - \widehat{f}_{t,w}(\mathbf{x}, \mathbf{y}_{t,w}^*(\mathbf{x})) \\
    \leq& \left\langle \nabla \widehat{f}_{t,w}(\mathbf{x}, \mathbf{y}_{t,w}^*(\mathbf{x})), \mathbf{x}^+ - \mathbf{x} \right\rangle + \frac{L_F}{2}\left\|\mathbf{x}^+ - \mathbf{x}\right\|^2 \\
    \leq& -\gamma\left\langle \nabla \widehat{f}_{t,w}(\mathbf{x}, \mathbf{y}_{t,w}^*(\mathbf{x})), \mathcal{G}_\mathcal{X}(\mathbf{x}, \widetilde{\nabla}\widehat{f}_{t,w}(\mathbf{x}, \mathbf{y}, \mathbf{v}), \gamma)  \right\rangle + \frac{\gamma^2L_F}{2}\left\|\mathcal{G}_\mathcal{X}(\mathbf{x}, \widetilde{\nabla}\widehat{f}_{t,w}(\mathbf{x}, \mathbf{y}, \mathbf{v}), \gamma)\right\|^2 \\
    =& -\gamma\left\langle \widetilde{\nabla}\widehat{f}_{t,w}(\mathbf{x}, \mathbf{y}, \mathbf{v}), \mathcal{G}_\mathcal{X}(\mathbf{x}, \widetilde{\nabla}\widehat{f}_{t,w}(\mathbf{x}, \mathbf{y}, \mathbf{v}), \gamma) \right\rangle + \frac{\gamma^2L_F}{2}\left\|\mathcal{G}_\mathcal{X}(\mathbf{x}, \widetilde{\nabla}\widehat{f}_{t,w}(\mathbf{x}, \mathbf{y}, \mathbf{v}), \gamma)\right\|^2 \\
    & - \gamma\left\langle \nabla \widehat{f}_{t,w}(\mathbf{x}, \mathbf{y}_{t,w}^*(\mathbf{x})) -  \widetilde{\nabla}\widehat{f}_{t,w}(\mathbf{x}, \mathbf{y}, \mathbf{v}), \mathcal{G}_\mathcal{X}(\mathbf{x}, \widetilde{\nabla}\widehat{f}_{t,w}(\mathbf{x}, \mathbf{y}, \mathbf{v}), \gamma) \right\rangle \\
    \leq& - \left(\gamma - \frac{\gamma^2L_F}{2} - \frac{\gamma}{2}\right)\left\|\mathcal{G}_\mathcal{X}(\mathbf{x}, \widetilde{\nabla}\widehat{f}_{t,w}(\mathbf{x}, \mathbf{y}, \mathbf{v}), \gamma)\right\|^2 + \frac{\gamma}{2}\left\|\nabla \widehat{f}_{t,w}(\mathbf{x}, \mathbf{y}_{t,w}^*(\mathbf{x})) -  \widetilde{\nabla}\widehat{f}_{t,w}(\mathbf{x}, \mathbf{y}, \mathbf{v})\right\|^2 \\
    \leq& - \left(\frac{\gamma}{2} - \frac{\gamma^2L_F}{2}\right)\left\|\mathcal{G}_\mathcal{X}(\mathbf{x}, \widetilde{\nabla}\widehat{f}_{t,w}(\mathbf{x}, \mathbf{y}, \mathbf{v}), \gamma)\right\|^2 + \frac{\gamma L_F\mu_g^2}{2}\left\|\mathbf{y}_{t,w}^*(\mathbf{x}) - \mathbf{y}\right\|^2 \\
    & + 4\gamma L_{g,1}^2\left\|\mathbf{v}_{t,w}^*(\mathbf{x}, \mathbf{y}) - \mathbf{v}\right\|^2,
\end{align*}
thus after some rearranging and summing over $k=0,\dots,K_t-1$, we have
\begin{align}
    &\left(\frac{\gamma}{2} - \frac{\gamma^2L_F}{2}\right) \sum_{k=0}^{K_t-1} \left\|\mathcal{G}_\mathcal{X} (\mathbf{x}_t^k, \widetilde{\nabla}\widehat{f}_{t,w}(\mathbf{x}_t^k, \mathbf{y}_t^k, \mathbf{v}_t^k), \gamma))\right\|^2 \nonumber\\
    \leq& \widehat{f}_{t,w}(\mathbf{x}_t, \mathbf{y}_{t,w}^*(\mathbf{x}_t)) - \widehat{f}_{t,w}(\mathbf{x}_{t+1}, \mathbf{y}_{t,w}^*(\mathbf{x}_{t+1})) + \frac{\gamma\kappa_F\mu_g^2}{2}\sum_{k=0}^{K_t-1}\left\|\mathbf{y}_{t,w}^*(\mathbf{x}_t^k) - \mathbf{y}_t^k\right\|^2 \nonumber\\
    & + 4\gamma L_{g,1}^2\sum_{k=0}^{K_t-1}\left\|\mathbf{v}_{t,w}^*(\mathbf{x}_t^k, \mathbf{y}_t^k) - \mathbf{v}_t^k\right\|^2. \label{C.4_6}
\end{align}

Now, in Lemma~\ref{lem:y_bound_win_2} and~\ref{lem:v_bound_win_2}, let $\alpha = \frac{1}{L_{g,1}}$, $\beta = \frac{1}{L_{g,1}}$, thus $\alpha\mu_g = \frac{1}{\kappa_g}$, $\beta\mu_g = \frac{1}{\kappa_g}$, it implies that
\begin{align}
    \sum_{k=0}^{K_t-1}\left\|\mathbf{y}_{t,w}^*(\mathbf{x}_t^k) - \mathbf{y}_t^k\right\|^2
    <& 2\kappa_g\left\|\mathbf{y}_{t,w}^*(\mathbf{x}_t^0) - \mathbf{y}_t^0\right\|^2 \nonumber\\
    & + 8\gamma^2\kappa_g^4\sum_{k=0}^{K_t-1} \left\|\mathcal{G}_\mathcal{X} (\mathbf{x}_t^k, \widetilde{\nabla}\widehat{f}_{t,w}(\mathbf{x}_t^k, \mathbf{y}_t^k, \mathbf{v}_t^k), \gamma))\right\|^2 \label{C.4_7} \\
    \sum_{k=0}^{K_t-1}\left\|\nabla_\mathbf{y}\widehat{g}_{t,w}(\mathbf{x}_t^k, \mathbf{y}_t^k)\right\|^2 
    <& 8\kappa_gL_{g,1}^2\left\|\mathbf{y}_{t,w}^*(\mathbf{x}_t^0) - \mathbf{y}_t^0\right\|^2 \nonumber\\
    & + 32\gamma^2\kappa_g^4L_{g,1}^2\sum_{k=0}^{K_t-1} \left\|\mathcal{G}_\mathcal{X} (\mathbf{x}_t^k, \widetilde{\nabla}\widehat{f}_{t,w}(\mathbf{x}_t^k, \mathbf{y}_t^k, \mathbf{v}_t^k), \gamma))\right\|^2 \label{eq:(1)}
\end{align}
and
\begin{align}
    &\sum_{k=0}^{K_t-1}\left\|\mathbf{v}_{t,w}^*(\mathbf{x}_t^k, \mathbf{y}_t^k) - \mathbf{v}_t^k\right\|^2 \nonumber\\
    <& 2\kappa_g\left\|\mathbf{v}_{t,w}^*(\mathbf{x}_t^0, \mathbf{y}_t^0) - \mathbf{v}_t^0\right\|^2 + 128\kappa_g^3L_\mathbf{v}^2\left\|\mathbf{y}_{t,w}^*(\mathbf{x}_t^0) - \mathbf{y}_t^0\right\|^2 \nonumber\\
    & + (512\kappa_g^6L_{g,1}^2L_\mathbf{v}^2 + 16\kappa_g^2L_\mathbf{v}^2)\gamma^2\sum_{k=0}^{K_t-1}\left\|\mathcal{G}_\mathcal{X} (\mathbf{x}_t^k, \widetilde{\nabla}\widehat{f}_{t,w}(\mathbf{x}_t^k, \mathbf{y}_t^k, \mathbf{v}_t^k), \gamma))\right\|^2 \label{C.4_8} \\
    &\sum_{k=0}^{K_t-1} \left\|\nabla_\mathbf{v} \widehat{\Phi}_{t,w}(\mathbf{x}_t^k, \mathbf{y}_t^k, \mathbf{v}_t^k)\right\|^2 \nonumber\\
    <& 8\kappa_gL_{g,1}^2\left\|\mathbf{v}_{t,w}^*(\mathbf{x}_t^0, \mathbf{y}_t^0) - \mathbf{v}_t^0\right\|^2 + 256\kappa_g^3L_{g,1}^2L_\mathbf{v}^2\left\|\mathbf{y}_{t,w}^*(\mathbf{x}_t^0) - \mathbf{y}_t^0\right\|^2 \nonumber\\
    & + (512\kappa_g^6L_{g,1}^2L_\mathbf{v}^2 + 32\kappa_g^2L_{g,1}^2L_\mathbf{v}^2)\gamma^2\sum_{k=0}^{K_t-1}\left\|\mathcal{G}_\mathcal{X} (\mathbf{x}_t^k, \widetilde{\nabla}\widehat{f}_{t,w}(\mathbf{x}_t^k, \mathbf{y}_t^k, \mathbf{v}_t^k), \gamma))  \right\|^2. \label{eq:(2)}
\end{align}

Substituting (\ref{C.4_7}) and (\ref{C.4_8}) into (\ref{C.4_6}), we obtain
\begin{align}
    &\gamma\left(\frac{1}{2} - \frac{\gamma L_F}{2} - (4\kappa_F\kappa_g^4\mu_g^2 + 2048\kappa_g^6L_{g,1}^4L_\mathbf{v}^2 + 64\kappa_g^2L_{g,1}^2L_\mathbf{v}^2)\gamma^2\right) \sum_{k=0}^{K_t-1} \left\|\mathcal{G}_\mathcal{X} (\mathbf{x}_t^k, \widetilde{\nabla}\widehat{f}_{t,w}(\mathbf{x}_t^k, \mathbf{y}_t^k, \mathbf{v}_t^k), \gamma))\right\|^2 \nonumber\\
    \leq& \widehat{f}_{t,w}(\mathbf{x}_t, \mathbf{y}_{t,w}^*(\mathbf{x}_t)) - \widehat{f}_{t,w}(\mathbf{x}_{t+1}, \mathbf{y}_{t,w}^*(\mathbf{x}_{t+1})) \nonumber\\
    & + (\kappa_F\kappa_g\mu_g^2 + 512\kappa_g^3 L_{g,1}^2 L_\mathbf{v}^2)\gamma\left\|\mathbf{y}_{t,w}^*(\mathbf{x}_t^0) - \mathbf{y}_t^0\right\| + 8\gamma\kappa_g L_{g,1}^2 \left\|\mathbf{v}_{t,w}^*(\mathbf{x}_t^0, \mathbf{y}_t^0) - \mathbf{v}_t^0\right\|^2, \label{eq:(3)}
\end{align}
let (\ref{eq:(1)})$\times\gamma\kappa_F\cdot\frac{1}{4}$ and (\ref{eq:(2)})$\times8\gamma\kappa_g^2\cdot\frac{1}{4}$ plus (\ref{eq:(3)}) to get
\begin{align*}
    &\gamma\left(\frac{1}{2} - \frac{\gamma L_F}{2} - (4\kappa_F\kappa_g^4\mu_g^2 + 2048\kappa_g^6L_{g,1}^4L_\mathbf{v}^2 + 64\kappa_g^2L_{g,1}^2L_\mathbf{v}^2)\gamma^2\right) \sum_{k=0}^{K_t-1} \left\|\mathcal{G}_\mathcal{X} (\mathbf{x}_t^k, \widetilde{\nabla}\widehat{f}_{t,w}(\mathbf{x}_t^k, \mathbf{y}_t^k, \mathbf{v}_t^k), \gamma))\right\|^2 \\
    & + \frac{\gamma\kappa_F}{4}\sum_{k=0}^{K_t-1}\left\|\nabla\widehat{g}_{t,w}(\mathbf{x}_t^k, \mathbf{y}_t^k)\right\|^2 + 2\gamma\kappa_g^2\sum_{k=0}^{K_t-1}\left\|\nabla_\mathbf{v}\widehat{\Phi}_{t,w}(\mathbf{x}_t^k, \mathbf{y}_t^k, \mathbf{v}_t^k)\right\|^2 \\
    \leq& \widehat{f}_{t,w}(\mathbf{x}_t, \mathbf{y}_{t,w}^*(\mathbf{x}_t)) - \widehat{f}_{t,w}(\mathbf{x}_{t+1}, \mathbf{y}_{t,w}^*(\mathbf{x}_{t+1})) + \gamma(8\kappa_g L_{g,1}^2 + 16\kappa_g^3L_{g,1}^2) \left\|\mathbf{v}_{t,w}^*(\mathbf{x}_t^0, \mathbf{y}_t^0) - \mathbf{v}_t^0\right\|^2 \\
    & + \gamma(\kappa_F\kappa_g\mu_g^2 + 512\kappa_g^3 L_{g,1}^2 L_\mathbf{v}^2 + 2\kappa_F\kappa_gL_{g,1}^2 + 512\kappa_g^5L_{g,1}^2L_\mathbf{v}^2)\left\|\mathbf{y}_{t,w}^*(\mathbf{x}_t^0) - \mathbf{y}_t^0\right\| \\
    & +  \gamma^3 (1024\kappa_g^8L_{g,1}^2L_\mathbf{v}^2 + 64\kappa_g^4L_{g,1}^2L_\mathbf{v}^2 + 8 \kappa_F\kappa_g^4L_{g,1}^2) \sum_{k=0}^{K_t-1}\left\|\mathcal{G}_\mathcal{X} (\mathbf{x}_t^k, \widetilde{\nabla}\widehat{f}_{t,w}(\mathbf{x}_t^k, \mathbf{y}_t^k, \mathbf{v}_t^k), \gamma))  \right\|^2,
\end{align*}
to ensure
\begin{align*}
    \frac{1}{4} \leq& \frac{1}{2} - \frac{\gamma L_F}{2} - (4\kappa_F\kappa_g^4\mu_g^2 + 2048\kappa_g^6L_{g,1}^4L_\mathbf{v}^2 + 64\kappa_g^2L_{g,1}^2L_\mathbf{v}^2)\gamma^2 - 8\gamma^2 \kappa_F\kappa_g^4L_{g,1}^2 \\
    & - \gamma^2 (1024\kappa_g^8L_{g,1}^2L_\mathbf{v}^2 + 64\kappa_g^4L_{g,1}^2L_\mathbf{v}^2) \\
    0 \leq& \frac{1}{4} - \frac{\gamma L_F}{2} - (12\kappa_F\kappa_g^4L_{g,1}^2 + 2112\kappa_g^6L_{g,1}^2L_\mathbf{v}^2 + 1080\kappa_g^8L_{g,1}^2L_{\mathbf{v}}^2)\gamma^2 \\
    0 \leq& \frac{1}{4} - \frac{\gamma L_F}{2} - 3204\kappa_F\kappa_g^8L_{g,1}^2L_\mathbf{v}^2\gamma^2
\end{align*}
with $\sqrt{3204} < 57$, we can set the following condition to satisfy the above inequality, 
\begin{align*}
    \gamma \leq \min\left\{\frac{1}{4L_F}, \frac{1}{456\kappa_g^4L_{g,1}L_\mathbf{v}\sqrt{\kappa_F}}\right\}.
\end{align*}

Finally, it follows that
\begin{align}
    \frac{\delta^2K_t}{4W^2} \leq& \frac{1}{4} \sum_{k=0}^{K_t-1} \left\|\mathcal{G}_\mathcal{X} (\mathbf{x}_t^k, \widetilde{\nabla}\widehat{f}_{t,w}(\mathbf{x}_t^k, \mathbf{y}_t^k, \mathbf{v}_t^k), \gamma))\right\|^2 + \frac{\kappa_F}{4}\sum_{k=0}^{K_t-1}\left\|\nabla\widehat{g}_{t,w}(\mathbf{x}_t^k, \mathbf{y}_t^k)\right\|^2 \nonumber\\
    & + 2\kappa_g^2\sum_{k=0}^{K_t-1}\left\|\nabla_\mathbf{v}\widehat{\Phi}_{t,w}(\mathbf{x}_t^k, \mathbf{y}_t^k, \mathbf{v}_t^k)\right\|^2 \nonumber\\
    \leq& \frac{\widehat{f}_{t,w}(\mathbf{x}_t, \mathbf{y}_{t,w}^*(\mathbf{x}_t)) {-} \widehat{f}_{t,w}(\mathbf{x}_{t+1}, \mathbf{y}_{t,w}^*(\mathbf{x}_{t+1}))}{\gamma} + (8\kappa_g L_{g,1}^2 {+} 16\kappa_g^3L_{g,1}^2)\left\|\mathbf{v}_{t,w}^*(\mathbf{x}_t^0, \mathbf{y}_t^0) {-} \mathbf{v}_t^0\right\|^2 \nonumber\\
    & + (\kappa_F\kappa_g\mu_g^2 + 512\kappa_g^3 L_{g,1}^2 L_\mathbf{v}^2 + 2\kappa_F\kappa_gL_{g,1}^2 + 512\kappa_g^5L_{g,1}^2L_\mathbf{v}^2)\left\|\mathbf{y}_{t,w}^*(\mathbf{x}_t^0) - \mathbf{y}_t^0\right\|^2. \label{C.4_9}
\end{align}
We also have
\begin{align}
    \left\|\mathbf{y}_{t,w}^*(\mathbf{x}_t^0) - \mathbf{y}_t^0\right\|^2 \leq& 2\left\|\mathbf{y}_{t,w}^*(\mathbf{x}_{t-1}^{K_{t-1}}) - \mathbf{y}_{t-1,w}^*(\mathbf{x}_{t-1}^{K_{t-1}})\right\|^2 + 2\left\|\mathbf{y}_{t-1,w}^*(\mathbf{x}_{t-1}^{K_{t-1}}) - \mathbf{y}_{t-1}^{K_{t-1}}\right\|^2 \nonumber\\
    \leq& \frac{2(1+\eta^w)^2L_{g,1}^2D^2}{\mu_g^2W^2} + \frac{2}{\mu_g^2}\left\|\nabla_\mathbf{y}\widehat{g}_{t,w}(\mathbf{x}_{t-1}^{K_{t-1}}, \mathbf{y}_{t-1}^{K_{t-1}})\right\|^2 \nonumber\\
    \leq& \frac{2(1+\eta^w)^2L_{g,1}^2D^2}{\mu_g^2W^2} + \frac{2}{\mu_g^2}\frac{\delta^2}{W^2\kappa_F}, \label{C.4_10}
\end{align}
where the second inequality comes from Lemma~\ref{lem:y_bound_win_1}, similarily, 
\begin{align}
    \left\|\mathbf{v}_{t,w}^*(\mathbf{x}_t^0, \mathbf{y}_t^0) - \mathbf{v}_t^0\right\|^2 \leq& 2\left\|\mathbf{v}_{t,w}^*(\mathbf{x}_{t-1}^{K_{t-1}}, \mathbf{y}_{t-1}^{K_{t-1}}) {-} \mathbf{v}_{t-1,w}^*(\mathbf{x}_{t-1}^{K_{t-1}}, \mathbf{y}_{t-1}^{K_{t-1}})\right\|^2 \nonumber\\
    & + 2\left\|\mathbf{v}_{t-1,w}^*(\mathbf{x}_{t-1}^{K_{t-1}}, \mathbf{y}_{t-1}^{K_{t-1}}) {-} \mathbf{v}_{t-1}^{K_{t-1}}\right\|^2 \nonumber\\
    \leq& \frac{2(1+\eta^w)L_{f,0}^2L_{g,1}^2}{\mu_g^4W^2} + \frac{2}{\mu_g^2}\frac
    {\delta^2}{8W^2\kappa_g^2}, \label{C.4_11}
\end{align}
where the second inequality comes from Lemma~\ref{lem:v_bound_win_1}.

Substituting (\ref{C.4_10}) and (\ref{C.4_11}) into (\ref{C.4_9}), we obtain 
\begin{align}
    \frac{\delta^2}{4W^2}K_t \leq& \frac{\widehat{f}_{t,w}(\mathbf{x}_t, \mathbf{y}_{t,w}^*(\mathbf{x}_t)) - \widehat{f}_{t,w}(\mathbf{x}_{t+1}, \mathbf{y}_{t,w}^*(\mathbf{x}_{t+1}))}{\gamma} \nonumber\\
    & + 24\kappa_g^3L_{g,1}^2 \left(\frac{2(1+\eta^w)L_{f,0}^2L_{g,1}^2}{\mu_g^4W^2} + \frac{2}{\mu_g^2}\frac
    {\delta^2}{8W^2\kappa_g^2}\right) \nonumber\\
    & + (3\kappa_F\kappa_gL_{g,1}^2 + 1024\kappa_g^5L_{g,1}^2L_\mathbf{v}^2)\left(\frac{2(1+\eta^w)^2L_{g,1}^2D^2}{\mu_g^2W^2} + \frac{2}{\mu_g^2}\frac{\delta^2}{W^2\kappa_F} \right). \label{C.4_12}
\end{align}
Finally, summing (\ref{C.4_12}) over $t=1,\ldots,T$ and using Lemma~\ref{lem:F-F_win}, we complete the proof as follows:
\begin{align*}
    \mathcal{I}_T \leq& \frac{8QW^2}{\gamma\delta^2} + \frac{4(1{+}\eta^w)L_{f,0}L_{g,1}DWT}{\gamma\delta^2\mu_g} + \frac{8QWT}{\gamma\delta^2} + 24\kappa_g^3L_{g,1}^2 \left(\frac{8(1{+}\eta^w)L_{f,0}^2L_{g,1}^2}{\delta^2\mu_g^4} {+} \frac
    {1}{L_{g,1}^2}\right)T\\
    & + (3\kappa_F\kappa_gL_{g,1}^2 + 1024\kappa_g^5L_{g,1}^2L_\mathbf{v}^2)\left(\frac{8(1+\eta^w)^2L_{g,1}^2D^2}{\delta^2\mu_g^2} + \frac{8}{\mu_g^2\kappa_F}\right) T = O\left(WT\right).
\end{align*}

\end{proof}

\subsection{Proof of Theorem~\ref{thm:unres_win_lower}}\label{sec:proof_win_lowerbound}

\begin{theorem}[Restatement of Theorem~\ref{thm:unres_win_lower}]
Consider any online hypergradient-based algorithm, under Assumptions~\ref{asm1}-\ref{asm3} and~\ref{asm4} with $d_1=d_2 \geq \Omega(T)$, given window size $w > 0$ and $\eta \in(0,1)$, there exist $\{f_t, g_t\}_{t=1}^T$ for which $\mathrm{Reg}_w(T) \geq \Omega(T/W^2)$.
\end{theorem}
\begin{proof}
Constructing a set of functions with linear variation is simple, for example, consider that we have $F_t(\mathbf{x}) = \widetilde{F}([\mathbf{x}]_t)$, it naturally holds that
\begin{align*}
    V_T = \sum_{t=2}^T\sup_\mathbf{x}|F_t(\mathbf{x}) - F_{t-1}(\mathbf{x})| =& \sum_{t=2}^T\sup_\mathbf{x}\left|\widetilde{F}([\mathbf{x}]_t) - \widetilde{F}([\mathbf{x}]_{t-1})\right| \\
    =& (T-1)\left(\max_{z\in\mathbb{R}}\widetilde{F}(z) - \min_{z\in\mathbb{R}}\widetilde{F}(z)\right)=\Theta(T),
\end{align*}
the function $F_t(\mathbf{x})$ requires the algorithm to achieve the smallest possible value at a new coordinate at each time step $t$, therefore, unless the algorithm knows the function at the next $t$ a prior, $[\mathbf{x}_t]_t = 0$ holds and the regret $\left\|\nabla \widetilde{F}(0)\right\|$ is unavoidable. 

\textbf{Step 1: Constructing linearly changing $f_t$ that satisfying Assumptions~\ref{asm2}, \ref{asm3}.}

Specifically, consider the upper-level objective function as follows:
\begin{align*}
    f_t(\mathbf{x}, \mathbf{y}) =& c\sigma([\mathbf{y}]_t) + c\sigma([\mathbf{x}]_t),
\end{align*}
where $\sigma(\cdot)$ denotes the sigmoid function. 
For any $\mathbf{x}, \mathbf{y}\in\mathbb{R}^{d}$ and $\mathbf{x}', \mathbf{y}'\in\mathbb{R}^{d}$, it holds that
\begin{align*}
    \left\|\nabla f_t(\mathbf{x}, \mathbf{y})\right\| \leq \frac{\sqrt{2}c}{4}, \quad \left\|\nabla f_t(\mathbf{x}, \mathbf{y}) - \nabla f_t(\mathbf{x}', \mathbf{y}')\right\| \leq \frac{\sqrt{3}c}{18}\left\|(\mathbf{x}, \mathbf{y}) - (\mathbf{x}', \mathbf{y}')\right\|,
\end{align*}
thus we set $c = \min\{\frac{Q}{2}, 2\sqrt{2}L_{f,0}, 6\sqrt{3}L_{f,1}\}$ to let $f_t$ satisfies Assumptions~\ref{asm2} and~\ref{asm3} for all $t\in[T]$. 

\textbf{Step 2: Constructing $g_t$ that satisfying Assumptions~\ref{asm1}, \ref{asm2} and \ref{asm4}.}

We define a simple inner-level objective function as follows:
\begin{align*}
    g_t(\mathbf{x}, \mathbf{y}) \equiv g(\mathbf{x}, \mathbf{y}) = \frac{\mu}{2}\left\|\mathbf{y}-\vec{\phi}(\mathbf{x})\right\|^2,
\end{align*}
where $\left[\vec{\phi}(\mathbf{x})\right]_i = \sigma([\mathbf{x}]_i)/\sqrt{d}$.
For any given $w>0$ and $\eta\in(0,1)$, it is true that
\begin{align*}
    \mathbf{y}_{t,w}^*(\mathbf{x}) = \mathbf{y}_t^*(\mathbf{x}) = \mathbf{y}_{t-1}^*(\mathbf{x}) = \cdots = \mathbf{y}_{t-w+1}^*(\mathbf{x}) = \vec{\phi}(\mathbf{x}).
\end{align*}
For any $\mathbf{x}, \mathbf{y}\in\mathbb{R}^d$, $\nabla_\mathbf{y}g_t(\mathbf{x}, \mathbf{y}) = \mu(\mathbf{y} - \vec{\phi}(\mathbf{x}))$, it follows that
\begin{align*}
    \nabla_{\mathbf{y}\mathbf{y}}^2 g_t(\mathbf{x}, \mathbf{y}) = \mu\mathbf{I}_d, \quad \nabla_{\mathbf{x}\mathbf{y}}^2g_t(\mathbf{x}, \mathbf{y}) = \begin{bmatrix}
        -\frac{\mu}{\sqrt{d}}\sigma'([\mathbf{x}]_1) \\
        & \ddots\\
        &&-\frac{\mu}{\sqrt{d}}\sigma'([\mathbf{x}]_d)
    \end{bmatrix},
\end{align*}
$g_t(\cdot)$ is $\mu_g$-strongly convex if $\mu \geq\mu_g$, and $\left\|\nabla_{\mathbf{x}\mathbf{y}}^2g_t(\mathbf{x}, \mathbf{y})\right\| \leq \frac{\mu}{4\sqrt{d}} < \frac{\mu}{4} \leq L_{g,1}$ if set $\mu \leq 4L_{g,1}$, thus $\nabla_\mathbf{y}g_t(\mathbf{x}, \mathbf{y})$ is $L_{g,1}$-Lipschitz continuous.
$\nabla_{\mathbf{y}\mathbf{y}}^2g_t(\cdot)$ satisfy $L_{g,2}$-Lipschitz continuity because it is a constant matrix.
For any $\mathbf{x}', \mathbf{y}'\in\mathbb{R}^d$, we have
\begin{align*}
    \left\|\nabla_{\mathbf{x}\mathbf{y}}^2g_t(\mathbf{x}, \mathbf{y}) - \nabla_{\mathbf{x}\mathbf{y}}^2g_t(\mathbf{x}', \mathbf{y}')\right\| \leq \frac{\sqrt{3}\mu}{18\sqrt{d}}\left\|(\mathbf{x}, \mathbf{y}) - (\mathbf{x}', \mathbf{y}')\right\|,
\end{align*}
and $\nabla_{\mathbf{x}\mathbf{y}}^2g_t(\mathbf{x}, \mathbf{y})$ is $L_{g,2}$-Lipschitz continuous if $\mu \leq 6\sqrt{3d}L_{g,2}$.
We set $\mu = \min\{4L_{g,1}, 6\sqrt{3d}L_{g,2}\}$ to let $g_t$ satisfies Assumptions~\ref{asm1} and~\ref{asm2} for all $t\in[T]$.

Also note that $\left\|\vec{\phi}(\mathbf{x})\right\| \leq \sqrt{\sum_{i=1}^d1/d} = 1$, thus $\mathbf{y}_t^*(\mathbf{x})\in \mathcal{Y} = \{\mathbf{y}\in\mathbb{R}^{d} \,|\, \left\|\mathbf{y}\right\| \leq D\}$ with $D = 1$, Assumption~\ref{asm4} is satisfied.

\textbf{Step 3: Characterizing the Lower Bound of Window-Averaged Regret.}

It holds that
\begin{align*}
    F_t(\mathbf{x}) =& f_t(\mathbf{x}, \mathbf{y}_t^*(\mathbf{x})) = c\sigma(\frac{\sigma([\mathbf{x}]_t)}{\sqrt{d}}) + c\sigma([\mathbf{x}]_t) \\
    [\nabla f_t(\mathbf{x}, \mathbf{y}_t^*(\mathbf{x}))]_i =& \begin{cases}
        c\sigma'(\frac{\sigma([\mathbf{x}]_i)}{\sqrt{d}})\frac{\sigma'([\mathbf{x}]_i)}{\sqrt{d}} + c\sigma'([\mathbf{x}]_i) > c\sigma'([\mathbf{x}]_i) \quad& \text{if } i = t\\
        0 & \text{otherwise}
    \end{cases},
\end{align*}
then with $[\mathbf{x}_t]_t=0$, $\left\|\nabla F_t(\mathbf{x}_t)\right\| > c\sigma'(0) = \frac{c}{4}$, we finally have
\begin{align*}
    \mathrm{Reg}_w(T) = \sum_{t=1}^T\left\|\nabla \widehat{f}_{t,w}(\mathbf{x}_t, \mathbf{y}_{t,w}^*(\mathbf{x}_t))\right\|^2 =& \sum_{t=1}^T\left\|\frac{1}{W}\sum_{i=0}^{w-1}\eta^i\nabla f_{t-i}(\mathbf{x}_t, \vec{\phi}(\mathbf{x}))\right\|^2 \\
    >& \frac{1}{W^2}\sum_{t=1}^T\left\|\frac{c}{4}e_t + c\sum_{i=1}^{w-1}\eta^i\sigma'([\mathbf{x}_t]_{t-i})e_{t-i}\right\|^2 \geq \frac{c^2T}{16W^2}.
\end{align*}
Thus we finish the proof.
\end{proof}


\end{document}